\documentclass{article}

\PassOptionsToPackage{numbers}{natbib}
\usepackage[main,final]{neurips_2026}

\usepackage[utf8]{inputenc} 
\usepackage[T1]{fontenc}    
\usepackage{hyperref}       
\usepackage{url}            
\usepackage{booktabs}       
\usepackage{amsfonts}       
\usepackage{nicefrac}       
\usepackage{microtype}      
\usepackage[table]{xcolor}  
\usepackage{tikz}
\usepackage{amsmath}
\usepackage{amssymb}
\usepackage{mathtools}
\usepackage{amsthm}
\usepackage{enumitem}
\usepackage{graphicx}

\usepackage[capitalize,nameinlink]{cleveref}
\usepackage{float}          
\usepackage{pifont}
\usepackage{pdfpages}       
\usepackage{subcaption}     
\usepackage{multirow}       
\usepackage{tabularx}       
\usepackage{makecell}       
\usepackage{threeparttable} 

\usepackage{wrapfig}
\usepackage{tcolorbox}

\usepackage{algorithm}
\usepackage{algorithmic}

\usepackage{amsmath,amsfonts,bm}

\def\eqref#1{equation~\ref{#1}}

\def\1{\bm{1}}

\DeclareMathAlphabet{\mathsfit}{\encodingdefault}{\sfdefault}{m}{sl}
\SetMathAlphabet{\mathsfit}{bold}{\encodingdefault}{\sfdefault}{bx}{n}

\newcommand{\KL}{D_{\mathrm{KL}}}

\newcommand{\ie}{\textit{i.e.}}
\newcommand{\eg}{\textit{e.g.}}
\newcommand{\myparagraph}[1]{\vspace{-0.12in}\paragraph{#1}}

\newcommand{\legendline}[3]{%
  \tikz[baseline=-0.5ex] \draw[line width=0.8pt, color=#1, #2]
    (0,0) -- (0.5,0)
    node[pos=0.5, draw=#1, fill=#1, shape=#3, inner sep=2pt] {};
}

\newcommand{\legendxline}[2]{%
  \tikz[baseline=-0.5ex] \draw[line width=0.8pt, color=#1, #2]
    (0,0) -- (0.5,0)
    node[pos=0.5, color=#1, font=\small] {\textbf{$\times$}};
}

\newcommand{\cmark}{\textcolor{green!60!black}{\ding{51}}}  
\newcommand{\xmark}{\textcolor{red!70!black}{\ding{55}}}    

\theoremstyle{plain}
\newtheorem{theorem}{Theorem}
\newtheorem{proposition}{Proposition}
\newtheorem{lemma}{Lemma}
\newtheorem{corollary}{Corollary}
\theoremstyle{definition}
\newtheorem{definition}{Definition}
\newtheorem{assumption}{Assumption}
\theoremstyle{remark}
\newtheorem{remark}{Remark}

\definecolor{Red}{rgb}{0.768, 0.054, 0.054}
\definecolor{Blue}{rgb}{0.152, 0.294, 0.925}
\definecolor{Green}{rgb}{0,0.4,0.7}
\hypersetup{
    pdftitle={Simple yet Effective Semi-supervised Knowledge Distillation from Vision-Language Models via Dual-Head Optimization},
    pdfauthor={Seongjae Kang, Dong Bok Lee, Hyungjoon Jang, Sung Ju Hwang},
    colorlinks=true,
    citecolor=teal,
    linkcolor=Red,
    urlcolor=Green,
}

\Crefname{table}{Tab.}{Tabs.}
\Crefname{figure}{Fig.}{Figs.}
\Crefname{equation}{Eq.}{Eqs.}
\Crefformat{section}{#2\S#1#3}
\Crefname{thm}{Theorem}{Theorems}
\Crefname{assumption}{Assumption}{Assumptions}
\Crefname{theorem}{Theorem}{Theorems}
\Crefname{algorithm}{Alg.}{Algs.}
\creflabelformat{equation}{#2#1#3}            
\crefrangelabelformat{equation}{#3#1#4--#5#2#6} 

\title{Simple yet Effective Semi-supervised Knowledge Distillation from Vision-Language Models via \textbf{\texttt{D}}ual-\textbf{\texttt{H}}ead \textbf{\texttt{O}}ptimization}

\author{%
  Seongjae Kang$^{1}$\thanks{Equal contribution.}\quad
  Dong Bok Lee$^{1}$\footnotemark[1]\quad
  Hyungjoon Jang$^{2}$\quad
  Sung Ju Hwang$^{1,3}$\\
  $^{1}$KAIST\quad $^{2}$VUNO Inc.\quad $^{3}$DeepAuto.ai\\
  \texttt{\{tjdwo2744,markhi,sjhwang\}@kaist.ac.kr}\\
  \texttt{hyungjoon.jang@vuno.co}
}

\begin{document}

\maketitle

\begin{abstract}

Semi-supervised learning (SSL) has emerged as a practical solution for addressing data scarcity challenges by leveraging unlabeled data.
Recently, vision-language models (VLMs), pre-trained on massive image-text pairs, have demonstrated remarkable zero-/few-shot performance that often surpasses SSL approaches due to their exceptional generalization capabilities.
This gap motivates us to question: how can we effectively \emph{harness the powerful generalization capabilities of VLMs into task-specific models}?
Knowledge distillation (KD) offers a natural framework for transferring VLM capabilities, but we identify that it suffers from \emph{gradient conflicts} between supervised and distillation losses.
To address this challenge, we propose \textbf{\texttt{D}}ual-\textbf{\texttt{H}}ead \textbf{\texttt{O}}ptimization (\textbf{\texttt{DHO}}), which introduces dual prediction heads for each distinct signal.
We observe that \textbf{\texttt{DHO}} mitigates \textit{gradient conflicts}, enabling improved feature learning compared to single-head KD baselines, with practical benefits of minimal computational overhead and test-time hyperparameter tuning without retraining.
Extensive experiments across 15 datasets show that \textbf{\texttt{DHO}} consistently outperforms KD baselines, often outperforming teacher models with smaller student models.
\textbf{\texttt{DHO}} also achieves new state-of-the-art performance on both in-distribution ImageNet semi-supervised learning and out-of-distribution generalization across ImageNet variants. Code and model checkpoints are publicly available at \url{https://github.com/erjui/DHO/tree/main}.

\end{abstract}


\vspace{-0.01in}
\section{Introduction}
\label{sec:introduction}
\vspace{-0.05in}

\looseness=-1


Vision-language models (VLMs), which learn joint vision-language representations through large-scale pre-training, have shown remarkable zero-shot capabilities across diverse tasks~\citep{radford2021learning, jia2021scaling}.
Building upon these strong foundational capabilities, recent work has explored various adaptation strategies, including parameter-efficient approaches such as linear probing~\citep{li2022elevater, huang2024lp++}, lightweight adapters~\citep{zhang2021tip, gao2024clip}, and prompt-based fine-tuning methods~\citep{zhou2022learning, zhou2022conditional, lafon2025gallop}, demonstrating the potential of VLMs for data-limited visual recognition tasks.

In parallel, semi-supervised learning (SSL) has emerged as a practical approach to address data scarcity by leveraging both labeled and unlabeled data~\citep{sohn2020fixmatch, assran2021semi, cai2022semi, zheng2023simmatchv2}.
While these methods have shown success to leverage large amounts of unlabeled data, they often struggle to match the impressive zero- and few-shot capabilities of large pre-trained VLMs~\citep{liu2023learning}.
This discrepancy highlights a fundamental limitation: traditional semi-supervised methods, despite their theoretical appeal, remain suboptimal compared to the rich representations learned by foundation models through massive-scale pre-training.

VLMs excel at zero- and few-shot generalization but may lack the fine-grained discriminative power needed for specific tasks, while models trained on limited labeled data capture task-specific patterns but generalize poorly.
This complementary nature motivates us to integrate generalist VLM knowledge with task-specific supervision in semi-supervised learning settings.
Therefore, the challenge naturally arises: \textit{how can we effectively transfer the powerful capabilities of large VLMs to task-specific models in semi-supervised settings?}

\begin{figure}
    \centering
    \vspace{-0.35in}
    {\includegraphics[width=0.58\textwidth]{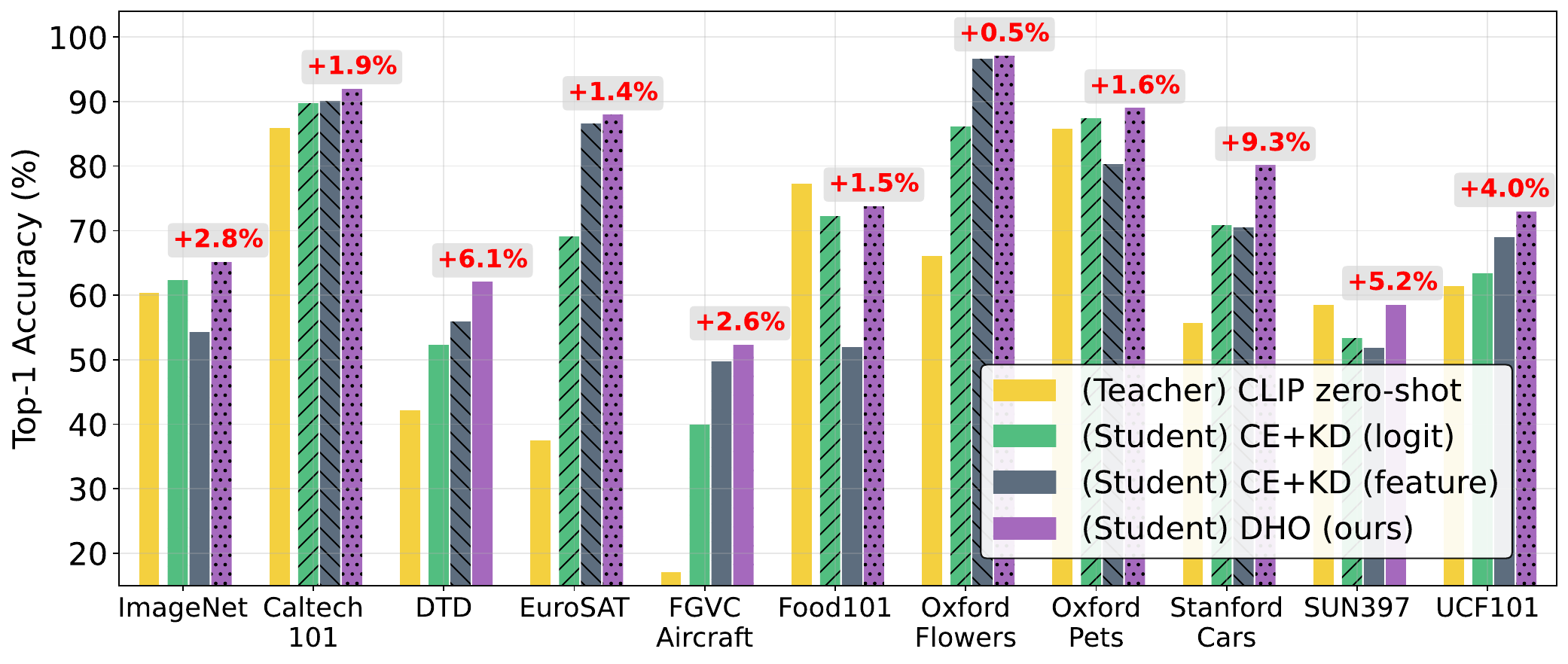}}
    {\includegraphics[width=0.41\textwidth]{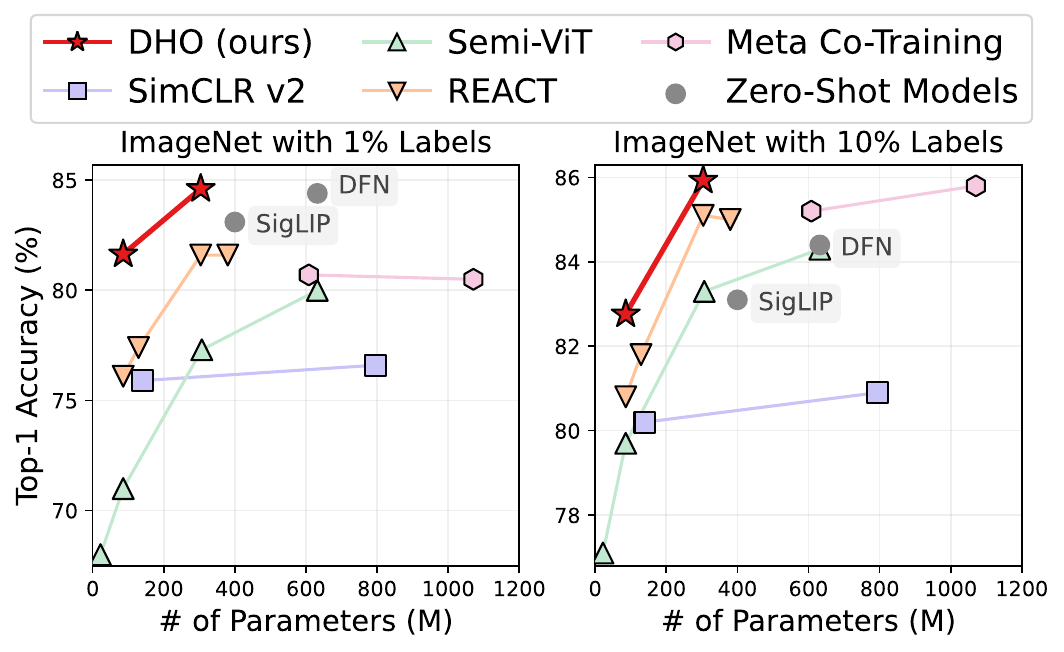}}
    \vspace{-0.25in}
    \caption{\small
     \textbf{(Left)}: \textbf{\texttt{DHO}} consistently outperforms single-head baselines on 11 datasets under 16-shot semi-supervised setting.
     The improvements are evaluated in comparison to the second-best one.
     \textbf{(Right)}: \textbf{\texttt{DHO}} achieves new SoTA on ImageNet in both 1\% and 10\% labeled data setting, with fewer parameters.
    }
    \vspace{-0.3in}
    \label{fig:both}
\end{figure}

Knowledge distillation \citep[KD;][]{hinton2015distilling} emerges as a natural solution to the challenge, as it offeres an efficient framework to transfer knowledge from large VLMs to student models while simultaneously leveraging task-specific patterns from limited labeled data.
However, existing VLM distillation methods primarily focus on general-purpose training \citep{yang2024clipkd, vasu2024mobileclip, udandarao2024active, yang2024clipcid} or employ multi-stage pipelines with unsupervised pre-distillation \citep{knowledgetransfer2024, wu2025cascade}, which require additional task-specific fine-tuning.
Conventional single-stage KD methods, such as logit distillation~\citep{hinton2015distilling, chen2020big}, offer a more direct approach, but we find them \emph{suboptimal} in semi-supervised settings.

Through analysis, we identify that this stems from a fundamental problem: \textit{gradient conflicts} between the supervised loss (from limited labeled data) and the distillation loss (from teacher predictions).
This misalignment is particularly severe in semi-supervised settings, where the strong and consistent distillation signal from the teacher can overwhelm the weak and potentially noisy signal from scarce labeled data.
Such gradient conflicts (\legendline{pink}{}{rectangle} and \legendline{red}{}{circle} in \Cref{fig:average_gradient_conflict}) are well-documented to impair effective feature learning~\citep{gradient_conflict1, gradient_conflict2, gradient_conflict3}, preventing the model from finding an optimal balance between task-specific adaptation and general knowledge transfer.
We formally characterize when this conflict arises---deriving geometric conditions under which the supervised and KD gradients oppose each other in both the classifier head (\Cref{prop:geometric_conflict}) and the shared feature extractor (\Cref{prop:feature_extractor_conflict})---and show that \textbf{\texttt{DHO}} removes the structural guarantee of conflict (\Cref{prop:dho_no_conflict}).

To address the above issue, we propose a \textit{simple yet effective} distillation framework, \textbf{\texttt{DHO}} (\textbf{\texttt{D}}ual-\textbf{\texttt{H}}ead \textbf{\texttt{O}}ptimization), which jointly leverages labeled samples and the probabilistic outputs of the teacher model. Specifically, it learns two distinct heads, each optimized with a separate loss: the supervised and the KD loss, respectively.
Our analysis reveals that \textbf{\texttt{DHO}} \textbf{mitigates gradient conflicts} both in the classification head and the shared feature extractor (\legendline{blue}{}{circle} in \Cref{fig:average_gradient_conflict}), arising from the two different training signals and \textbf{improves feature representations} compared to baselines (\Cref{tab:linear_evaluation,fig:tsne}).
Our framework also enables controlling the relative influence of supervised and teacher predictions through \textbf{linear combination} of both head outputs, whose effectiveness is demonstrated in \Cref{fig:F3_head_interpolation}.
Additionally, \textbf{\texttt{DHO}} enables \textbf{post-training hyperparameter adjustment} by tuning the linear combination weights at inference (\Cref{fig:ablation_alpha_beta}), eliminating costly training-time hyperparameter search.

We conduct extensive experiments across 15 different datasets including ImageNet~\citep{russakovsky2015imagenet}.
The experimental results demonstrate the \textbf{consistent improvement of \textbf{\texttt{DHO}}} over conventional KD methods across all evaluated datasets (\Cref{fig:both}-Left). \textbf{\texttt{DHO}} sometimes even \textbf{outperforms the zero/few-shot teacher models} with smaller student models with the same labeled data, demonstrating effective task-specific enhancement beyond teacher capabilities (\Cref{fig:F1_10_datasets_zeroshot,fig:F1_10_datasets_fewshot}).
Furthermore, \textbf{\texttt{DHO}} achieves \textbf{new state-of-the-art (SoTA)} performance on ImageNet semi-supervised setting, improving accuracy by 3\% and 0.1\% with 1\% and 10\% labeled data, respectively, while using fewer parameters (\Cref{fig:both}-Right).
Notably, \textbf{\texttt{DHO}} can be seamlessly \textbf{integrated with existing adaptation techniques} with minimal computational overhead, achieving new \textbf{SoTA on Out-of-Distribution (OOD) tasks} across ImageNet distribution-shifted variants (\Cref{tab:domain_generalization_adaptation} and \Cref{sec:appendix_robustness_full}).

Our contributions and findings are summarized as follows:
\begin{itemize}
[itemsep=1mm,parsep=1pt,topsep=2pt,leftmargin=*]
\vspace{-0.05in}

\item We firstly \textbf{identify} \emph{gradient conflict} when integrating VLMs' general knowledge with task-specific supervision from limited data. To address this, we propose \textbf{\texttt{D}}ual-\textbf{\texttt{H}}ead \textbf{\texttt{O}}ptimization (\textbf{\texttt{DHO}}), which optimizes the supervised and distillation objectives in separate heads.

\item \textbf{\texttt{DHO}} effectively resolves \emph{gradient conflicts} both in the classification head and the shared feature extractor, leading to improved \textbf{feature representations}.
Our framework enables flexible \textbf{post-training adjustment} of dual head output weights with \textbf{minimal computational overhead}.

\item In extensive experiments on 15 datasets, \textbf{\texttt{DHO}} \textbf{consistently outperforms} baselines and sometimes \textbf{surpasses teacher model} performance.
It establishes new \textbf{SoTA} for \textbf{in-distribution ImageNet} (\eg, +3\%/+0.1\% at 1\%/10\% labels) with fewer parameters (76M/767M), and also achieves \textbf{SoTA OOD generalization} across ImageNet variants when integrated with adaptation methods.
\vspace{-0.05in}
\end{itemize}



\vspace{-0.05in}
\section{Method}
\vspace{-0.05in}

\subsection{Preliminaries}
\label{sec:preliminary}
\vspace{-0.05in}
We begin with preliminaries: a brief background on VLMs, the problem formulation for few-/low-shot learning, and single-head KD baselines. We defer related work---\textbf{vision--language pretraining, data-limited adaptation of VLMs, knowledge distillation, and dual-head methods}---to \Cref{sec:related_work}.


\vspace{-0.1in}
\paragraph{Background on VLMs.}
Our work is based on VLMs such as CLIP~\citep{radford2021learning} and ALIGN~\citep{jia2021scaling}.
These models consist of multimodal encoders: an image encoder $f_\mathcal{X}: \mathcal{X} \to \mathbb{R}^d$ and a text encoder $f_\mathcal{T}: \mathcal{T} \to \mathbb{R}^d$ where $\mathcal{X}$ and $\mathcal{T}$ denote the domains of images and texts, respectively.
For zero-shot classification of VLMs across $C$ classes, we use predefined prompt templates, \eg, ``a photo of a \texttt{[CLASS]}'', where \texttt{[CLASS]} is the name of class.
Given a set of $C$ target classes, \ie, $y\in\{1,\ldots,C\}$, we generate prompted text descriptions $\{t_1, t_2, \dots, t_C\}$.
We obtain the categorical probability vector $p$ using the cosine similarity $\mathtt{CosSim}(x, y)=\frac{x^\top y}{||x||_2\cdot||y||_2}$ over $\{t_1, t_2, \dots, t_C\}$, \ie, $p\coloneq\sigma([\mathtt{CosSim}(f_\mathcal{T}(x), f_\mathcal{T}(t_1))/\zeta,\ldots,\mathtt{CosSim}(f_\mathcal{T}(x), f_\mathcal{T}(t_C))/\zeta])$, where $\sigma$ is the softmax function, $\zeta\in\mathbb{R}_{>0}$ is the temperature scaling~\citep{hinton2015distilling}, and final classification is determined by $\arg\max_{c\in\{1,\ldots,C\}}p_c$.

\myparagraph{Problem Formulation.}
We focus on transferring knowledge from VLMs to task-specific models under few-shot or low-shot semi-supervised learning scenarios, where both labeled and unlabeled data are utilized.
Specifically, given a $K$-shot and $C$-class classification problem, we are provided with a labeled dataset $\mathcal{D}^{(l)}=\{(x^{(l)}_n, y_n)\}_{n=1}^{N}$, where $N = K \times C$ is the total number of labeled examples, and $y_n \in \{1,\ldots,C\}$ denotes the class labels.
Additionally, we have access to an unlabeled dataset $\mathcal{D}^{(u)}=\{x^{(u)}_m\}_{m=1}^M$ consisting of $M$ unlabeled images. Low-shot learning represents a more realistic setting than traditional few-shot learning where only a small fraction, \eg, 1\% ($\frac{N}{N+M}\approx0.01$) or 10\% ($\frac{N}{N+M}\approx0.1$) of the total dataset is labeled.
Our goal is then to develop a student model by leveraging \( \mathcal{D}^{(l)} \) and \( \mathcal{D}^{(u)} \), guided by the knowledge of the VLM encoders \( f_\mathcal{X} \) and \( f_\mathcal{T} \).
The student model consists of a \textbf{feature extractor $g(x)$ parameterzied by $\theta$} and \textbf{a linear prediction head} $h(z)=Wz+b\in\mathbb{R}^c$, where $W\in\mathbb{R}^{C\times d},b\in\mathbb{R}^{C}$, followed by \textbf{the softmax function} $\sigma$.

\begin{figure}[t]
    \centering
    \vspace{-0.25in}
    \includegraphics[width=1.0\linewidth]{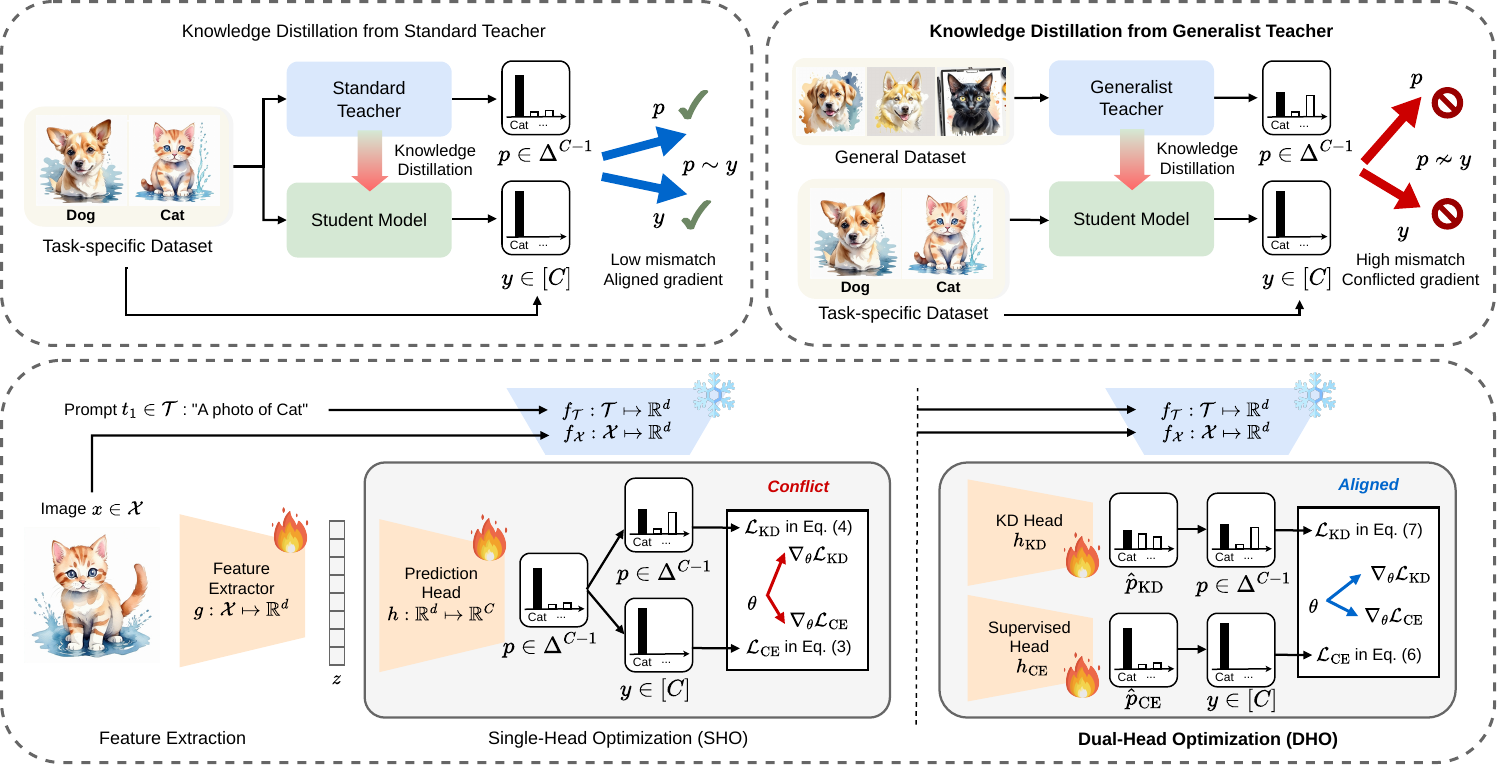}
    \vspace{-0.2in}
    \caption{\small
        \textbf{Conceptual illustration.}
        \textbf{(Top)} Unlike conventional KD where teacher and student share the dataset, a \emph{generalist} VLM teacher (\eg, CLIP) produces predictions misaligned with task labels, yielding \emph{conflicting} CE/KD gradients (\Cref{sec:appendix_motivation}).
        \textbf{(Bottom)} Single-Head Optimization (SHO) routes both losses through one classifier and inherits this conflict (\Cref{fig:average_gradient_conflict}); \textbf{\texttt{D}}ual-\textbf{\texttt{H}}ead \textbf{\texttt{O}}ptimization (\textbf{\texttt{DHO}}) decouples the heads and \textbf{mitigates it}.
    }
\vspace{-0.15in}
\end{figure}

\myparagraph{Single-head optimization (SHO) of KD.}
Our method builds on \textit{logit distillation} in semi-supervised settings~\citep{hinton2015distilling, chen2020big} that combines supervised loss $\mathcal{L}_{\text{CE}}$ on the labeled dataset $\mathcal{D}^{(l)}$ with KD loss $\mathcal{L}_{\text{KD}}$ on both labeled and unlabeled datasets $\mathcal{D}^{(l)}\cup\mathcal{D}^{(u)}$, \ie, $\lambda\mathcal{L}_{\text{CE}} + (1-\lambda) \mathcal{L}_{\text{KD}}$, where $\lambda\in[0,1]$ is a loss balancing hyperparameter.
Specifically, the supervised loss $\mathcal{L}_{\text{CE}}$ and the KD loss $\mathcal{L}_{\text{KD}}$ are defined as follows:
\vspace{-0.05in}
\begin{align}
    \mathcal{L}_{\text{CE}} &= \frac{1}{N}\sum_{n}\ell\left(\sigma (h(z_n^{(l)})), y_n\right), \label{eq:ce}\\
    \mathcal{L}_{\text{KD}} &= \frac{1}{N}\sum_{n}\KL\left[p_n^{(l)} \Vert \sigma(h(z_n^{(l)}))\right] + \frac{1}{M}\sum_{m}\KL\left[p_m^{(u)} \Vert \sigma(h(z_m^{(u)}))\right].  \label{eq:logit}
\end{align}
where $\ell$ denotes $\KL$ represent the cross-entropy and Kullback-Leibler divergence, respectively.
$z_n^{(l)}=g(x_n^{(l)})$ and $z_m^{(u)}=g(x_m^{(u)})\in\mathbb{R}^d$ are feature representations obtained by the feature extractor $g$. $p_n^{(l)}$ and $p_m^{(u)}$ are the categorical probability vectors of labeled $x_n^{(l)}$ and unlabeled data $x_m^{(u)}$, respectively, obtained by teacher VLM encoders $f_\mathcal{X}$ and $f_\mathcal{T}$.
Another well-studied single-head KD baseline is \textit{feature distillation}, which leverages mean squared error (MSE) loss to directly align feature representations extracted by the student encoder $g$ and the teacher image encoder $f_\mathcal{X}$.
We defer the details of feature distillation for VLMs to CLIP-KD~\citep{yang2024clipkd}.


\vspace{-0.05in}
\subsection{\textbf{\texttt{D}}ual-\textbf{\texttt{H}}ead \textbf{\texttt{O}}ptimization (\textbf{\texttt{DHO}})}\label{sec:dho}
\vspace{-0.05in}

\begin{wrapfigure}{r}{0.4\textwidth}
    \vspace{0pt}
    \includegraphics[width=0.4\textwidth]{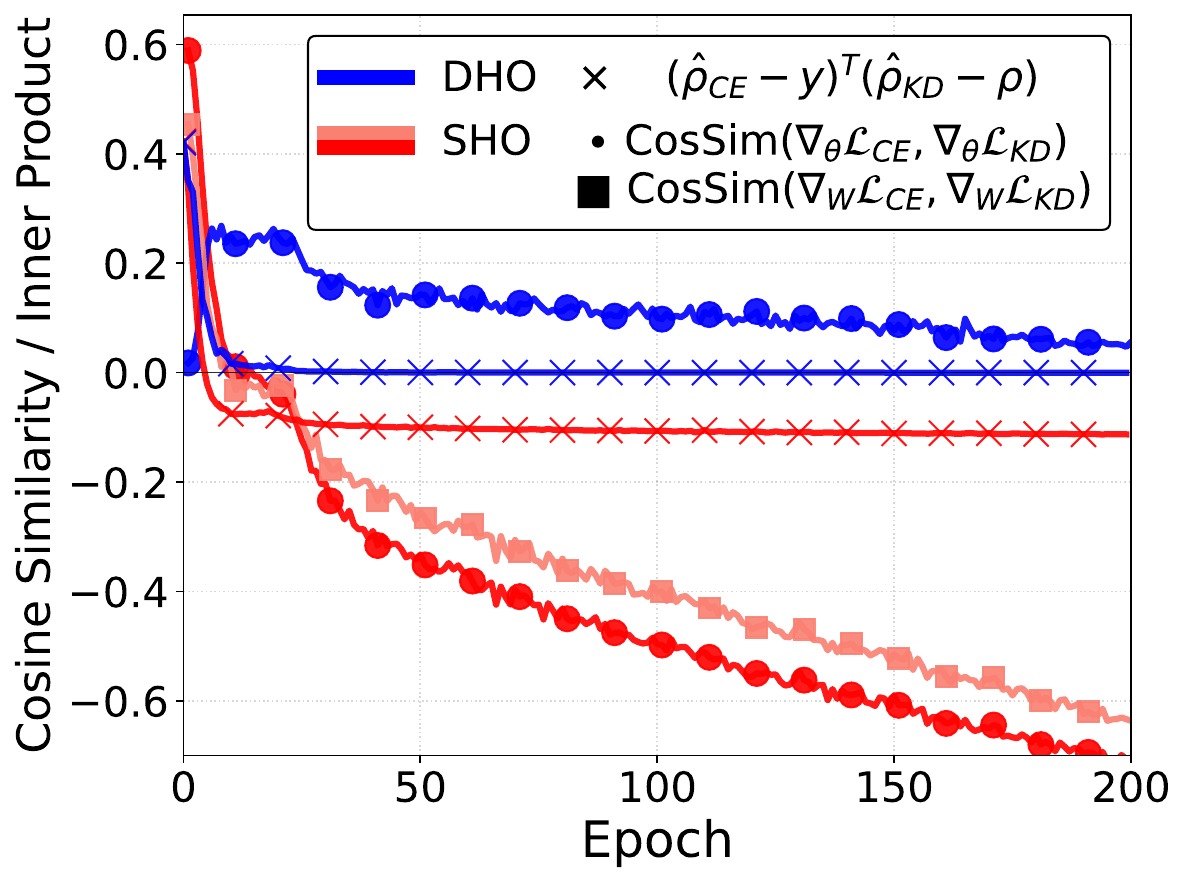}
    \vspace{-0.25in}
    \caption{\small \textbf{The average cosine similarity and inner product} over 10 datasets.}
    \vspace{-0.1in}
    \label{fig:average_gradient_conflict}
\end{wrapfigure}

\paragraph{Gradient conflicts in SHO.}

Logit distillation in SHO, as described in \Cref{sec:preliminary}, provides a simple approach for transferring knowledge from VLMs to task-specific models.
However, we find that its performance gain is \textit{suboptimal}, which we attribute to \textbf{gradient conflicts} between the supervised and KD loss signals which hinder effective feature learning~\citep{gradient_conflict1, gradient_conflict2, gradient_conflict3}.
As illustrated in \legendline{pink}{}{rectangle} and \legendline{red}{}{circle} of \Cref{fig:average_gradient_conflict}, both \textbf{the classifier weight $W$ of classifier $h$ and the parameter vector $\theta$ of feature extractor $g$} suffer from gradient conflicts:
the cosine similarity between their respective gradients turns negative, \ie, $\mathtt{CosSim}(\nabla_{\theta} \mathcal{L}_{\text{CE}}, \nabla_\theta \mathcal{L}_{\text{KD}})<0$ and $\mathtt{CosSim}(\nabla_{W} \mathcal{L}_{\text{CE}}, \nabla_W \mathcal{L}_{\text{KD}})<0$, indicating misaligned optimization directions. 

To understand this, we first analyze the gradient with respect to (w.r.t) the weight $W$ of a linear head $h$.
The gradients w.r.t the classifier weight $W$ are $\nabla_{W} \mathcal{L}_{\text{CE}} = (\hat{p} - y) z^\top$ and $\nabla_{W} \mathcal{L}_{\text{KD}} = (\hat{p} - p) z^\top$, respectively.
Their cosine similarity is proportional to $(\hat{p} - y)^\top (\hat{p} - p) \cdot \|z\|^2$:
\begin{gather}
\mathtt{CosSim}(\nabla_{W} \mathcal{L}_{\text{CE}}, \nabla_{W} \mathcal{L}_{\text{KD}}) \propto (\hat{p} - y)^\top (\hat{p} - p) \cdot \|z\|^2,
\end{gather}
which misaligns when $(\hat{p} - y)^\top (\hat{p} - p) < 0$ due to prediction mismatch, falling below zero during training (\legendxline{red}{} in \Cref{fig:average_gradient_conflict}), leading to gradient conflicts (\legendline{pink}{}{rectangle} in \Cref{fig:average_gradient_conflict}).
In \Cref{sec:gradient_conflict_analysis}, we prove that this inner product equals $\|\hat{p} - m\|^2 - r^2$ where $m{=}\frac{y+p}{2}$ and $r{=}\frac{\|y-p\|}{2}$, so \textbf{gradient conflict is guaranteed whenever $\hat{p}$ lies inside the ball $\mathcal{B}(m, r)$} with the segment $[y, p]$ as its diameter (\Cref{prop:geometric_conflict}).
In particular, if $\hat{p} = (1{-}t)\,y + t\,p$ for $t\in(0,1)$, then $(\hat{p} - y)^\top(\hat{p} - p) = -t(1{-}t)\|p{-}y\|^2 < 0$ (\Cref{cor:interpolation_conflict}).

We further show that \textbf{gradient conflicts propagate to the feature extractor $g$}.
Let the gradients w.r.t to the feature representation $z$ be $\nabla_z \mathcal{L}_{\text{CE}} = W^\top (\hat{p} - y)$ and $\nabla_z \mathcal{L}_{\text{KD}} = W^\top (\hat{p} - p)$.
Applying the chain rule, the cosine similarity for the feature extractor parameters $\theta$ becomes:
\begin{gather}
\mathtt{CosSim}(\nabla_{\theta} \mathcal{L}_{\text{CE}}, \nabla_\theta \mathcal{L}_{\text{KD}}) \propto (\hat{p} - y)^\top W \left(\frac{\partial z}{\partial \theta}\right)^\top \frac{\partial z}{\partial \theta} W^\top (\hat{p} - p)
\end{gather}
Since $A = W \left(\frac{\partial z}{\partial \theta}\right)^\top \frac{\partial z}{\partial \theta} W^\top \succeq 0$ is PSD, the gradient conflict is also guaranteed under the interpolation condition (\Cref{prop:feature_extractor_conflict}), which is empirically confirmed (\legendline{red}{}{circle} in \Cref{fig:average_gradient_conflict}).

\myparagraph{Dual-head architecture.}
To mitigate this issue, we propose \textbf{\texttt{D}}ual \textbf{\texttt{H}}ead \textbf{\texttt{O}}ptimization (\textbf{\texttt{DHO}}) to decouple $\mathcal{L}_\text{CE}$ and $\mathcal{L}_\text{KD}$ via \textit{two independent prediction heads}: \( h_{\text{CE}}(z) = W_{\text{CE}} z + b_{\text{CE}} \) and \( h_{\text{KD}}(z) = W_{\text{KD}} z + b_{\text{KD}} \), with \( W_{\text{CE}}, W_{\text{KD}} \in \mathbb{R}^{C \times d} \) and \( b_{\text{CE}}, b_{\text{KD}} \in \mathbb{R}^{C} \).
The corresponding losses are:
\begin{align}
\mathcal{L}_{\text{CE}} &= \frac{1}{N}\sum_{n} \ell\left(\sigma(h_{\text{CE}}(z_n^{(l)})), y_n\right), \\
\mathcal{L}_{\text{KD}} &= \frac{1}{N}\sum_{n} \KL\left[ p_n^{(l)} \Vert \sigma(h_{\text{KD}}(z_n^{(l)}))\right] + \frac{1}{M} \sum_{m} \KL\left[ p_m^{(u)} \Vert  \sigma(h_{\text{KD}}(z_m^{(u)}))\right],\label{eq:dho_kd}
\end{align}
where the final loss combines both objectives as \( \lambda\mathcal{L}_{\text{CE}} + (1 - \lambda)\mathcal{L}_{\text{KD}} \). 


\myparagraph{Mitigation of gradient conflict in \textbf{\texttt{DHO}}.}

In \textbf{\texttt{DHO}}, gradient conflicts in \textbf{the classification head naturally disappear by decoupling the optimization} of $W_{\text{CE}}$ and $W_{\text{KD}}$.
This separation enables each head to learn distinct signals without interference, reducing prediction mismatch (\legendxline{blue}{} of \Cref{fig:average_gradient_conflict}).
Let \( \hat{p}_{\text{CE}}{=}\sigma(h_{\text{CE}}(z)) \) and \( \hat{p}_{\text{KD}}{=}\sigma(h_{\text{KD}}(z)) \), then $\nabla_z \mathcal{L}_{\text{CE}}{=}W_{\text{CE}}^\top (\hat{p}_{\text{CE}} - y)$ and $\nabla_z \mathcal{L}_{\text{KD}}{=}W_{\text{KD}}^\top (\hat{p}_{\text{KD}} - p)$. The cosine similarity between gradients w.r.t $\theta$ in \textbf{\texttt{DHO}} is defined as:
\begin{gather}
\mathtt{CosSim}(\nabla_{\theta} \mathcal{L}_{\text{CE}}, \nabla_\theta \mathcal{L}_{\text{KD}}) \propto (\hat{p}_{\text{CE}} - y)^\top W_{\text{CE}} \left(\frac{\partial z}{\partial \theta}\right)^\top \frac{\partial z}{\partial \theta} W_{\text{KD}}^\top (\hat{p}_{\text{KD}} - p),\label{eq:cossim_dho}
\end{gather}
In \Cref{prop:dho_no_conflict}, we prove that the cross-head matrix $W_{\text{CE}} \left(\frac{\partial z}{\partial \theta}\right)^\top \frac{\partial z}{\partial \theta} W_{\text{KD}}^\top$ is \textbf{not PSD} in general, breaking the structural guarantee of negative alignment that causes conflicts in SHO.
Empirically, \textbf{\texttt{DHO}} maintains positive gradient alignment throughout training (\ie, \Cref{eq:cossim_dho} > 0; \legendline{blue}{}{circle} of \Cref{fig:average_gradient_conflict}), enabling conflict-free representation learning and \textbf{better feature representations} compared to SHO (\Cref{tab:linear_evaluation}).
We defer an algorithm that describes the full training procedure of \textbf{\texttt{DHO}} to \Cref{alg:training}.


\vspace{-0.05in}
\subsection{Dual-head Interpolation}\label{sec:dual_head_interpolation}
\vspace{-0.05in}
After training, we find that using only one of the two heads at inference is \emph{suboptimal} (see \Cref{fig:F3_head_interpolation}).
Motivated by the mixture-of-experts paradigm~\citep{jacobs1991adaptive}, we adopt a \emph{simple yet effective} inference rule for \textbf{\texttt{DHO}} that linearly interpolates the output probability vectors of two heads:
\begin{align}
\hat{p}_{\textbf{\texttt{DHO}}}
= \alpha\, \sigma\!\big(h_{\text{CE}}(z)\big)
+ (1-\alpha)\, \sigma\!\big(h_{\text{KD}}(z)/\beta\big),
\label{eq:inference}
\end{align}
where $\alpha\in[0,1]$ balances the supervised and KD heads, $\beta>0$ is a temperature that softens the KD logits, and the final prediction is $\arg\max_{c\in[C]} \hat{p}_{\textbf{\texttt{DHO}}}$.
In practice, $(\alpha,\beta){=}(0.5,1)$ already outperforms SHO without tuning (\Cref{tab:dho_fixed_vs_searched}); $\alpha,\beta$ can also be tuned on a validation split for further gains.
See \Cref{alg:inference} for the full inference procedure.

\myparagraph{Effect of $\alpha$ and $\beta$.}
We now demonstrate the effect of tuning $\alpha$ and $\beta$ using a validation set.
Under mild assumptions, \textbf{\texttt{DHO}} $\varepsilon$-approximates SHO in $\ell_1$ by setting $\alpha=\lambda$ and $\beta=1$:

\begin{assumption}[$\varepsilon$-convergence]

\label{main_assump:convergence}
\textit{Assume that, after sufficient training, both heads converge to their respective target distributions with $\ell_1$-bounded error:}
\begin{equation}
\sup_x \|\sigma(h_{\mathrm{CE}}(z)) - y\|_1 \le \varepsilon,\quad
\sup_x \|\sigma\!(h_{\mathrm{KD}}(z)) - p\|_1 \le \varepsilon, \quad \text{where} \quad \varepsilon\in\mathbb{R}_{>0}.
\end{equation}
\end{assumption}

\begin{theorem}[Inference equivalence]
\label{main_thm:inference_equiv}
Under \Cref{main_assump:convergence}, by setting $\alpha=\lambda$ and $\beta=1$, then
$\|\hat{p}_{\textbf{\texttt{DHO}}}-\hat{p}_{\mathrm{SHO}}\|_1 \le \varepsilon,$
where $\hat{p}_{\mathrm{SHO}}$ is the output of SHO optimally trained with $\lambda$.
\end{theorem}
Details, proofs, and a discussion of \Cref{main_assump:convergence} (\Cref{rem:assumption_role}) are deferred to \Cref{sec:theoretical_analysis}. \Cref{main_thm:inference_equiv} implies that $\hat{p}_{\text{SHO}}$ trained with any $\lambda$ can be approximated by the dual-head interpolation in \Cref{eq:inference}, by setting $\alpha=\lambda$ and $\beta=1$. Here, $\lambda$ is a \textbf{training} hyperparameter of SHO, while $\alpha$ and $\beta$ are \textbf{inference} hyperparameters of \textbf{\texttt{DHO}}, allowing it to \textbf{emulate SHO hyperparameter tuning} without retraining.

\myparagraph{Language-aware initialization for VLM students}
In the case of VLM-to-VLM distillation, we leverage the text encoder $f_{\mathcal{T}}$ of teachers when initializing the dual heads $h_{\text{CE}}$ and $h_{\text{KD}}$.
Following prior work~\citep{li2022elevater}, we initialize the weights as $W_{\text{CE}}, W_{\text{KD}} \leftarrow [f_{\mathcal{T}}(t_1), \ldots, f_{\mathcal{T}}(t_C)]^\top \in \mathbb{R}^{C \times d}$.
We further align the prediction logic of KD head $h_\text{KD}$, with the cosine similarity-based approach of the teacher VLMs as follows:
\begin{align}
h_{\text{KD}} = \frac{1}{\zeta}[\mathtt{CosSim}(g(x), w_{1}), \ldots, \mathtt{CosSim}(g(x), w_C)]^\top \in \mathbb{R}^{C},
\label{eq:align}
\end{align}
where $w_c \in \mathbb{R}^d$ denotes the $c$-th row of $W_{\text{KD}}$.

\begin{table}[t]
\centering
\vspace{-0.25in}
\caption{\small Results on \textbf{ImageNet} under few-shot semi-supervision using \textbf{ResNet-18} and
\textbf{ResNet-50}. \textbf{\texttt{DHO}} \textbf{consistently outperforms all baselines and even the teacher} with ResNet-50 (\eg, +0.7/1.9/2.2/3.4/4.4\% with a zero-shot teacher; +1.3/1.5/1.4/1.7/1.4\% with a few-shot teacher). See \Cref{sec:related_work} for comparison with prior dual-head methods and \Cref{sec:appendix_implementation_details} for student initialization rationale (orthogonal to \textbf{\texttt{DHO}}).
}
\vspace{0pt}
\resizebox{\textwidth}{!}{%
    \begin{tabular}{lcccccc|cccccc}
        \toprule
        & \multicolumn{6}{c|}{\textbf{ResNet-18 trained from scratch}} & \multicolumn{6}{c}{\textbf{Self-supervised ResNet-50}} \\

        \textbf{Method} & 0-shot & 1-shot & 2-shot & 4-shot & 8-shot & 16-shot & 0-shot & 1-shot & 2-shot & 4-shot & 8-shot & 16-shot \\
        \midrule
        \rowcolor{gray!20}\multicolumn{13}{l}{\textit{Single-head KD methods}} \\
        KD & 51.0 & - & - & - & - & - & 61.0 & - & - & - & - & - \\
        CE & - & 0.7 & 1.1 & 1.8 & 3.4 & 8.2 & - & 11.4 & 17.3 & 26.4 & 36.7 & 47.0 \\
        CE+KD (feature) & - & 17.1 & 23.5 & 28.0 & 32.2 & 33.8 & - & 23.0 & 32.3 & 41.3 & 48.2 & 54.3 \\
        CE+KD (logit) & - & 50.5 & 50.6 & 50.6 & 51.0 & 51.2 & - & 60.4 & 60.8 & 61.2 & 61.6 & 62.3 \\
        \midrule
        \rowcolor{gray!20}\multicolumn{13}{l}{\textit{Dual-head KD methods}} \\
        SSKD~\citep{he2021semi} & - & 42.5 & 46.2 & 48.0 & 50.6 & 52.0 & - & 55.2 & 58.1 & 60.0 & 62.3 & 64.0 \\
        DHKD~\citep{yang2024dual} & - & 19.7 & 23.5 & 23.4 & 23.7 & 26.8 & - & 25.6 & 34.8 & 42.7 & 49.2 & 55.2 \\
        \textbf{\texttt{DHO}} (Ours) & - & \underline{51.8} & \underline{52.4} & \underline{52.6} & \underline{53.3} & \underline{54.5} & - & \underline{61.0} & \underline{62.1} & \underline{62.5} & \underline{63.7} & \underline{64.7} \\
        \textbf{\texttt{DHO-F}} (Ours) & - & \textbf{53.7} & \textbf{54.2} & \textbf{54.8} & \textbf{56.2} & \textbf{57.7} & - & \textbf{62.3} & \textbf{63.1} & \textbf{63.9} & \textbf{65.5} & \textbf{66.8} \\
        \midrule
        \rowcolor{gray!20}\multicolumn{13}{l}{\textit{Teacher Models (Resnet-50)}} \\
        CLIP & 60.3 & - & - & - & - & - & 60.3 & - & - & - & - & - \\
        Tip-Adapter-F & - & 61.0 & 61.6 & 62.5 & 63.8 & 65.4 & - & 61.0 & 61.6 & 62.5 & 63.8 & 65.4 \\
        \bottomrule
    \end{tabular}
}
\label{tab:F1_imagenet}
\vspace{-0.15in}
\end{table}



\vspace{-0.05in}
\section{Experiments}\label{sec:experiments}
\vspace{-0.05in}


\subsection{Experimental Setups}
\label{sec:experimental_setups}
\vspace{-0.05in}

\paragraph{Datasets.}
For in-distribution evaluation, we use ImageNet~\citep{russakovsky2015imagenet} and 10 fine-grained datasets~\citep{fei2004learning,parkhi2012cats,krause20133d,nilsback2008automated,bossard2014food,maji2013fine,xiao2010sun,cimpoi2014describing,helber2019eurosat,soomro2012ucf101}. To assess out-of-distribution (OOD) generalization, we use four ImageNet variants~\citep{recht2019imagenet,wang2019learning,hendrycks2021many,hendrycks2021natural}. See \Cref{sec:appendix_datasets} for details.

\myparagraph{Baselines.}
We compare \textbf{\texttt{DHO}} with conventional single-head KD baselines;
\textbf{CE}: training only on labeled dataset $\mathcal{D}^{(l)}$ with cross entropy loss (\Cref{eq:ce}),
\textbf{KD (logit)}: on unlabeled dataset $\mathcal{D}^{(u)}$ with logit distillation (\Cref{eq:logit}), and
\textbf{KD (feature)}: on $\mathcal{D}^{(u)}$ with feature distillation~\citep{yang2024clipkd}.
We train on both $\mathcal{D}^{(l)}$ and $\mathcal{D}^{(u)}$ with \textbf{CE+KD (logit)} or \textbf{CE+KD (feature)}: combining CE with each KD variant using balancing hyperparameter $\lambda$.
We also consider dual-head KD approaches \textbf{SSKD}~\citep{he2021semi} and \textbf{DHKD}~\citep{yang2024dual}, though for different purposes as detailed in \Cref{sec:related_work}.

For in-distribution evaluation on ImageNet with low-shot settings, we compare against \textbf{self and semi-supervised learning}~\citep{chen2020big,assran2021semi,assran2022masked,cai2022semi,zheng2023simmatchv2}, \textbf{CLIP-based-training}~\citep{li2022elevater,liu2023learning}, \textbf{co-training}~\citep{rothenberger2023meta}, \textbf{KD}~\citep{chen2020big}, and \textbf{zero-shot VLMs}~\citep{zhai2023sigmoid,fang2023data}.

For OOD evaluation, we compare against VLM adaptation methods, including \textbf{VPT}~\citep{jia2022visual}, \textbf{CoOp}~\citep{zhou2022learning}, \textbf{PromptSRC}~\citep{khattak2023self}, and \textbf{CasPL}~\citep{wu2025cascade}.

\myparagraph{Implementation details.}
We evaluate across few-shot (1/2/4/8/16-shot) and low-shot (1\%/10\%) settings, treating remaining data as unlabeled.
We adopt Tip-Adapter-F~\citep{zhang2021tip} for few-shot teachers, denoting this variant as \textbf{\texttt{DHO-F}}.
On ImageNet, to avoid label leakage from supervised pretraining, students are ResNet-18 trained from scratch or DINO-pretrained ResNet-50~\citep{caron2021emerging}, paired with a CLIP ResNet-50 zero-shot teacher; on the 10 fine-grained datasets the leakage concern does not apply, so we additionally include MobileNetV2~\citep{sandler2018mobilenetv2} (see \Cref{sec:appendix_implementation_details}).
All students are trained for 20 epochs with AdamW (learning rate $1{\times}10^{-3}$, weight decay $1{\times}10^{-2}$, batch size 512 split evenly between labeled and unlabeled streams).
For a fair comparison, we use the same hyperparameters for \textbf{\texttt{DHO}} and SHO baselines (except $\alpha, \beta$); for other methods, we report the published results. 
When validation data is unavailable (\eg, ImageNet), we fix $\beta = 0.5$ across all settings. We heuristically set $\alpha = 0.4$ for zero-shot teachers and $\alpha = 0.2$ for few-shot teachers, reflecting the latter's higher reliability. In low-shot settings, we use $\alpha = 0.5$ due to increased label availability.  
See \Cref{sec:algorithms_and_implementation} for details.

\begin{figure}[!t]
    \centering
    \includegraphics[width=\textwidth, scale=0.8]{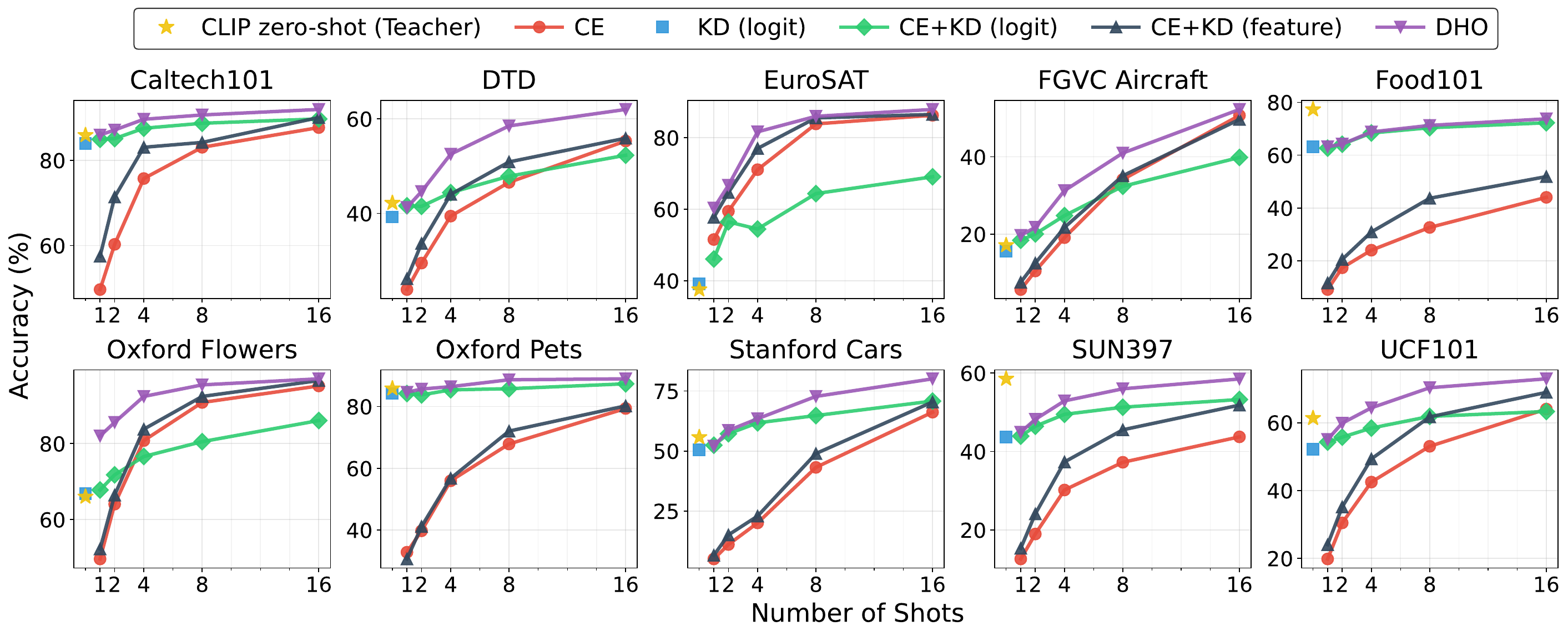}
    \vspace{-0.27in}
    \caption{\small
    Results on \textbf{10 datasets} under few-shot semi-supervision using \textbf{ResNet-18} with \textbf{zero-shot teacher}.
    }
    \label{fig:F1_10_datasets_zeroshot}

    \centering
    \includegraphics[width=\textwidth, scale=0.8]{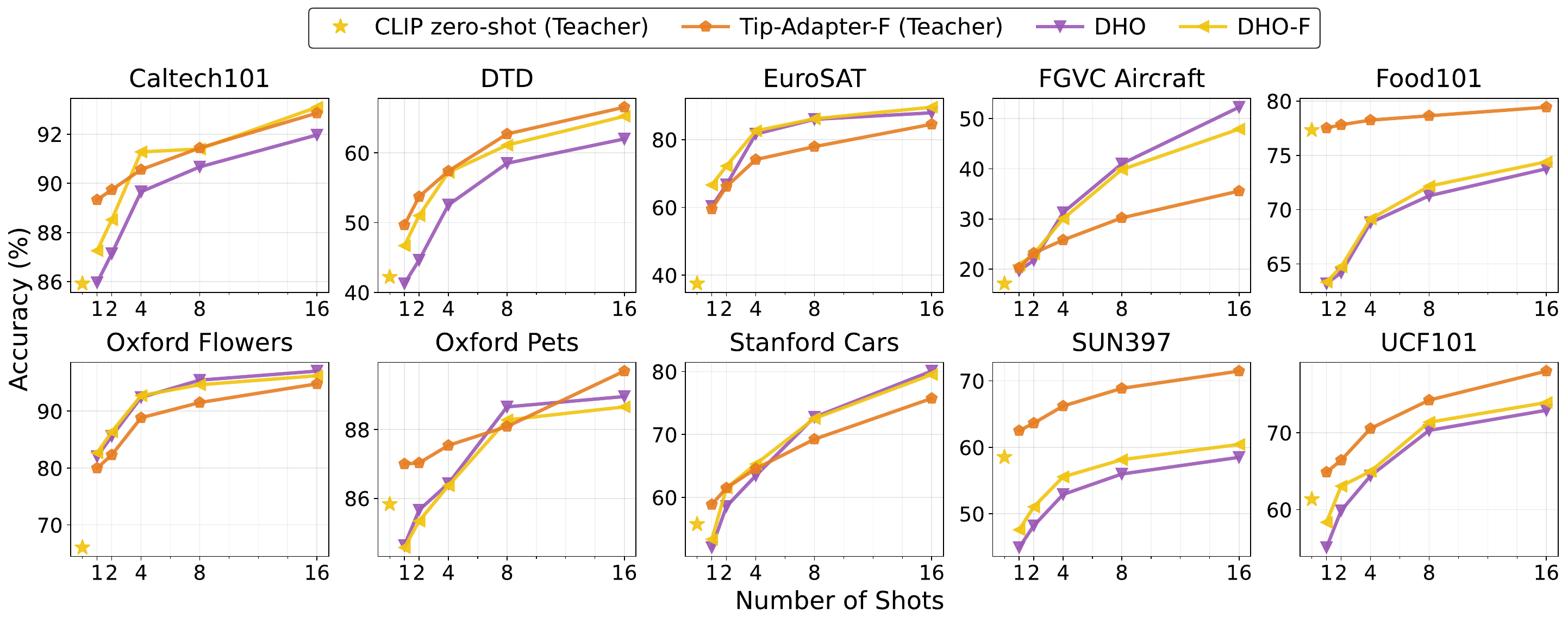}
    \vspace{-0.27in}
    \caption{\small Results on \textbf{10 datasets} using \textbf{ResNet-18} with either zero- or \textbf{few-shot teacher}.
    }
    \label{fig:F1_10_datasets_fewshot}
    \vspace{-0.2in}
\end{figure}

\vspace{-0.05in}
\subsection{Main Results}
\label{sec:experimental_results}
\vspace{-0.05in}

\paragraph{Effectiveness of \textbf{\texttt{DHO}} compared to conventional KD baselines.}
\Cref{tab:F1_imagenet} presents results on ImageNet under few-shot semi-supervision using ResNet-18/ResNet-50 student models and ResNet-50 VLM teachers.
\textbf{\texttt{DHO}} \textbf{consistently outperforms} all baselines across all settings. Other dual-head methods (\ie, SSKD/DHKD), not designed for few-shot semi-supervision, even underperform CE+KD on logits.
Notably, with ResNet-50, \textbf{\texttt{DHO}} \textbf{outperforms the teacher in every few-shot setting} (\eg, +0.7/1.8/2.2/3.4/4.4\% or +1.3/1.5/1.4/1.7/1.4\% with zero-shot or few-shot teachers).

\vspace{-0.02in}
Next, we evaluate \textbf{\texttt{DHO}} on 10 additional datasets using the ResNet-18 student and ResNet-50 VLM teachers.

\begin{wraptable}{r}{0.48\textwidth}
\centering
\caption{\small
    Results on \textbf{ImageNet under low-shot settings}. For CT/MCT methods, numbers in parentheses indicate the number of architectures for co-training.
}
\label{tab:semisup_results}
\vspace{-0.1in}
\scriptsize
\setlength{\tabcolsep}{1.5pt}
\resizebox{\linewidth}{!}{
\begin{tabular}{l@{\hspace{1.5pt}}l@{\hspace{1pt}}r@{\hspace{2pt}}r@{\hspace{2pt}}r}
    \toprule
    \textbf{Method} & \textbf{Architecture} & \textbf{Params} & \textbf{1\%} & \textbf{10\%} \\
    &  & \textbf{(M)} & \textbf{(\%)} & \textbf{(\%)} \\
    \midrule
    \rowcolor{gray!20}\multicolumn{5}{l}{\textit{Self and Semi-supervised Learning}} \\
    MSN & ViT-B/4 & 86 & 75.7 & 80.2 \\
    Semi-ViT & ViT-L/14 & 307 & 77.3 & 83.3 \\
    Semi-ViT & ViT-H/14 & 632 & 80.0 & 84.3 \\
    \midrule
    \rowcolor{gray!20}\multicolumn{5}{l}{\textit{CLIP-based Training}} \\
    CLIP & ViT-B/16 & 86 & 74.3 & 80.4 \\
    REACT & ViT-B/16 & 86 & 76.1 & 80.8 \\
    REACT (Gated-Image) & ViT-B/16 & 129 & 77.4 & 81.8 \\
    CLIP & ViT-L/14 & 304 & 80.5 & 84.7 \\
    REACT & ViT-L/14 & 304 & \underline{81.6} & 85.1 \\
    REACT (Gated-Image) & ViT-L/14 & 380 & \underline{81.6} & 85.0 \\
    \midrule
    \rowcolor{gray!20}\multicolumn{5}{l}{\textit{Co-training based Methods}} \\
    CT & Multi-arch (2) & 608 & 80.1 & 85.1 \\
    MCT & Multi-arch (2) & 608 & 80.7 & 85.2 \\
    CT & Multi-arch (4) & 1071 & 80.0 & 84.8 \\
    MCT & Multi-arch (4) & 1071 & 80.5 & \underline{85.8} \\
    \midrule
    \rowcolor{gray!20}\multicolumn{5}{l}{\textit{Knowledge Distillation}} \\
    SimCLR v2 distill & ResNet-50 (2×+SK) & 140 & 75.9 & 80.2 \\
    SimCLR v2 self-distill & ResNet-154 (3×+SK) & 795 & 76.6 & 80.9 \\
    CE + KD (logit) & ViT-B/16 & 86 & 79.8 & 80.4 \\
    CE + KD (logit) & ViT-L/14 & 304 & 83.1 & 83.6 \\
    \textbf{\texttt{DHO}}  & ViT-B/16 & 86 & \underline{81.6} & 82.8 \\
    \textbf{\texttt{DHO}}  & ViT-L/14 & 304 & \textbf{84.6} & \textbf{85.9} \\
    \midrule
    \rowcolor{gray!20}\multicolumn{5}{l}{\textit{Zero-shot VLMs}} \\
    SigLIP & ViT-SO400M/14 & 400 & 83.1 & - \\
    DFN & ViT-H/14 & 632 & 84.4 & - \\
    \bottomrule
\end{tabular}
}
\vspace{-0.15in}
\end{wraptable}

On 10 additional datasets, \textbf{\texttt{DHO}} also \textbf{consistently outperforms all baselines} (\Cref{fig:F1_10_datasets_zeroshot}); the few-shot teacher is more effective than the zero-shot one (\Cref{fig:F1_10_datasets_fewshot}), and \textbf{the ResNet-18 student trained with \texttt{DHO} beats the ResNet-50 teacher} in most cases.

\myparagraph{\textbf{\texttt{DHO}} achieves SoTA performance on ImageNet under low-shot settings.}
On ImageNet under 1\%/10\% labeled data, \textbf{\texttt{DHO}} with ViT-L/14 \textbf{surpasses the previous SoTA by +3.0\%/+0.1\% using fewer parameters (76M/767M)} (\Cref{tab:semisup_results}); all baselines are taken from published papers except \textbf{\texttt{DHO}} and CE+KD (logit). Notably, CE+KD (logit) outperforms semi-supervised methods of the same size, showing they are \textbf{suboptimal compared to methods leveraging VLM representations}.

\myparagraph{\textbf{\texttt{DHO}} achieves SoTA performance on ImageNet OOD benchmarks.}
\textbf{\texttt{DHO}} \textbf{consistently improves} VLM adaptation methods (\Cref{tab:domain_generalization_adaptation}) and outperforms CasPL across in-distribution (Val) and OOD (V2/Sketch/R/A) benchmarks, \textbf{establishing new SoTA OOD results} in semi-supervised setting---joint labeled-supervision + teacher-distillation is more effective than CasPL's sequential approach, consistent with \textbf{\texttt{DHO}}'s gradient-conflict resolution. See \Cref{sec:appendix_robustness_full} for fully trained models.

\myparagraph{Additional results.}
\textbf{\texttt{DHO}} also outperforms \textbf{PCGrad}~\citep{yu2020gradient} (\Cref{sec:appendix_pcgrad}); see \Cref{sec:additional_experiments} for \textbf{MobileNetV2} and \textbf{category-aware KD methods}~\citep{zhao2022decoupled,lv2024wasserstein}.

\vspace{-0.05in}
\subsection{Analysis}
\label{sec:experimental_analysis}
\vspace{-0.05in}

\begin{wraptable}{r}{0.5\textwidth}
    \vspace{-0.2in}
    \scriptsize
    \caption{\small
        \textbf{Inference overhead} using RTX 4090.
    }
    \label{tab:computation_cost}
    \vspace{-0.1in}
    \resizebox{0.5\textwidth}{!}{%
    \begin{tabular}{lrrr}
        \toprule
        \textbf{Model} & \textbf{Params} & \textbf{FLOPs} & \textbf{Throughput} \\
         & \textbf{(M)} & \textbf{(G)} & \textbf{(im/s)} \\
        \midrule
        ResNet-18 & 11.69 & 1.83 & 3525.7 \\
        \rowcolor{gray!20} + \textbf{\texttt{DHO}} & 12.20 {\scriptsize\color{red}(+4.4\%)} & 1.83 {\scriptsize\color{red}(+0.0\%)} & 3518.6 {\scriptsize\color{red}(-0.20\%)} \\
        \midrule
        ResNet-50 & 25.56 & 4.14 & 1018.4 \\
        \rowcolor{gray!20} + \textbf{\texttt{DHO}} & 27.61 {\scriptsize\color{red}(+8.0\%)} & 4.15 {\scriptsize\color{red}(+0.2\%)} & 1016.4 {\scriptsize\color{red}(-0.20\%)} \\
        \bottomrule
    \end{tabular}%
    }
    \vspace{-0.15in}
\end{wraptable}

\paragraph{Minimal computational overhead of \textbf{\texttt{DHO}}.}
As shown in \Cref{tab:computation_cost}, \textbf{\texttt{DHO}} introduces \textbf{negligible computational overhead with minimal parameter increase} on the ImageNet with 1000 classes, which can be further reduced down for datasets with fewer classes ($C<1000$).
We further provide inference computational overhead for other models in \Cref{tab:full_computation_cost}. The gain therefore is not from added capacity nor pseudo-labeling: with a specialist teacher (no gradient conflict), \textbf{\texttt{DHO}} matches SHO (\Cref{sec:appendix_motivation}).

\begin{table}[t]
\vspace{-0.25in}
    \centering
    \caption{
        Results on \textbf{ImageNet and its distribution-shifted variants}. $\dagger$: backbone frozen during training. \textbf{\texttt{DHO}} \textbf{consistently improves} VLM adaptation methods and \textbf{achieves SoTA results} on OOD benchmarks.
    }
    \vspace{0pt}
    \label{tab:domain_generalization_adaptation}
    \resizebox{0.9\textwidth}{!}{
        \begin{tabular}{lcccccccccc}
            \toprule
            \textbf{Method} & \textbf{Labeled Data} & \textbf{Teacher Model} & \textbf{Unlabeled Data} & \textbf{Val} & \textbf{V2} & \textbf{Sketch} & \textbf{R} & \textbf{A} & \textbf{Avg} \\
            \midrule
            \rowcolor{gray!20}\multicolumn{10}{l}{\textit{CLIP Zero-shot}} \\
            CLIP & - & - & - & 66.7 & 60.8 & 46.2 & 74.0 & 47.8 & 59.1 \\
            \midrule
            \rowcolor{gray!20}\multicolumn{10}{l}{\textit{Linear Evaluation}} \\
            Linear Evaluation & 1\% & - & - & 72.8 & 64.2 & 47.0 & 74.6 & 48.6 & 61.4 \\
            \textbf{\texttt{DHO}}$^\dagger$ & 1\% & \checkmark & - & \underline{73.4} & \underline{65.3} & \underline{48.3} & \underline{75.8} & \underline{49.4} & \underline{62.4} \\
            \textbf{\texttt{DHO}}$^\dagger$ & 1\% & \checkmark & \checkmark & \textbf{74.6} & \textbf{66.4} & \textbf{49.2} & \textbf{76.1} & \textbf{49.8} & \textbf{63.2} \\
            \midrule
            \rowcolor{gray!20}\multicolumn{10}{l}{\textit{Visual Prompt Tuning}} \\
            VPT & 1.25\% & - & - & 73.6 & 64.6 & 47.7 & 75.2 & 48.7 & 62.0 \\
            VPT+\textbf{\texttt{DHO}} & 1.25\% & \checkmark & - & \underline{73.6} & \underline{65.3} & \underline{48.7} & \underline{75.9} & \underline{49.1} & \underline{62.5} \\
            VPT+\textbf{\texttt{DHO}} & 1.25\% & \checkmark & \checkmark & \textbf{75.1} & \textbf{66.8} & \textbf{50.0} & \textbf{76.9} & \textbf{50.5} & \textbf{63.9} \\
            \midrule
            \rowcolor{gray!20}\multicolumn{10}{l}{\textit{VLM Text-encoder Prompt Tuning}} \\
            CoOp & 1.25\% & - & - & 71.5 & 64.2 & 48.0 & 75.2 & \underline{49.7} & 61.7 \\
            CoOp+CasPL & 1.25\% & \checkmark & \checkmark & 71.9 & 64.3 & 48.3 & 76.0 & - & - \\
            CoOp+\textbf{\texttt{DHO}} & 1.25\% & \checkmark & - & \underline{72.8} & \underline{65.5} & \underline{49.3} & \underline{76.4} & 49.5 & \underline{62.7} \\
            CoOp+\textbf{\texttt{DHO}} & 1.25\% & \checkmark & \checkmark & \textbf{73.4} & \textbf{66.2} & \textbf{49.5} & \textbf{77.0} & \textbf{50.5} & \textbf{63.3} \\
            \midrule
            \rowcolor{gray!20}\multicolumn{10}{l}{\textit{VLM Multimodal Prompt Tuning}} \\
            PromptSRC & 1.25\% & - & - & 71.3 & 64.4 & 49.6 & 77.8 & 50.9 & 62.8 \\
            PromptSRC+CasPL & 1.25\% & \checkmark & \checkmark & 72.8 & \underline{65.7} & \underline{49.7} & \underline{77.9} & - & - \\
            PromptSRC+\textbf{\texttt{DHO}} & 1.25\% & \checkmark & - & \underline{73.0} & 65.3 & 49.5 & 77.8 & \textbf{51.3} & \underline{63.4} \\
            PromptSRC+\textbf{\texttt{DHO}} & 1.25\% & \checkmark & \checkmark & \textbf{73.6} & \textbf{66.1} & \textbf{49.8} & \textbf{78.1} & \underline{51.0} & \textbf{63.7} \\
            \bottomrule
        \end{tabular}
    }
    \vspace{-0.05in}
\end{table}

\begin{figure}[t]
    \vspace{-0.15in}
    \centering
    \includegraphics[width=\textwidth, scale=0.8]{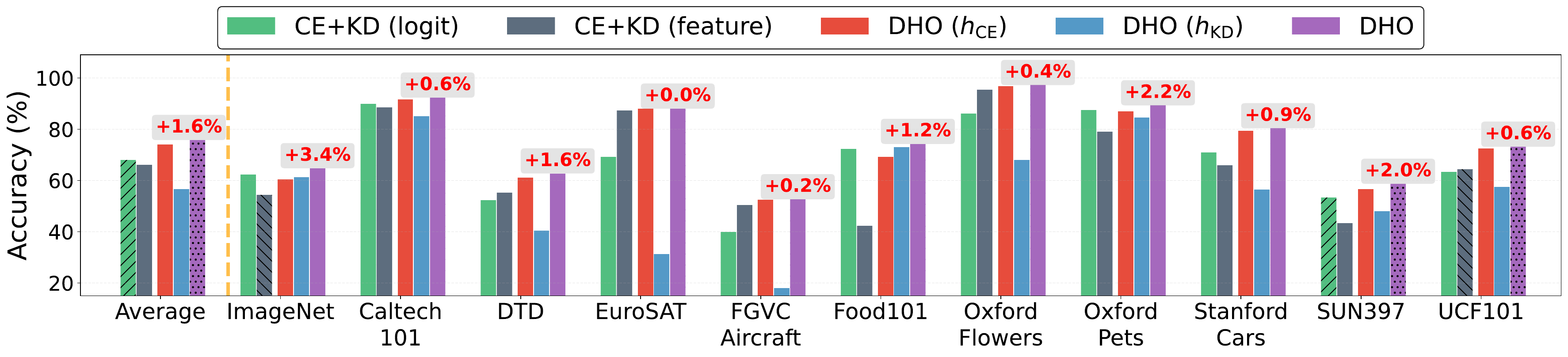}
    \vspace{-0.25in}
    \caption{\small Results of ablation studies on \textbf{dual-heads interpolation strategy} in \Cref{eq:inference} of \textbf{\texttt{DHO}}.
    }
    \label{fig:F3_head_interpolation}
    \vspace{-0.25in}
\end{figure}

\begin{wrapfigure}{r}{0.53\textwidth}
    \vspace{-0.1in}
    \centering
    \vspace{-0.1in} 
    \includegraphics[width=\linewidth]{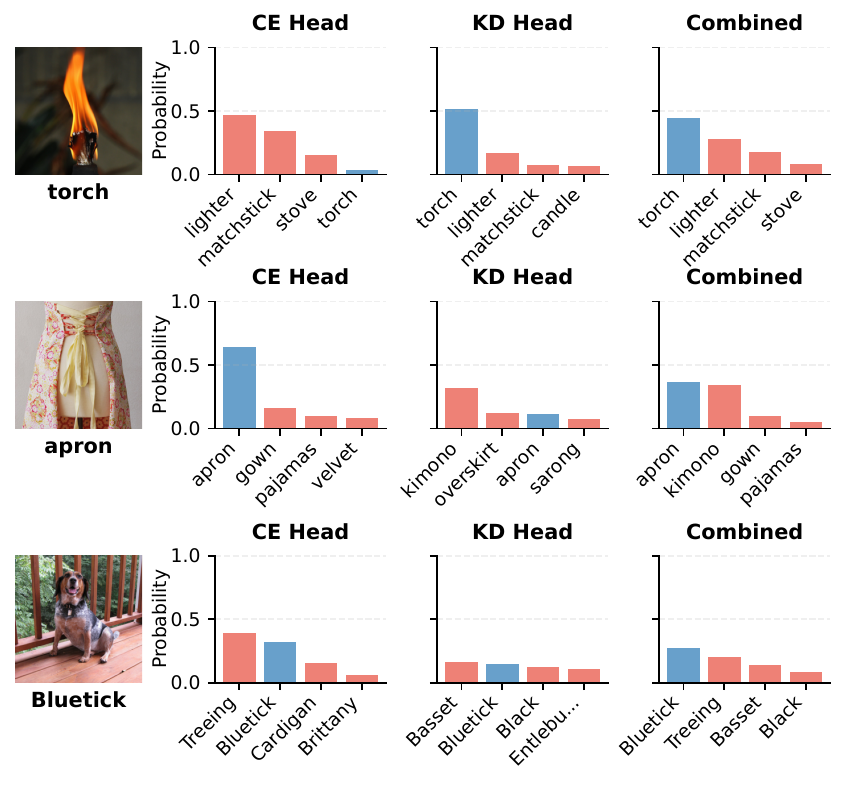}
    \vspace{-0.3in}
    \caption{\small Qualitative results on challenging cases.}
    \label{fig:qualitative_results}
    \vspace{-0.1in}
\end{wrapfigure}

\myparagraph{Effectiveness of dual-head interpolation.}
We evaluate the effectiveness of dual-head interpolation (\Cref{eq:inference}) by comparing \textbf{\texttt{DHO}} with CE+KD (logit), CE+KD (feature), and ablations \textbf{\texttt{DHO}} ($h_{\text{CE}}$) and \textbf{\texttt{DHO}} ($h_{\text{KD}}$), which predict using only one head (\ie, $\alpha=1$ or $0$).
As shown in \Cref{fig:F3_head_interpolation}, \textbf{\texttt{DHO}} outperforms \textbf{\texttt{DHO}} ($h_{\text{CE}}$) by an average of 1.6\% across 11 datasets, with a maximum gain of +3.4\% on ImageNet and no degradation on any dataset.
Since $\alpha$ and $\beta$ are inference-time hyperparameters, dual-head interpolation introduces minimal overhead while consistently \textbf{improving or maintaining performance}.
We also investigate the effectiveness of our adaptive weighting strategy in \Cref{sec:appendix_adaptive_weighting}.
\Cref{fig:qualitative_results} illustrates three challenging examples:
CE head ($h_{\text{CE}}$) is correct in the first, KD head ($h_{\text{KD}}$) in the second, and both fail in the third case, yet the proposed combined prediction is correct---demonstrating the ability of \textbf{\texttt{DHO}} to resolve individual head failures.
See \Cref{sec:appendix_qualitative_results} for additional analysis on theses challenging examples.

\begin{wrapfigure}{r}{0.49\textwidth}
    \vspace{-0.15in}
    \centering
    \captionof{table}{\small Linear evaluation results.}
    \vspace{-0.1in}
    \begin{minipage}[t]{0.49\textwidth}
        \centering
        \scriptsize
        \resizebox{\textwidth}{!}{%
        \begin{tabular}{l@{\hspace{5pt}}c@{\hspace{5pt}}c@{\hspace{5pt}}}
            \toprule
            \textbf{Method} & \textbf{Top-1 (\%)} & \textbf{Top-5 (\%)} \\
            \midrule
            CE+KD (feature)~\citep{yang2024clipkd} & 62.3 & 85.0 \\
            CE+KD (logit)~\citep{chen2020big} & \underline{66.2} & \underline{88.8} \\
            \rowcolor{gray!20} \textbf{\texttt{DHO}} & \textbf{67.1} & \textbf{89.3} \\
            \bottomrule
        \end{tabular}
        }
        \label{tab:linear_evaluation}
    \end{minipage}
    \hfill
    \begin{minipage}[t]{0.49\textwidth}
        \centering
        \includegraphics[width=\linewidth]{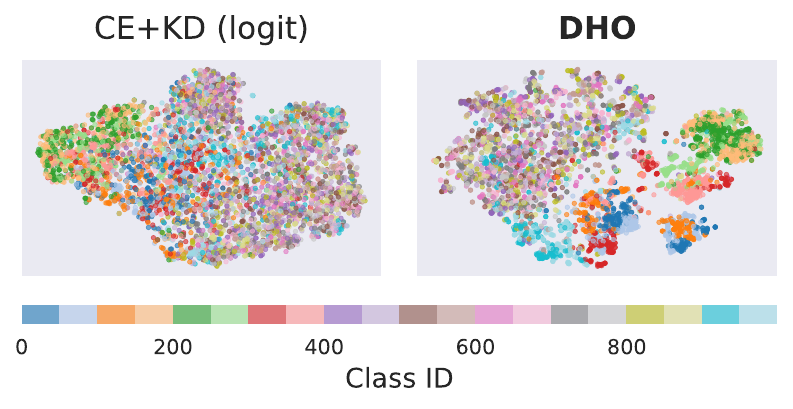}
        \vspace{-0.3in}
        \captionof{figure}{\small t-SNE visualization.}
        \label{fig:tsne}
    \end{minipage}    \label{fig:F2_feature_representation}
    \vspace{-0.18in}
\end{wrapfigure}

\myparagraph{Enhanced feature representation of \textbf{\texttt{DHO}}.}
To validate our claim that mitigating gradient conflicts improves feature representations, we evaluate features using the standard \emph{linear evaluation} protocol~\citep{chen2020big}. We train CE+KD (feature), CE+KD (logit), and \textbf{\texttt{DHO}} under the 16-shot semi-supervised setting on ImageNet, freeze the feature extractor $g$, and train a new prediction linear head $h_{\text{LE}}$ on the top of $g$ using fully labeled data.
As shown in \Cref{tab:linear_evaluation}, \textbf{\texttt{DHO}} achieves higher Top-1 and Top-5 accuracy than other methods (\eg, +0.9\% and 0.5\%, respectively).
To further assess feature quality, we visualize embeddings $z$ using t-SNE~\citep{van2008visualizing} in \Cref{fig:tsne}.
Compared to the CE+KD (logit) baseline, \textbf{\texttt{DHO}} produces more compact and class-separated feature clusters.
These results support our claim that \textbf{\texttt{DHO}} \textbf{enhances feature representations} by mitigating gradient conflicts; this improvement leads to better performance of \textbf{\texttt{DHO}} compared to SHO, as discussed through \Cref{tab:F1_imagenet}, \Cref{fig:F1_10_datasets_zeroshot,fig:F1_10_datasets_fewshot,fig:mobilenet_baseline,fig:mobilenet_fewshot}.

\begin{wraptable}{r}{0.4\textwidth}
    \vspace{-0.2in}
    \caption{
        \small Results on \textbf{language-aware initialization} and \textbf{KD-head alignment} for VLM students on ImageNet with 1\% labeled data.
    }
    \vspace{-0.1in}
    \centering
    \scriptsize
    \setlength{\tabcolsep}{3pt} 
    \resizebox{0.4\textwidth}{!}{
        \begin{tabular}{ccr}
            \toprule
            \textbf{Init ($h_{\text{CE}}$/$h_{\text{KD}}$).} & \textbf{Align.} & \textbf{Accuracy (\%)} \\
            \midrule
            \xmark\ /\ \xmark & \xmark & 78.3 \\
            \cmark\ /\ \xmark & \xmark & 78.5 {\scriptsize\color{blue}(+0.2)} \\
            \cmark\ /\ \cmark & \xmark & \underline{78.6} {\scriptsize\color{blue}(+0.3)} \\
            \rowcolor{gray!20}\cmark\ /\ \cmark & \cmark & \textbf{78.7} {\scriptsize\color{blue}(+0.4)} \\
            \bottomrule
        \end{tabular}
    }            
    \label{tab:language_init}
    \vspace{-0.1in}
\end{wraptable}

\myparagraph{Effectiveness of Init. and Align in \Cref{sec:dual_head_interpolation}.}
To assess the effect of the proposed language-aware initialization and KD-head alignment in \Cref{sec:dual_head_interpolation}, we run ablations that apply language-aware initialization (\textbf{Init.}) to either $h_{\text{CE}}$ or $h_{\text{KD}}$, and optionally enable KD-head alignment (\textbf{Align.}). We evaluate on ImageNet with 1\% labels using a ViT-B/16 student and ViT-L/14 teacher for computational efficiency. As shown in \Cref{tab:language_init}, applying Init. to either head independently improves accuracy (\eg, +0.2/+0.3\%), and adding Align. yields additional gains (\eg, +0.4\%).

\begin{wrapfigure}{r}{0.43\textwidth}
    \vspace{-0.2in}
    \centering
    \includegraphics[width=\linewidth]{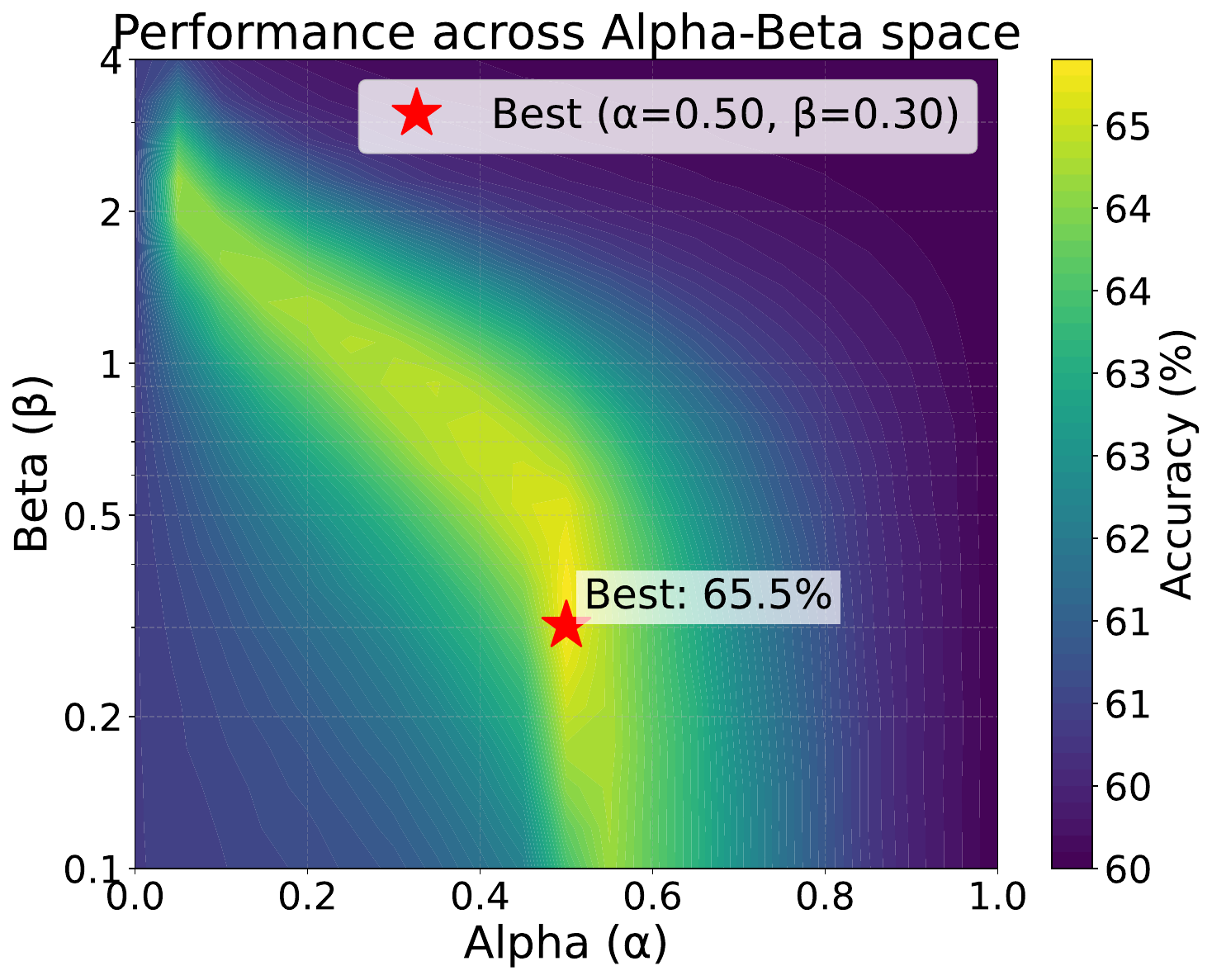}
    \vspace{-0.25in}
    \caption{\small \textbf{Grid search results for $\alpha$ and $\beta$}.}
    \label{fig:ablation_alpha_beta}
    \vspace{-0.2in}
\end{wrapfigure}

\myparagraph{Effect of $\alpha$ and $\beta$.} In \Cref{main_thm:inference_equiv}, we show that under mild assumptions (\Cref{main_assump:convergence}), \textbf{\texttt{DHO}} \(\epsilon\)-approximates the SHO baseline, namely CE+KD (logit), in the \(\ell_1\) norm by setting \(\alpha{=}\lambda\) and \(\beta{=}1\). This enables \textbf{\texttt{DHO}} to emulate SHO hyperparameter tuning at inference time without retraining. To study how \(\alpha\) and \(\beta\) affect performance, we visualize a grid search on ImageNet with ResNet-50 under the 16-shot setting (\Cref{fig:ablation_alpha_beta}). With balanced heads ($\alpha{\approx}0.5$), performance is stable for $\beta\in[0.1,1]$ and degrades only at extreme $\alpha$ values regardless of $\beta$. The accuracy peaks at 65.5\% near $(\alpha,\beta){\approx}(0.5,0.3)$, close to our heuristic setting (64.7\%); even with fixed $(\alpha,\beta){=}(0.5,1)$, \textbf{\texttt{DHO}} exceeds SHO by +1.6 to +9.5\% across 1--16 shots (\Cref{tab:dho_fixed_vs_searched}).

\vspace{-0.05in}
\section{Conclusion, Limitation, and Future Work}
\label{sec:conclusion_future_work}
\vspace{-0.05in}

We identify \emph{gradient conflicts} between teacher and labeled-data signals as a fundamental obstacle to transferring VLM zero-/few-shot capabilities to task-specific models under limited supervision.
We propose \textbf{\texttt{DHO}} (\textbf{\texttt{D}}ual-\textbf{\texttt{H}}ead \textbf{\texttt{O}}ptimization), a \emph{plug-and-play} framework that decouples the two signals into separate classification heads, \textbf{mitigating gradient conflicts} in both heads and the shared feature extractor while \textbf{enabling post-training hyperparameter tuning} via linear interpolation of head outputs.
Across 15 datasets, \textbf{\texttt{DHO}} \textbf{outperforms conventional KD methods} and achieves \textbf{state-of-the-art} on ImageNet semi-supervised learning and on OOD benchmarks; see \Cref{sec:appendix_limitations} for limitations and future work.


\bibliographystyle{unsrtnat}
\bibliography{reference}

\begin{thebibliography}{126}
\providecommand{\natexlab}[1]{#1}
\providecommand{\url}[1]{\texttt{#1}}
\expandafter\ifx\csname urlstyle\endcsname\relax
  \providecommand{\doi}[1]{doi: #1}\else
  \providecommand{\doi}{doi: \begingroup \urlstyle{rm}\Url}\fi

\bibitem[Radford et~al.(2021)Radford, Kim, Hallacy, Ramesh, Goh, Agarwal,
  Sastry, Askell, Mishkin, Clark, et~al.]{radford2021learning}
Alec Radford, Jong~Wook Kim, Chris Hallacy, Aditya Ramesh, Gabriel Goh,
  Sandhini Agarwal, Girish Sastry, Amanda Askell, Pamela Mishkin, Jack Clark,
  et~al.
\newblock Learning transferable visual models from natural language
  supervision.
\newblock In \emph{International conference on machine learning}, pages
  8748--8763. PMLR, 2021.

\bibitem[Jia et~al.(2021)Jia, Yang, Xia, Chen, Parekh, Pham, Le, Sung, Li, and
  Duerig]{jia2021scaling}
Chao Jia, Yinfei Yang, Ye~Xia, Yi-Ting Chen, Zarana Parekh, Hieu Pham, Quoc Le,
  Yun-Hsuan Sung, Zhen Li, and Tom Duerig.
\newblock Scaling up visual and vision-language representation learning with
  noisy text supervision.
\newblock In \emph{International conference on machine learning}, pages
  4904--4916. PMLR, 2021.

\bibitem[Li et~al.(2022)Li, Liu, Li, Zhang, Aneja, Yang, Jin, Hu, Liu, Lee,
  et~al.]{li2022elevater}
Chunyuan Li, Haotian Liu, Liunian Li, Pengchuan Zhang, Jyoti Aneja, Jianwei
  Yang, Ping Jin, Houdong Hu, Zicheng Liu, Yong~Jae Lee, et~al.
\newblock Elevater: A benchmark and toolkit for evaluating language-augmented
  visual models.
\newblock \emph{Advances in Neural Information Processing Systems},
  35:\penalty0 9287--9301, 2022.

\bibitem[Huang et~al.(2024)Huang, Shakeri, Dolz, Boudiaf, Bahig, and
  Ben~Ayed]{huang2024lp++}
Yunshi Huang, Fereshteh Shakeri, Jose Dolz, Malik Boudiaf, Houda Bahig, and
  Ismail Ben~Ayed.
\newblock Lp++: A surprisingly strong linear probe for few-shot clip.
\newblock In \emph{Proceedings of the IEEE/CVF Conference on Computer Vision
  and Pattern Recognition}, pages 23773--23782, 2024.

\bibitem[Zhang et~al.(2021)Zhang, Fang, Zhang, Gao, Li, Dai, Qiao, and
  Li]{zhang2021tip}
Renrui Zhang, Rongyao Fang, Wei Zhang, Peng Gao, Kunchang Li, Jifeng Dai,
  Yu~Qiao, and Hongsheng Li.
\newblock Tip-adapter: Training-free clip-adapter for better vision-language
  modeling.
\newblock \emph{arXiv preprint arXiv:2111.03930}, 2021.

\bibitem[Gao et~al.(2024)Gao, Geng, Zhang, Ma, Fang, Zhang, Li, and
  Qiao]{gao2024clip}
Peng Gao, Shijie Geng, Renrui Zhang, Teli Ma, Rongyao Fang, Yongfeng Zhang,
  Hongsheng Li, and Yu~Qiao.
\newblock Clip-adapter: Better vision-language models with feature adapters.
\newblock \emph{International Journal of Computer Vision}, 132\penalty0
  (2):\penalty0 581--595, 2024.

\bibitem[Zhou et~al.(2022{\natexlab{a}})Zhou, Yang, Loy, and
  Liu]{zhou2022learning}
Kaiyang Zhou, Jingkang Yang, Chen~Change Loy, and Ziwei Liu.
\newblock Learning to prompt for vision-language models.
\newblock \emph{International Journal of Computer Vision}, 130\penalty0
  (9):\penalty0 2337--2348, 2022{\natexlab{a}}.

\bibitem[Zhou et~al.(2022{\natexlab{b}})Zhou, Yang, Loy, and
  Liu]{zhou2022conditional}
Kaiyang Zhou, Jingkang Yang, Chen~Change Loy, and Ziwei Liu.
\newblock Conditional prompt learning for vision-language models.
\newblock In \emph{Proceedings of the IEEE/CVF conference on computer vision
  and pattern recognition}, pages 16816--16825, 2022{\natexlab{b}}.

\bibitem[Lafon et~al.(2025)Lafon, Ramzi, Rambour, Audebert, and
  Thome]{lafon2025gallop}
Marc Lafon, Elias Ramzi, Cl{\'e}ment Rambour, Nicolas Audebert, and Nicolas
  Thome.
\newblock Gallop: Learning global and local prompts for vision-language models.
\newblock In \emph{European Conference on Computer Vision}, pages 264--282.
  Springer, 2025.

\bibitem[Sohn et~al.(2020)Sohn, Berthelot, Carlini, Zhang, Zhang, Raffel,
  Cubuk, Kurakin, and Li]{sohn2020fixmatch}
Kihyuk Sohn, David Berthelot, Nicholas Carlini, Zizhao Zhang, Han Zhang,
  Colin~A Raffel, Ekin~Dogus Cubuk, Alexey Kurakin, and Chun-Liang Li.
\newblock Fixmatch: Simplifying semi-supervised learning with consistency and
  confidence.
\newblock \emph{Advances in neural information processing systems},
  33:\penalty0 596--608, 2020.

\bibitem[Assran et~al.(2021)Assran, Caron, Misra, Bojanowski, Joulin, Ballas,
  and Rabbat]{assran2021semi}
Mahmoud Assran, Mathilde Caron, Ishan Misra, Piotr Bojanowski, Armand Joulin,
  Nicolas Ballas, and Michael Rabbat.
\newblock Semi-supervised learning of visual features by non-parametrically
  predicting view assignments with support samples.
\newblock In \emph{Proceedings of the IEEE/CVF International Conference on
  Computer Vision}, pages 8443--8452, 2021.

\bibitem[Cai et~al.(2022)Cai, Ravichandran, Favaro, Wang, Modolo, Bhotika, Tu,
  and Soatto]{cai2022semi}
Zhaowei Cai, Avinash Ravichandran, Paolo Favaro, Manchen Wang, Davide Modolo,
  Rahul Bhotika, Zhuowen Tu, and Stefano Soatto.
\newblock Semi-supervised vision transformers at scale.
\newblock \emph{Advances in Neural Information Processing Systems},
  35:\penalty0 25697--25710, 2022.

\bibitem[Zheng et~al.(2023)Zheng, You, Huang, Luo, Wang, Qian, and
  Xu]{zheng2023simmatchv2}
Mingkai Zheng, Shan You, Lang Huang, Chen Luo, Fei Wang, Chen Qian, and Chang
  Xu.
\newblock Simmatchv2: Semi-supervised learning with graph consistency.
\newblock In \emph{Proceedings of the IEEE/CVF International Conference on
  Computer Vision}, pages 16432--16442, 2023.

\bibitem[Liu et~al.(2023{\natexlab{a}})Liu, Son, Yang, Liu, Gao, Lee, and
  Li]{liu2023learning}
Haotian Liu, Kilho Son, Jianwei Yang, Ce~Liu, Jianfeng Gao, Yong~Jae Lee, and
  Chunyuan Li.
\newblock Learning customized visual models with retrieval-augmented knowledge.
\newblock In \emph{Proceedings of the IEEE/CVF Conference on Computer Vision
  and Pattern Recognition}, pages 15148--15158, 2023{\natexlab{a}}.

\bibitem[Hinton(2015)]{hinton2015distilling}
Geoffrey Hinton.
\newblock Distilling the knowledge in a neural network.
\newblock \emph{arXiv preprint arXiv:1503.02531}, 2015.

\bibitem[Yang et~al.(2024{\natexlab{a}})Yang, An, Huang, Bi, Yu, Yang, Diao,
  and Xu]{yang2024clipkd}
Chuanguang Yang, Zhulin An, Libo Huang, Junyu Bi, Xinqiang Yu, Han Yang, Boyu
  Diao, and Yongjun Xu.
\newblock Clip-kd: An empirical study of clip model distillation.
\newblock In \emph{Proceedings of the IEEE/CVF Conference on Computer Vision
  and Pattern Recognition}, pages 15952--15962, 2024{\natexlab{a}}.

\bibitem[Vasu et~al.(2024)Vasu, Pouransari, Faghri, Vemulapalli, and
  Tuzel]{vasu2024mobileclip}
Pavan Kumar~Anasosalu Vasu, Hadi Pouransari, Fartash Faghri, Raviteja
  Vemulapalli, and Oncel Tuzel.
\newblock Mobileclip: Fast image-text models through multi-modal reinforced
  training.
\newblock In \emph{Proceedings of the IEEE/CVF Conference on Computer Vision
  and Pattern Recognition}, pages 15963--15974, 2024.

\bibitem[Udandarao et~al.(2024)Udandarao, Parthasarathy, Naeem, Evans, Albanie,
  Tombari, Xian, Tonioni, and H{\'e}naff]{udandarao2024active}
Vishaal Udandarao, Nikhil Parthasarathy, Muhammad~Ferjad Naeem, Talfan Evans,
  Samuel Albanie, Federico Tombari, Yongqin Xian, Alessio Tonioni, and
  Olivier~J H{\'e}naff.
\newblock Active data curation effectively distills large-scale multimodal
  models.
\newblock \emph{arXiv preprint arXiv:2411.18674}, 2024.

\bibitem[Yang et~al.(2024{\natexlab{b}})Yang, Gu, An, Jiang, Dai, Feng, Cai,
  and Deng]{yang2024clipcid}
Kaicheng Yang, Tiancheng Gu, Xiang An, Haiqiang Jiang, Xiangzi Dai, Ziyong
  Feng, Weidong Cai, and Jiankang Deng.
\newblock Clip-cid: Efficient clip distillation via cluster-instance
  discrimination.
\newblock \emph{arXiv preprint arXiv:2408.09441}, 2024{\natexlab{b}}.

\bibitem[Vemulapalli et~al.(2024)Vemulapalli, Pouransari, Faghri, Mehta,
  Farajtabar, Rastegari, and Tuzel]{knowledgetransfer2024}
Raviteja Vemulapalli, Hadi Pouransari, Fartash Faghri, Sachin Mehta, Mehrdad
  Farajtabar, Mohammad Rastegari, and Oncel Tuzel.
\newblock Knowledge transfer from vision foundation models for efficient
  training of small task-specific models.
\newblock In \emph{International Conference on Machine Learning (ICML)}, 2024.

\bibitem[Wu et~al.(2024)Wu, Zhang, Li, Chen, Liang, Yang, and
  Li]{wu2025cascade}
Ge~Wu, Xin Zhang, Zheng Li, Zhaowei Chen, Jiajun Liang, Jian Yang, and Xiang
  Li.
\newblock Cascade prompt learning for vision-language model adaptation.
\newblock In \emph{European Conference on Computer Vision}, pages 304--321.
  Springer, 2024.

\bibitem[Chen et~al.(2020)Chen, Kornblith, Swersky, Norouzi, and
  Hinton]{chen2020big}
Ting Chen, Simon Kornblith, Kevin Swersky, Mohammad Norouzi, and Geoffrey~E
  Hinton.
\newblock Big self-supervised models are strong semi-supervised learners.
\newblock \emph{Advances in neural information processing systems},
  33:\penalty0 22243--22255, 2020.

\bibitem[Yu et~al.(2020{\natexlab{a}})Yu, Kumar, Gupta, Levine, Hausman, and
  Finn]{gradient_conflict1}
Tianhe Yu, Saurabh Kumar, Abhishek Gupta, Sergey Levine, Karol Hausman, and
  Chelsea Finn.
\newblock Gradient surgery for multi-task learning.
\newblock \emph{Advances in neural information processing systems},
  33:\penalty0 5824--5836, 2020{\natexlab{a}}.

\bibitem[Liu et~al.(2021)Liu, Liu, Jin, Stone, and Liu]{gradient_conflict2}
Bo~Liu, Xingchao Liu, Xiaojie Jin, Peter Stone, and Qiang Liu.
\newblock Conflict-averse gradient descent for multi-task learning.
\newblock \emph{Advances in Neural Information Processing Systems},
  34:\penalty0 18878--18890, 2021.

\bibitem[Chen and Er(2025)]{gradient_conflict3}
Jie Chen and Meng~Joo Er.
\newblock Mitigating gradient conflicts via expert squads in multi-task
  learning.
\newblock \emph{Neurocomputing}, 614:\penalty0 128832, 2025.

\bibitem[Russakovsky et~al.(2015)Russakovsky, Deng, Su, Krause, Satheesh, Ma,
  Huang, Karpathy, Khosla, Bernstein, et~al.]{russakovsky2015imagenet}
Olga Russakovsky, Jia Deng, Hao Su, Jonathan Krause, Sanjeev Satheesh, Sean Ma,
  Zhiheng Huang, Andrej Karpathy, Aditya Khosla, Michael Bernstein, et~al.
\newblock Imagenet large scale visual recognition challenge.
\newblock \emph{International journal of computer vision}, 115:\penalty0
  211--252, 2015.

\bibitem[Jacobs et~al.(1991)Jacobs, Jordan, Nowlan, and
  Hinton]{jacobs1991adaptive}
Robert~A Jacobs, Michael~I Jordan, Steven~J Nowlan, and Geoffrey~E Hinton.
\newblock Adaptive mixtures of local experts.
\newblock \emph{Neural computation}, 3\penalty0 (1):\penalty0 79--87, 1991.

\bibitem[He et~al.(2021)He, Liu, Liang, Zheng, Liao, Cheng, and
  Mei]{he2021semi}
Lingxiao He, Wu~Liu, Jian Liang, Kecheng Zheng, Xingyu Liao, Peng Cheng, and
  Tao Mei.
\newblock Semi-supervised domain generalizable person re-identification.
\newblock \emph{arXiv preprint arXiv:2108.05045}, 2021.

\bibitem[Yang et~al.(2024{\natexlab{c}})Yang, Zong, Huang, Feng, and
  An]{yang2024dual}
Penghui Yang, Chen-Chen Zong, Sheng-Jun Huang, Lei Feng, and Bo~An.
\newblock Dual-head knowledge distillation: Enhancing logits utilization with
  an auxiliary head.
\newblock \emph{arXiv preprint arXiv:2411.08937}, 2024{\natexlab{c}}.

\bibitem[Fei-Fei et~al.(2004)Fei-Fei, Fergus, and Perona]{fei2004learning}
Li~Fei-Fei, Rob Fergus, and Pietro Perona.
\newblock Learning generative visual models from few training examples: An
  incremental bayesian approach tested on 101 object categories.
\newblock In \emph{2004 conference on computer vision and pattern recognition
  workshop}, pages 178--178. IEEE, 2004.

\bibitem[Parkhi et~al.(2012)Parkhi, Vedaldi, Zisserman, and
  Jawahar]{parkhi2012cats}
Omkar~M Parkhi, Andrea Vedaldi, Andrew Zisserman, and CV~Jawahar.
\newblock Cats and dogs.
\newblock In \emph{2012 IEEE conference on computer vision and pattern
  recognition}, pages 3498--3505. IEEE, 2012.

\bibitem[Krause et~al.(2013)Krause, Stark, Deng, and Fei-Fei]{krause20133d}
Jonathan Krause, Michael Stark, Jia Deng, and Li~Fei-Fei.
\newblock 3d object representations for fine-grained categorization.
\newblock In \emph{Proceedings of the IEEE international conference on computer
  vision workshops}, pages 554--561, 2013.

\bibitem[Nilsback and Zisserman(2008)]{nilsback2008automated}
Maria-Elena Nilsback and Andrew Zisserman.
\newblock Automated flower classification over a large number of classes.
\newblock In \emph{2008 Sixth Indian conference on computer vision, graphics \&
  image processing}, pages 722--729. IEEE, 2008.

\bibitem[Bossard et~al.(2014)Bossard, Guillaumin, and
  Van~Gool]{bossard2014food}
Lukas Bossard, Matthieu Guillaumin, and Luc Van~Gool.
\newblock Food-101--mining discriminative components with random forests.
\newblock In \emph{Computer vision--ECCV 2014: 13th European conference,
  zurich, Switzerland, September 6-12, 2014, proceedings, part VI 13}, pages
  446--461. Springer, 2014.

\bibitem[Maji et~al.(2013)Maji, Rahtu, Kannala, Blaschko, and
  Vedaldi]{maji2013fine}
Subhransu Maji, Esa Rahtu, Juho Kannala, Matthew Blaschko, and Andrea Vedaldi.
\newblock Fine-grained visual classification of aircraft.
\newblock \emph{arXiv preprint arXiv:1306.5151}, 2013.

\bibitem[Xiao et~al.(2010)Xiao, Hays, Ehinger, Oliva, and
  Torralba]{xiao2010sun}
Jianxiong Xiao, James Hays, Krista~A Ehinger, Aude Oliva, and Antonio Torralba.
\newblock Sun database: Large-scale scene recognition from abbey to zoo.
\newblock In \emph{2010 IEEE computer society conference on computer vision and
  pattern recognition}, pages 3485--3492. IEEE, 2010.

\bibitem[Cimpoi et~al.(2014)Cimpoi, Maji, Kokkinos, Mohamed, and
  Vedaldi]{cimpoi2014describing}
Mircea Cimpoi, Subhransu Maji, Iasonas Kokkinos, Sammy Mohamed, and Andrea
  Vedaldi.
\newblock Describing textures in the wild.
\newblock In \emph{Proceedings of the IEEE conference on computer vision and
  pattern recognition}, pages 3606--3613, 2014.

\bibitem[Helber et~al.(2019)Helber, Bischke, Dengel, and
  Borth]{helber2019eurosat}
Patrick Helber, Benjamin Bischke, Andreas Dengel, and Damian Borth.
\newblock Eurosat: A novel dataset and deep learning benchmark for land use and
  land cover classification.
\newblock \emph{IEEE Journal of Selected Topics in Applied Earth Observations
  and Remote Sensing}, 12\penalty0 (7):\penalty0 2217--2226, 2019.

\bibitem[Soomro(2012)]{soomro2012ucf101}
K~Soomro.
\newblock Ucf101: A dataset of 101 human actions classes from videos in the
  wild.
\newblock \emph{arXiv preprint arXiv:1212.0402}, 2012.

\bibitem[Recht et~al.(2019)Recht, Roelofs, Schmidt, and
  Shankar]{recht2019imagenet}
Benjamin Recht, Rebecca Roelofs, Ludwig Schmidt, and Vaishaal Shankar.
\newblock Do imagenet classifiers generalize to imagenet?
\newblock In \emph{International conference on machine learning}, pages
  5389--5400. PMLR, 2019.

\bibitem[Wang et~al.(2019)Wang, Ge, Lipton, and Xing]{wang2019learning}
Haohan Wang, Songwei Ge, Zachary Lipton, and Eric~P Xing.
\newblock Learning robust global representations by penalizing local predictive
  power.
\newblock In \emph{Advances in Neural Information Processing Systems}, pages
  10506--10518, 2019.

\bibitem[Hendrycks et~al.(2021{\natexlab{a}})Hendrycks, Basart, Mu, Kadavath,
  Wang, Dorundo, Desai, Zhu, Parajuli, Guo, Song, Steinhardt, and
  Gilmer]{hendrycks2021many}
Dan Hendrycks, Steven Basart, Norman Mu, Saurav Kadavath, Frank Wang, Evan
  Dorundo, Rahul Desai, Tyler Zhu, Samyak Parajuli, Mike Guo, Dawn Song, Jacob
  Steinhardt, and Justin Gilmer.
\newblock The many faces of robustness: A critical analysis of
  out-of-distribution generalization.
\newblock \emph{ICCV}, 2021{\natexlab{a}}.

\bibitem[Hendrycks et~al.(2021{\natexlab{b}})Hendrycks, Zhao, Basart,
  Steinhardt, and Song]{hendrycks2021natural}
Dan Hendrycks, Kevin Zhao, Steven Basart, Jacob Steinhardt, and Dawn Song.
\newblock Natural adversarial examples.
\newblock In \emph{Proceedings of the IEEE/CVF conference on computer vision
  and pattern recognition}, pages 15262--15271, 2021{\natexlab{b}}.

\bibitem[Assran et~al.(2022)Assran, Caron, Misra, Bojanowski, Bordes, Vincent,
  Joulin, Rabbat, and Ballas]{assran2022masked}
Mahmoud Assran, Mathilde Caron, Ishan Misra, Piotr Bojanowski, Florian Bordes,
  Pascal Vincent, Armand Joulin, Mike Rabbat, and Nicolas Ballas.
\newblock Masked siamese networks for label-efficient learning.
\newblock In \emph{European Conference on Computer Vision}, pages 456--473.
  Springer, 2022.

\bibitem[Rothenberger and Diochnos(2023)]{rothenberger2023meta}
Jay~C Rothenberger and Dimitrios~I Diochnos.
\newblock Meta co-training: Two views are better than one.
\newblock \emph{arXiv preprint arXiv:2311.18083}, 2023.

\bibitem[Zhai et~al.(2023)Zhai, Mustafa, Kolesnikov, and
  Beyer]{zhai2023sigmoid}
Xiaohua Zhai, Basil Mustafa, Alexander Kolesnikov, and Lucas Beyer.
\newblock Sigmoid loss for language image pre-training.
\newblock In \emph{Proceedings of the IEEE/CVF International Conference on
  Computer Vision}, pages 11975--11986, 2023.

\bibitem[Fang et~al.(2023)Fang, Jose, Jain, Schmidt, Toshev, and
  Shankar]{fang2023data}
Alex Fang, Albin~Madappally Jose, Amit Jain, Ludwig Schmidt, Alexander Toshev,
  and Vaishaal Shankar.
\newblock Data filtering networks.
\newblock \emph{arXiv preprint arXiv:2309.17425}, 2023.

\bibitem[Jia et~al.(2022)Jia, Tang, Chen, Cardie, Belongie, Hariharan, and
  Lim]{jia2022visual}
Menglin Jia, Luming Tang, Bor-Chun Chen, Claire Cardie, Serge Belongie, Bharath
  Hariharan, and Ser-Nam Lim.
\newblock Visual prompt tuning.
\newblock In \emph{European Conference on Computer Vision}, pages 709--727.
  Springer, 2022.

\bibitem[Khattak et~al.(2023{\natexlab{a}})Khattak, Wasim, Naseer, Khan, Yang,
  and Khan]{khattak2023self}
Muhammad~Uzair Khattak, Syed~Talal Wasim, Muzammal Naseer, Salman Khan,
  Ming-Hsuan Yang, and Fahad~Shahbaz Khan.
\newblock Self-regulating prompts: Foundational model adaptation without
  forgetting.
\newblock In \emph{Proceedings of the IEEE/CVF International Conference on
  Computer Vision}, pages 15190--15200, 2023{\natexlab{a}}.

\bibitem[Caron et~al.(2021)Caron, Touvron, Misra, J{\'e}gou, Mairal,
  Bojanowski, and Joulin]{caron2021emerging}
Mathilde Caron, Hugo Touvron, Ishan Misra, Herv{\'e} J{\'e}gou, Julien Mairal,
  Piotr Bojanowski, and Armand Joulin.
\newblock Emerging properties in self-supervised vision transformers.
\newblock In \emph{Proceedings of the IEEE/CVF international conference on
  computer vision}, pages 9650--9660, 2021.

\bibitem[Sandler et~al.(2018)Sandler, Howard, Zhu, Zhmoginov, and
  Chen]{sandler2018mobilenetv2}
Mark Sandler, Andrew Howard, Menglong Zhu, Andrey Zhmoginov, and Liang-Chieh
  Chen.
\newblock Mobilenetv2: Inverted residuals and linear bottlenecks.
\newblock In \emph{Proceedings of the IEEE conference on computer vision and
  pattern recognition}, pages 4510--4520, 2018.

\bibitem[Yu et~al.(2020{\natexlab{b}})Yu, Kumar, Gupta, Levine, Hausman, and
  Finn]{yu2020gradient}
Tianhe Yu, Saurabh Kumar, Abhishek Gupta, Sergey Levine, Karol Hausman, and
  Chelsea Finn.
\newblock Gradient surgery for multi-task learning.
\newblock \emph{Advances in neural information processing systems},
  33:\penalty0 5824--5836, 2020{\natexlab{b}}.

\bibitem[Zhao et~al.(2022)Zhao, Cui, Song, Qiu, and Liang]{zhao2022decoupled}
Borui Zhao, Quan Cui, Renjie Song, Yiyu Qiu, and Jiajun Liang.
\newblock Decoupled knowledge distillation.
\newblock In \emph{Proceedings of the IEEE/CVF Conference on computer vision
  and pattern recognition}, pages 11953--11962, 2022.

\bibitem[Lv et~al.(2024)Lv, Yang, and Li]{lv2024wasserstein}
Jiaming Lv, Haoyuan Yang, and Peihua Li.
\newblock Wasserstein distance rivals kullback-leibler divergence for knowledge
  distillation.
\newblock \emph{Advances in Neural Information Processing Systems},
  37:\penalty0 65445--65475, 2024.

\bibitem[Van~der Maaten and Hinton(2008)]{van2008visualizing}
Laurens Van~der Maaten and Geoffrey Hinton.
\newblock Visualizing data using t-sne.
\newblock \emph{Journal of machine learning research}, 9\penalty0 (11), 2008.

\bibitem[Wei et~al.(2021)Wei, Bosma, Zhao, Guu, Yu, Lester, Du, Dai, and
  Le]{wei2021finetuned}
Jason Wei, Maarten Bosma, Vincent~Y Zhao, Kelvin Guu, Adams~Wei Yu, Brian
  Lester, Nan Du, Andrew~M Dai, and Quoc~V Le.
\newblock Finetuned language models are zero-shot learners.
\newblock \emph{arXiv preprint arXiv:2109.01652}, 2021.

\bibitem[Liu et~al.(2023{\natexlab{b}})Liu, Li, Wu, and Lee]{liu2023visual}
Haotian Liu, Chunyuan Li, Qingyang Wu, and Yong~Jae Lee.
\newblock Visual instruction tuning.
\newblock \emph{Advances in neural information processing systems},
  36:\penalty0 34892--34916, 2023{\natexlab{b}}.

\bibitem[Bai et~al.(2023)Bai, Bai, Chu, Cui, Dang, Deng, Fan, Ge, Han, Huang,
  et~al.]{bai2023qwen}
Jinze Bai, Shuai Bai, Yunfei Chu, Zeyu Cui, Kai Dang, Xiaodong Deng, Yang Fan,
  Wenbin Ge, Yu~Han, Fei Huang, et~al.
\newblock Qwen technical report.
\newblock \emph{arXiv preprint arXiv:2309.16609}, 2023.

\bibitem[Achiam et~al.(2023)Achiam, Adler, Agarwal, Ahmad, Akkaya, Aleman,
  Almeida, Altenschmidt, Altman, Anadkat, et~al.]{achiam2023gpt}
Josh Achiam, Steven Adler, Sandhini Agarwal, Lama Ahmad, Ilge Akkaya,
  Florencia~Leoni Aleman, Diogo Almeida, Janko Altenschmidt, Sam Altman,
  Shyamal Anadkat, et~al.
\newblock Gpt-4 technical report.
\newblock \emph{arXiv preprint arXiv:2303.08774}, 2023.

\bibitem[Team et~al.(2023)Team, Anil, Borgeaud, Alayrac, Yu, Soricut,
  Schalkwyk, Dai, Hauth, Millican, et~al.]{team2023gemini}
Gemini Team, Rohan Anil, Sebastian Borgeaud, Jean-Baptiste Alayrac, Jiahui Yu,
  Radu Soricut, Johan Schalkwyk, Andrew~M Dai, Anja Hauth, Katie Millican,
  et~al.
\newblock Gemini: a family of highly capable multimodal models.
\newblock \emph{arXiv preprint arXiv:2312.11805}, 2023.

\bibitem[Wang et~al.(2023)Wang, Chen, Qian, Gao, Wei, Wang, Tian, and
  Gao]{wang2023large}
Xiao Wang, Guangyao Chen, Guangwu Qian, Pengcheng Gao, Xiao-Yong Wei, Yaowei
  Wang, Yonghong Tian, and Wen Gao.
\newblock Large-scale multi-modal pre-trained models: A comprehensive survey.
\newblock \emph{Machine Intelligence Research}, 20\penalty0 (4):\penalty0
  447--482, 2023.

\bibitem[Chen et~al.(2023)Chen, Zhang, Han, Chen, Shi, Xu, and Xu]{chen2023vlp}
Fei-Long Chen, Du-Zhen Zhang, Ming-Lun Han, Xiu-Yi Chen, Jing Shi, Shuang Xu,
  and Bo~Xu.
\newblock Vlp: A survey on vision-language pre-training.
\newblock \emph{Machine Intelligence Research}, 20\penalty0 (1):\penalty0
  38--56, 2023.

\bibitem[Gan et~al.(2022)Gan, Li, Li, Wang, Liu, Gao, et~al.]{gan2022vision}
Zhe Gan, Linjie Li, Chunyuan Li, Lijuan Wang, Zicheng Liu, Jianfeng Gao, et~al.
\newblock Vision-language pre-training: Basics, recent advances, and future
  trends.
\newblock \emph{Foundations and Trends{\textregistered} in Computer Graphics
  and Vision}, 14\penalty0 (3--4):\penalty0 163--352, 2022.

\bibitem[Zhang et~al.(2024{\natexlab{a}})Zhang, Huang, Jin, and
  Lu]{zhang2024vision}
Jingyi Zhang, Jiaxing Huang, Sheng Jin, and Shijian Lu.
\newblock Vision-language models for vision tasks: A survey.
\newblock \emph{IEEE Transactions on Pattern Analysis and Machine
  Intelligence}, 2024{\natexlab{a}}.

\bibitem[Dong et~al.(2023)Dong, Bao, Zheng, Zhang, Chen, Yang, Zeng, Zhang,
  Yuan, Chen, et~al.]{dong2023maskclip}
Xiaoyi Dong, Jianmin Bao, Yinglin Zheng, Ting Zhang, Dongdong Chen, Hao Yang,
  Ming Zeng, Weiming Zhang, Lu~Yuan, Dong Chen, et~al.
\newblock Maskclip: Masked self-distillation advances contrastive
  language-image pretraining.
\newblock In \emph{Proceedings of the IEEE/CVF Conference on Computer Vision
  and Pattern Recognition}, pages 10995--11005, 2023.

\bibitem[Gao et~al.(2022)Gao, Liu, Xu, Zhang, Li, Ji, and
  Shen]{gao2022pyramidclip}
Yuting Gao, Jinfeng Liu, Zihan Xu, Jun Zhang, Ke~Li, Rongrong Ji, and Chunhua
  Shen.
\newblock Pyramidclip: Hierarchical feature alignment for vision-language model
  pretraining.
\newblock \emph{Advances in neural information processing systems},
  35:\penalty0 35959--35970, 2022.

\bibitem[Yu et~al.(2022)Yu, Wang, Vasudevan, Yeung, Seyedhosseini, and
  Wu]{yu2022coca}
Jiahui Yu, Zirui Wang, Vijay Vasudevan, Legg Yeung, Mojtaba Seyedhosseini, and
  Yonghui Wu.
\newblock Coca: Contrastive captioners are image-text foundation models.
\newblock \emph{arXiv preprint arXiv:2205.01917}, 2022.

\bibitem[Li et~al.(2023)Li, Fan, Hu, Feichtenhofer, and He]{li2023scaling}
Yanghao Li, Haoqi Fan, Ronghang Hu, Christoph Feichtenhofer, and Kaiming He.
\newblock Scaling language-image pre-training via masking.
\newblock In \emph{Proceedings of the IEEE/CVF Conference on Computer Vision
  and Pattern Recognition}, pages 23390--23400, 2023.

\bibitem[Dehghani et~al.(2023)Dehghani, Djolonga, Mustafa, Padlewski, Heek,
  Gilmer, Steiner, Caron, Geirhos, Alabdulmohsin, et~al.]{dehghani2023scaling}
Mostafa Dehghani, Josip Djolonga, Basil Mustafa, Piotr Padlewski, Jonathan
  Heek, Justin Gilmer, Andreas~Peter Steiner, Mathilde Caron, Robert Geirhos,
  Ibrahim Alabdulmohsin, et~al.
\newblock Scaling vision transformers to 22 billion parameters.
\newblock In \emph{International Conference on Machine Learning}, pages
  7480--7512. PMLR, 2023.

\bibitem[Sun et~al.(2023{\natexlab{a}})Sun, Fang, Wu, Wang, and
  Cao]{sun2023eva}
Quan Sun, Yuxin Fang, Ledell Wu, Xinlong Wang, and Yue Cao.
\newblock Eva-clip: Improved training techniques for clip at scale.
\newblock \emph{arXiv preprint arXiv:2303.15389}, 2023{\natexlab{a}}.

\bibitem[Cherti et~al.(2023)Cherti, Beaumont, Wightman, Wortsman, Ilharco,
  Gordon, Schuhmann, Schmidt, and Jitsev]{cherti2023reproducible}
Mehdi Cherti, Romain Beaumont, Ross Wightman, Mitchell Wortsman, Gabriel
  Ilharco, Cade Gordon, Christoph Schuhmann, Ludwig Schmidt, and Jenia Jitsev.
\newblock Reproducible scaling laws for contrastive language-image learning.
\newblock In \emph{Proceedings of the IEEE/CVF Conference on Computer Vision
  and Pattern Recognition}, pages 2818--2829, 2023.

\bibitem[Sun et~al.(2024)Sun, Wang, Yu, Cui, Zhang, Zhang, and
  Wang]{sun2024eva}
Quan Sun, Jinsheng Wang, Qiying Yu, Yufeng Cui, Fan Zhang, Xiaosong Zhang, and
  Xinlong Wang.
\newblock Eva-clip-18b: Scaling clip to 18 billion parameters.
\newblock \emph{arXiv preprint arXiv:2402.04252}, 2024.

\bibitem[Guo et~al.(2024)Guo, Xu, Zhang, Ren, Ma, Ju, Wang, Chen, and
  Yang]{guo2024boldsymbol}
Qingpei Guo, Furong Xu, Hanxiao Zhang, Wang Ren, Ziping Ma, Lin Ju, Jian Wang,
  Jingdong Chen, and Ming Yang.
\newblock {$\mathbf{M}^2$}-encoder: {A}dvancing {B}ilingual {I}mage-{T}ext
  {U}nderstanding by {L}arge-scale {E}fficient {P}retraining.
\newblock \emph{arXiv preprint arXiv:2401.15896}, 2024.

\bibitem[Zhai et~al.(2022)Zhai, Wang, Mustafa, Steiner, Keysers, Kolesnikov,
  and Beyer]{zhai2022lit}
Xiaohua Zhai, Xiao Wang, Basil Mustafa, Andreas Steiner, Daniel Keysers,
  Alexander Kolesnikov, and Lucas Beyer.
\newblock Lit: Zero-shot transfer with locked-image text tuning.
\newblock In \emph{Proceedings of the IEEE/CVF conference on computer vision
  and pattern recognition}, pages 18123--18133, 2022.

\bibitem[Pham et~al.(2023)Pham, Dai, Ghiasi, Kawaguchi, Liu, Yu, Yu, Chen,
  Luong, Wu, et~al.]{pham2023combined}
Hieu Pham, Zihang Dai, Golnaz Ghiasi, Kenji Kawaguchi, Hanxiao Liu, Adams~Wei
  Yu, Jiahui Yu, Yi-Ting Chen, Minh-Thang Luong, Yonghui Wu, et~al.
\newblock Combined scaling for zero-shot transfer learning.
\newblock \emph{Neurocomputing}, 555:\penalty0 126658, 2023.

\bibitem[Lester et~al.(2021)Lester, Al-Rfou, and Constant]{lester2021power}
Brian Lester, Rami Al-Rfou, and Noah Constant.
\newblock The power of scale for parameter-efficient prompt tuning.
\newblock \emph{arXiv preprint arXiv:2104.08691}, 2021.

\bibitem[Khattak et~al.(2023{\natexlab{b}})Khattak, Rasheed, Maaz, Khan, and
  Khan]{khattak2023maple}
Muhammad~Uzair Khattak, Hanoona Rasheed, Muhammad Maaz, Salman Khan, and
  Fahad~Shahbaz Khan.
\newblock Maple: Multi-modal prompt learning.
\newblock In \emph{Proceedings of the IEEE/CVF Conference on Computer Vision
  and Pattern Recognition}, pages 19113--19122, 2023{\natexlab{b}}.

\bibitem[Zhu et~al.(2023)Zhu, Niu, Han, Wu, and Zhang]{zhu2023prompt}
Beier Zhu, Yulei Niu, Yucheng Han, Yue Wu, and Hanwang Zhang.
\newblock Prompt-aligned gradient for prompt tuning.
\newblock In \emph{Proceedings of the IEEE/CVF International Conference on
  Computer Vision}, pages 15659--15669, 2023.

\bibitem[Menghini et~al.(2023)Menghini, Delworth, and
  Bach]{menghini2023enhancing}
Cristina Menghini, Andrew Delworth, and Stephen Bach.
\newblock Enhancing clip with clip: Exploring pseudolabeling for limited-label
  prompt tuning.
\newblock \emph{Advances in Neural Information Processing Systems},
  36:\penalty0 60984--61007, 2023.

\bibitem[Zhao et~al.(2024)Zhao, Wang, Jiang, Shen, Song, Li, and
  Miao]{zhao2024learning}
Cairong Zhao, Yubin Wang, Xinyang Jiang, Yifei Shen, Kaitao Song, Dongsheng Li,
  and Duoqian Miao.
\newblock Learning domain invariant prompt for vision-language models.
\newblock \emph{IEEE Transactions on Image Processing}, 2024.

\bibitem[Roy and Etemad(2023)]{roy2023consistency}
Shuvendu Roy and Ali Etemad.
\newblock Consistency-guided prompt learning for vision-language models.
\newblock \emph{arXiv preprint arXiv:2306.01195}, 2023.

\bibitem[Zhang et~al.(2024{\natexlab{b}})Zhang, Wu, Gao, Shen, and
  Song]{zhang2024dept}
Ji~Zhang, Shihan Wu, Lianli Gao, Heng~Tao Shen, and Jingkuan Song.
\newblock Dept: Decoupled prompt tuning.
\newblock In \emph{Proceedings of the IEEE/CVF Conference on Computer Vision
  and Pattern Recognition}, pages 12924--12933, 2024{\natexlab{b}}.

\bibitem[Yu et~al.(2023{\natexlab{a}})Yu, Lu, Jin, Chen, and Wang]{yu2023task}
Tao Yu, Zhihe Lu, Xin Jin, Zhibo Chen, and Xinchao Wang.
\newblock Task residual for tuning vision-language models.
\newblock In \emph{Proceedings of the IEEE/CVF Conference on Computer Vision
  and Pattern Recognition}, pages 10899--10909, 2023{\natexlab{a}}.

\bibitem[Silva-Rodriguez et~al.(2024)Silva-Rodriguez, Hajimiri, Ben~Ayed, and
  Dolz]{silva2024closer}
Julio Silva-Rodriguez, Sina Hajimiri, Ismail Ben~Ayed, and Jose Dolz.
\newblock A closer look at the few-shot adaptation of large vision-language
  models.
\newblock In \emph{Proceedings of the IEEE/CVF Conference on Computer Vision
  and Pattern Recognition}, pages 23681--23690, 2024.

\bibitem[Lopes et~al.(2017)Lopes, Fenu, and Starner]{lopes2017data}
Raphael~Gontijo Lopes, Stefano Fenu, and Thad Starner.
\newblock Data-free knowledge distillation for deep neural networks.
\newblock \emph{arXiv preprint arXiv:1710.07535}, 2017.

\bibitem[Kimura et~al.(2018)Kimura, Ghahramani, Takeuchi, Iwata, and
  Ueda]{kimura2018few}
Akisato Kimura, Zoubin Ghahramani, Koh Takeuchi, Tomoharu Iwata, and Naonori
  Ueda.
\newblock Few-shot learning of neural networks from scratch by pseudo example
  optimization.
\newblock \emph{arXiv preprint arXiv:1802.03039}, 2018.

\bibitem[Nayak et~al.(2019)Nayak, Mopuri, Shaj, Radhakrishnan, and
  Chakraborty]{nayak2019zero}
Gaurav~Kumar Nayak, Konda~Reddy Mopuri, Vaisakh Shaj, Venkatesh~Babu
  Radhakrishnan, and Anirban Chakraborty.
\newblock Zero-shot knowledge distillation in deep networks.
\newblock In \emph{International Conference on Machine Learning}, pages
  4743--4751. PMLR, 2019.

\bibitem[Yoo et~al.(2019)Yoo, Cho, Kim, and Kang]{yoo2019knowledge}
Jaemin Yoo, Minyong Cho, Taebum Kim, and U~Kang.
\newblock Knowledge extraction with no observable data.
\newblock \emph{Advances in Neural Information Processing Systems}, 32, 2019.

\bibitem[Chen et~al.(2019)Chen, Wang, Xu, Yang, Liu, Shi, Xu, Xu, and
  Tian]{chen2019data}
Hanting Chen, Yunhe Wang, Chang Xu, Zhaohui Yang, Chuanjian Liu, Boxin Shi,
  Chunjing Xu, Chao Xu, and Qi~Tian.
\newblock Data-free learning of student networks.
\newblock In \emph{Proceedings of the IEEE/CVF international conference on
  computer vision}, pages 3514--3522, 2019.

\bibitem[Yin et~al.(2020)Yin, Molchanov, Alvarez, Li, Mallya, Hoiem, Jha, and
  Kautz]{yin2020dreaming}
Hongxu Yin, Pavlo Molchanov, Jose~M Alvarez, Zhizhong Li, Arun Mallya, Derek
  Hoiem, Niraj~K Jha, and Jan Kautz.
\newblock Dreaming to distill: Data-free knowledge transfer via deepinversion.
\newblock In \emph{Proceedings of the IEEE/CVF conference on computer vision
  and pattern recognition}, pages 8715--8724, 2020.

\bibitem[Fang et~al.(2021{\natexlab{a}})Fang, Song, Wang, Shen, Wang, and
  Song]{fang2021contrastive}
Gongfan Fang, Jie Song, Xinchao Wang, Chengchao Shen, Xingen Wang, and Mingli
  Song.
\newblock Contrastive model inversion for data-free knowledge distillation.
\newblock \emph{arXiv preprint arXiv:2105.08584}, 2021{\natexlab{a}}.

\bibitem[Nguyen et~al.(2022)Nguyen, Gupta, Do, and Venkatesh]{nguyen2022black}
Dang Nguyen, Sunil Gupta, Kien Do, and Svetha Venkatesh.
\newblock Black-box few-shot knowledge distillation.
\newblock In \emph{European Conference on Computer Vision}, pages 196--211.
  Springer, 2022.

\bibitem[Patel et~al.(2023)Patel, Mopuri, and Qiu]{patel2023learning}
Gaurav Patel, Konda~Reddy Mopuri, and Qiang Qiu.
\newblock Learning to retain while acquiring: Combating distribution-shift in
  adversarial data-free knowledge distillation.
\newblock In \emph{Proceedings of the IEEE/CVF Conference on Computer Vision
  and Pattern Recognition}, pages 7786--7794, 2023.

\bibitem[Yu et~al.(2023{\natexlab{b}})Yu, Chen, Han, and Jiang]{yu2023data}
Shikang Yu, Jiachen Chen, Hu~Han, and Shuqiang Jiang.
\newblock Data-free knowledge distillation via feature exchange and activation
  region constraint.
\newblock In \emph{Proceedings of the IEEE/CVF Conference on Computer Vision
  and Pattern Recognition}, pages 24266--24275, 2023{\natexlab{b}}.

\bibitem[Liu et~al.(2024)Liu, Wang, Liu, Sun, and Yao]{liu2024small}
He~Liu, Yikai Wang, Huaping Liu, Fuchun Sun, and Anbang Yao.
\newblock Small scale data-free knowledge distillation.
\newblock In \emph{Proceedings of the IEEE/CVF Conference on Computer Vision
  and Pattern Recognition}, pages 6008--6016, 2024.

\bibitem[Tran et~al.(2024)Tran, Le, Le, Cai, Harandi, and Phung]{tran2024large}
Minh-Tuan Tran, Trung Le, Xuan-May Le, Jianfei Cai, Mehrtash Harandi, and Dinh
  Phung.
\newblock Large-scale data-free knowledge distillation for imagenet via
  multi-resolution data generation.
\newblock \emph{arXiv preprint arXiv:2411.17046}, 2024.

\bibitem[Wang et~al.(2024)Wang, Yang, Chen, Liu, Liu, Zhang, Zhang, and
  Qi]{wang2024confounded}
Yuzheng Wang, Dingkang Yang, Zhaoyu Chen, Yang Liu, Siao Liu, Wenqiang Zhang,
  Lihua Zhang, and Lizhe Qi.
\newblock De-confounded data-free knowledge distillation for handling
  distribution shifts.
\newblock In \emph{Proceedings of the IEEE/CVF Conference on Computer Vision
  and Pattern Recognition}, pages 12615--12625, 2024.

\bibitem[Du et~al.(2023)Du, Zhao, Sheng, Li, and Chen]{du2023semi}
Pan Du, Suyun Zhao, Zisen Sheng, Cuiping Li, and Hong Chen.
\newblock Semi-supervised learning via weight-aware distillation under class
  distribution mismatch.
\newblock In \emph{Proceedings of the IEEE/CVF International Conference on
  Computer Vision}, pages 16410--16420, 2023.

\bibitem[Yang et~al.(2024{\natexlab{d}})Yang, Zhu, Bulat, Martinez, and
  Tzimiropoulos]{yang2024knowledge}
Jing Yang, Xiatian Zhu, Adrian Bulat, Brais Martinez, and Georgios
  Tzimiropoulos.
\newblock Knowledge distillation meets open-set semi-supervised learning.
\newblock \emph{International Journal of Computer Vision}, pages 1--20,
  2024{\natexlab{d}}.

\bibitem[Fang et~al.(2021{\natexlab{b}})Fang, Wang, Wang, Zhang, Yang, and
  Liu]{fang2021seed}
Zhiyuan Fang, Jianfeng Wang, Lijuan Wang, Lei Zhang, Yezhou Yang, and Zicheng
  Liu.
\newblock Seed: Self-supervised distillation for visual representation.
\newblock \emph{arXiv preprint arXiv:2101.04731}, 2021{\natexlab{b}}.

\bibitem[Abbasi~Koohpayegani et~al.(2020)Abbasi~Koohpayegani, Tejankar, and
  Pirsiavash]{abbasi2020compress}
Soroush Abbasi~Koohpayegani, Ajinkya Tejankar, and Hamed Pirsiavash.
\newblock Compress: Self-supervised learning by compressing representations.
\newblock \emph{Advances in Neural Information Processing Systems},
  33:\penalty0 12980--12992, 2020.

\bibitem[Navaneet et~al.(2021)Navaneet, Koohpayegani, Tejankar, and
  Pirsiavash]{navaneet2021simreg}
K~L Navaneet, Soroush~Abbasi Koohpayegani, Ajinkya Tejankar, and Hamed
  Pirsiavash.
\newblock Simreg: Regression as a simple yet effective tool for self-supervised
  knowledge distillation.
\newblock In \emph{British Machine Vision Conference (BMVC)}, 2021.

\bibitem[Wang et~al.(2022{\natexlab{a}})Wang, Yang, and van~de
  Weijer]{wang2022attention}
Kai Wang, Fei Yang, and Joost van~de Weijer.
\newblock Attention distillation: self-supervised vision transformer students
  need more guidance.
\newblock \emph{arXiv preprint arXiv:2210.00944}, 2022{\natexlab{a}}.

\bibitem[Xu et~al.(2021)Xu, Fang, Zhang, Xie, Wang, Dai, Xiong, and
  Tian]{xu2021bag}
Haohang Xu, Jiemin Fang, Xiaopeng Zhang, Lingxi Xie, Xinggang Wang, Wenrui Dai,
  Hongkai Xiong, and Qi~Tian.
\newblock Bag of instances aggregation boosts self-supervised distillation.
\newblock \emph{arXiv preprint arXiv:2107.01691}, 2021.

\bibitem[Singh and Wang(2025)]{singh2025simple}
Aditya Singh and Haohan Wang.
\newblock Simple unsupervised knowledge distillation with space similarity.
\newblock In \emph{European Conference on Computer Vision}, pages 147--164.
  Springer, 2025.

\bibitem[Fang et~al.(2021{\natexlab{c}})Fang, Wang, Hu, Wang, Yang, and
  Liu]{fang2021compressing}
Zhiyuan Fang, Jianfeng Wang, Xiaowei Hu, Lijuan Wang, Yezhou Yang, and Zicheng
  Liu.
\newblock Compressing visual-linguistic model via knowledge distillation.
\newblock In \emph{Proceedings of the IEEE/CVF International Conference on
  Computer Vision}, pages 1428--1438, 2021{\natexlab{c}}.

\bibitem[Wu et~al.(2023)Wu, Peng, Zhou, Xiao, Liu, Yuan, Xuan, Valenzuela,
  Chen, Wang, et~al.]{wu2023tinyclip}
Kan Wu, Houwen Peng, Zhenghong Zhou, Bin Xiao, Mengchen Liu, Lu~Yuan, Hong
  Xuan, Michael Valenzuela, Xi~Stephen Chen, Xinggang Wang, et~al.
\newblock Tinyclip: Clip distillation via affinity mimicking and weight
  inheritance.
\newblock In \emph{Proceedings of the IEEE/CVF International Conference on
  Computer Vision}, pages 21970--21980, 2023.

\bibitem[Sun et~al.(2023{\natexlab{b}})Sun, Zhang, Zhang, Shah, Saenko, and
  Xia]{sun2023dime}
Ximeng Sun, Pengchuan Zhang, Peizhao Zhang, Hardik Shah, Kate Saenko, and Xide
  Xia.
\newblock Dime-fm: Distilling multimodal and efficient foundation models.
\newblock In \emph{Proceedings of the IEEE/CVF International Conference on
  Computer Vision}, pages 15521--15533, 2023{\natexlab{b}}.

\bibitem[Kim et~al.(2024)Kim, Jang, and Yang]{kim2024promptkd}
Gyeongman Kim, Doohyuk Jang, and Eunho Yang.
\newblock Promptkd: Distilling student-friendly knowledge for generative
  language models via prompt tuning.
\newblock \emph{arXiv preprint arXiv:2402.12842}, 2024.

\bibitem[Chen et~al.(2024)Chen, Qiao, Sun, and Li]{chen2024comkd}
Yifan Chen, Xiaozhen Qiao, Zhe Sun, and Xuelong Li.
\newblock Comkd-clip: Comprehensive knowledge distillation for contrastive
  language-image pre-traning model.
\newblock \emph{arXiv preprint arXiv:2408.04145}, 2024.

\bibitem[Mistretta et~al.(2025)Mistretta, Baldrati, Bertini, and
  Bagdanov]{mistretta2025improving}
Marco Mistretta, Alberto Baldrati, Marco Bertini, and Andrew~D Bagdanov.
\newblock Improving zero-shot generalization of learned prompts via
  unsupervised knowledge distillation.
\newblock In \emph{European Conference on Computer Vision}, pages 459--477.
  Springer, 2025.

\bibitem[Cho and Hariharan(2019)]{cho2019efficacy}
Jang~Hyun Cho and Bharath Hariharan.
\newblock On the efficacy of knowledge distillation.
\newblock In \emph{Proceedings of the IEEE/CVF international conference on
  computer vision}, pages 4794--4802, 2019.

\bibitem[Mirzadeh et~al.(2020)Mirzadeh, Farajtabar, Li, Levine, Matsukawa, and
  Ghasemzadeh]{mirzadeh2020improved}
Seyed~Iman Mirzadeh, Mehrdad Farajtabar, Ang Li, Nir Levine, Akihiro Matsukawa,
  and Hassan Ghasemzadeh.
\newblock Improved knowledge distillation via teacher assistant.
\newblock In \emph{Proceedings of the AAAI conference on artificial
  intelligence}, volume~34, pages 5191--5198, 2020.

\bibitem[Zhu and Wang(2021)]{zhu2021student}
Yichen Zhu and Yi~Wang.
\newblock Student customized knowledge distillation: Bridging the gap between
  student and teacher.
\newblock In \emph{Proceedings of the IEEE/CVF International Conference on
  Computer Vision}, pages 5057--5066, 2021.

\bibitem[Huang et~al.(2022)Huang, You, Wang, Qian, and Xu]{huang2022knowledge}
Tao Huang, Shan You, Fei Wang, Chen Qian, and Chang Xu.
\newblock Knowledge distillation from a stronger teacher.
\newblock \emph{Advances in Neural Information Processing Systems},
  35:\penalty0 33716--33727, 2022.

\bibitem[Li et~al.(2024)Li, Fan, Tao, Gan, and Zhan]{li2024exploring}
Xin-Chun Li, Wen-Shu Fan, Bowen Tao, Le~Gan, and De-Chuan Zhan.
\newblock Exploring dark knowledge under various teacher capacities and
  addressing capacity mismatch.
\newblock \emph{arXiv preprint arXiv:2405.13078}, 2024.

\bibitem[Zhang et~al.(2023)Zhang, Shen, Liu, Liu, Bendersky, Najork, and
  Zhang]{zhang2023not}
Rongzhi Zhang, Jiaming Shen, Tianqi Liu, Jialu Liu, Michael Bendersky, Marc
  Najork, and Chao Zhang.
\newblock Do not blindly imitate the teacher: Using perturbed loss for
  knowledge distillation.
\newblock \emph{arXiv preprint arXiv:2305.05010}, 2023.

\bibitem[Papyan et~al.(2020)Papyan, Han, and Donoho]{papyan2020prevalence}
Vardan Papyan, XY~Han, and David~L Donoho.
\newblock Prevalence of neural collapse during the terminal phase of deep
  learning training.
\newblock \emph{Proceedings of the National Academy of Sciences}, 117\penalty0
  (40):\penalty0 24652--24663, 2020.

\bibitem[He et~al.(2016)He, Zhang, Ren, and Sun]{he2016deep}
Kaiming He, Xiangyu Zhang, Shaoqing Ren, and Jian Sun.
\newblock Deep residual learning for image recognition.
\newblock In \emph{Proceedings of the IEEE conference on computer vision and
  pattern recognition}, pages 770--778, 2016.

\bibitem[Dosovitskiy(2020)]{dosovitskiy2020image}
Alexey Dosovitskiy.
\newblock An image is worth 16x16 words: Transformers for image recognition at
  scale.
\newblock \emph{arXiv preprint arXiv:2010.11929}, 2020.

\bibitem[Micikevicius et~al.(2017)Micikevicius, Narang, Alben, Diamos, Elsen,
  Garcia, Ginsburg, Houston, Kuchaiev, Venkatesh,
  et~al.]{micikevicius2017mixed}
Paulius Micikevicius, Sharan Narang, Jonah Alben, Gregory Diamos, Erich Elsen,
  David Garcia, Boris Ginsburg, Michael Houston, Oleksii Kuchaiev, Ganesh
  Venkatesh, et~al.
\newblock Mixed precision training.
\newblock \emph{arXiv preprint arXiv:1710.03740}, 2017.

\bibitem[Bolya et~al.(2022)Bolya, Fu, Dai, Zhang, Feichtenhofer, and
  Hoffman]{bolya2022token}
Daniel Bolya, Cheng-Yang Fu, Xiaoliang Dai, Peizhao Zhang, Christoph
  Feichtenhofer, and Judy Hoffman.
\newblock Token merging: Your vit but faster.
\newblock \emph{arXiv preprint arXiv:2210.09461}, 2022.

\bibitem[Wang et~al.(2022{\natexlab{b}})Wang, Wu, Lian, and
  Yu]{wang2022debiased}
Xudong Wang, Zhirong Wu, Long Lian, and Stella~X Yu.
\newblock Debiased learning from naturally imbalanced pseudo-labels.
\newblock In \emph{Proceedings of the IEEE/CVF Conference on Computer Vision
  and Pattern Recognition}, pages 14647--14657, 2022{\natexlab{b}}.

\bibitem[Ba(2016)]{ba2016layer}
Jimmy~Lei Ba.
\newblock Layer normalization.
\newblock \emph{arXiv preprint arXiv:1607.06450}, 2016.

\bibitem[Hendrycks and Gimpel(2016)]{hendrycks2016gaussian}
Dan Hendrycks and Kevin Gimpel.
\newblock Gaussian error linear units (gelus).
\newblock \emph{arXiv preprint arXiv:1606.08415}, 2016.

\bibitem[Srivastava et~al.(2014)Srivastava, Hinton, Krizhevsky, Sutskever, and
  Salakhutdinov]{srivastava2014dropout}
Nitish Srivastava, Geoffrey Hinton, Alex Krizhevsky, Ilya Sutskever, and Ruslan
  Salakhutdinov.
\newblock Dropout: a simple way to prevent neural networks from overfitting.
\newblock \emph{The journal of machine learning research}, 15\penalty0
  (1):\penalty0 1929--1958, 2014.

\end{thebibliography}


\clearpage
\appendix

\section*{Appendix Overview}

This appendix provides supplementary material to support the main paper and is organized as follows:

\begin{itemize}[itemsep=1mm,parsep=1pt,topsep=2pt,leftmargin=*]
    \item \textbf{Limitations and Future Work} (\Cref{sec:appendix_limitations}): discusses the scope of foundation models and tasks considered in this work and outlines directions for future research.

    \item \textbf{Related Work} (\Cref{sec:related_work}): discusses previous work relevant to ours, including vision-language pre-training, data-limited adaptation of VLMs, knowledge distillation (KD), and dual-head approaches.

    \item \textbf{Theoretical Analysis} (\Cref{sec:theoretical_analysis}): provides formal analysis of gradient conflicts in SHO with sufficient conditions for their occurrence (\Cref{sec:gradient_conflict_analysis}), followed by inference equivalence guarantees for \textbf{\texttt{DHO}} (\Cref{sec:inference_equivalence_analysis}).

    \item \textbf{Algorithms and Implementation} (\Cref{sec:algorithms_and_implementation}): presents detailed pseudocode (\Cref{sec:appendix_algorithm}), implementation specifics (\Cref{sec:appendix_implementation_details}), and computational overhead analysis (\Cref{sec:appendix_computational_overhead}).

    \item \textbf{Datasets} (\Cref{sec:appendix_datasets}): describes the datasets used in our experiments, including statistics and preprocessing details.

    \item \textbf{Additional Experiments} (\Cref{sec:additional_experiments}): presents motivation (\Cref{sec:appendix_motivation}), MobileNet (\Cref{sec:appendix_mobilenet}), gradient surgery (\Cref{sec:appendix_pcgrad}), additional KD methods (\Cref{sec:appendix_dkd_wkd}), adaptive weighting (\Cref{sec:appendix_adaptive_weighting}), SHO hyperparameters and $(\alpha,\beta)$ (\Cref{sec:appendix_sho_hyperparameter}), 2-stage baselines (\Cref{sec:appendix_2-stage}), orthogonality to SSL methods (\Cref{sec:orthogornal_ssl}), full-fine-tuning OOD results (\Cref{sec:appendix_robustness_full}), repeated-run reliability (\Cref{sec:appendix_reliability}), and additional VLM teachers (\Cref{sec:appendix_additional_teachers}).

    \item \textbf{Additional Analysis} (\Cref{sec:additional_analyses}): contains feature-quality analysis (\Cref{sec:appendix_feature_quality}), per-layer gradient conflicts (\Cref{sec:appendix_per_layer_conflict}), non-linear head design and regularization (\Cref{sec:appendix_nonlinear}), and dual-head and qualitative results (\Cref{sec:appendix_qualitative_results}).

    \item \textbf{The Use of LLMs} (\Cref{sec:llm}): describes our use of LLMs during paper preparation.
\end{itemize}



\section{Limitations and Future Work}
\label{sec:appendix_limitations}

While our main focus is to address the core problem of \emph{gradient conflicts} arising from \textbf{general knowledge of VLMs and task-specific patterns from labeled data}, we acknowledge several limitations that present opportunities for future research.

\paragraph{Foundation-model scope.}
We primarily focus on VLMs as general knowledge source, as they provide strong zero-shot and few-shot capabilities in visual recognition tasks.
However, we believe the fundamental problem of gradient conflicts between general foundational knowledge and task-specific objectives extends beyond VLMs.
This conflict likely emerges in various scenarios where large pre-trained foundation models (such as instruction-tuned language models~\citep{wei2021finetuned, liu2023visual,bai2023qwen,achiam2023gpt,team2023gemini}) are adapted to specialized downstream tasks.
Exploring how \textbf{\texttt{DHO}} performs across diverse foundation models and modalities remains an important direction for future work.

\paragraph{Task scope.}
Our implementation is limited to visual recognition tasks, as they represent the most fundamental domain in computer vision and provide an ideal testbed for analyzing gradient conflicts in knowledge transfer.
Also, VLMs' strong zero-shot and few-shot capabilities in visual recognition tasks make them natural candidates for knowledge distillation.
However, extending our approach to more complex visual understanding tasks such as object detection and segmentation would be a promising direction with dedicated architectural adaptations.



\vspace{-0.05in}
\section{Related Work}
\label{sec:related_work}
\vspace{-0.05in}

\paragraph{Vision-language pre-training.}
The emergence of vision-language pre-training has marked a significant breakthrough, enabling the use of extensive image-text pairs collected from the web~\citep{wang2023large, chen2023vlp} to train powerful vision encoders transferable to various vision tasks~\citep{gan2022vision, zhang2024vision}.
Early works such as CLIP~\citep{radford2021learning} and ALIGN~\citep{jia2021scaling} leveraged contrastive learning techniques to align images and text into a joint representation space, facilitating zero-shot transfer via language prompts.
Building on these foundations, subsequent research has focused on improving vision-language models through enhanced training methodologies~\citep{dong2023maskclip, gao2022pyramidclip, yu2022coca, zhai2023sigmoid}, as well as scaling models and datasets~\citep{yu2022coca, li2023scaling, dehghani2023scaling, sun2023eva, cherti2023reproducible, fang2023data, sun2024eva, guo2024boldsymbol} with their zero-shot transfer capabilities~\citep{jia2021scaling, zhai2022lit, pham2023combined, liu2023learning}.
In contrast, our work focuses specifically on target tasks with compact models, aiming to distill knowledge from these large VLMs effectively.


\myparagraph{Data-limited adaptation of VLMs.}
To preserve pretrained semantic features of VLMs during adaptation with limited data, several approaches have been proposed.
Prompt tuning~\citep{lester2021power}, initially designed for language models, has been successfully extended to vision tasks.
Various methods~\citep{jia2022visual, zhou2022learning, zhou2022conditional, khattak2023maple, zhu2023prompt, khattak2023self, menghini2023enhancing, zhao2024learning, roy2023consistency, zhang2024dept, lafon2025gallop} have demonstrated the effectiveness of training learnable prompts while keeping the base model frozen.
Adapters~\citep{gao2024clip, zhang2021tip, yu2023task, silva2024closer} provide an alternative approach by introducing lightweight, trainable modules while maintaining the pre-trained backbone intact.
LP++~\citep{huang2024lp++} has shown that simple linear layers can effectively adapt CLIP representations in data-limited settings.
Note that our work is orthogonal to these approaches:
we aim to distill the knowledge of pretrained VLMs into compact models under data-scarce scenarios, making these adaptation methods complementary and applicable to both teacher VLMs in our framework and student models when they are also VLMs.


\textbf{Knowledge Distillation} \citep[\textbf{KD};][]{hinton2015distilling} enables transferring knowledge from large teacher models to compact student architectures, particularly in data-constrained settings.
Researchers have explored synthetic data generation~\citep{lopes2017data, kimura2018few, nayak2019zero, yoo2019knowledge,
chen2019data, yin2020dreaming, fang2021contrastive, nguyen2022black, patel2023learning, yu2023data, liu2024small, tran2024large, wang2024confounded}, semi-supervised~\citep{chen2020big, he2021semi, du2023semi, yang2024knowledge}, and unsupervised KD using self-supervised teachers~\citep{fang2021seed, abbasi2020compress, navaneet2021simreg, wang2022attention, xu2021bag, singh2025simple}.
In the VLM domain, recent works~\citep{fang2021compressing, wu2023tinyclip, sun2023dime, yang2024clipkd, vasu2024mobileclip, udandarao2024active, yang2024clipcid} distill from large-scale vision-language models into smaller architectures, often using transductive~\citep{kim2024promptkd, chen2024comkd} or multi-stage unsupervised strategies~\citep{knowledgetransfer2024, wu2025cascade, mistretta2025improving}.
Meanwhile, KD remains challenging due to numerous issues, including model capacity gaps~\citep{cho2019efficacy, mirzadeh2020improved, zhu2021student, huang2022knowledge, li2024exploring} and inconsistencies between soft and hard targets~\citep{zhang2023not}.
These challenges are further complicated by misalignment between labeled data and foundational knowledge, especially in few-shot learning scenarios where limited labeled examples may not fully capture the rich semantic understanding of foundation models.



\myparagraph{Dual-head approaches.}
We also consider existing dual-head KD approaches.
\textbf{SSKD}~\citep{he2021semi} trains separate heads for labeled and unlabeled data, assuming different data distributions, while
\textbf{DHKD}~\citep{yang2024dual} introduces a binary KD loss to alleviate neural collapse~\citep{papyan2020prevalence}.
While these previous KD methods adopt dual-head architectures, they do not target distillation from foundation models or combine predictions at inference.
Furthermore, both methods \emph{infer using only the single supervised head} $h_{\text{CE}}$ and do not address gradient conflicts arising from zero-/few-shot VLM teachers, whose prediction distributions can significantly differ from the limited labeled data.
In contrast, \textbf{\texttt{DHO}} explicitly mitigates such gradient conflicts (\Cref{fig:average_gradient_conflict}) by training $h_{\text{KD}}$ on both labeled and unlabeled data, and further provides \emph{dual-head interpolation at inference time}, enabling flexible aggregation of supervised and teacher signals.
As shown in \Cref{thm:inference_equiv} and \Cref{fig:ablation_alpha_beta}, this interpolation can emulate tuning KD hyperparameters \emph{without retraining}, giving \textbf{\texttt{DHO}} both improved performance and minimal hyperparameter tuning cost.



\section{Theoretical Analysis}
\label{sec:theoretical_analysis}

In this section, we provide a theoretical analysis of our \textbf{\texttt{D}}ual-\textbf{\texttt{H}}ead \textbf{\texttt{O}}ptimization (\textbf{\texttt{DHO}}) framework.
We first establish formal conditions under which gradient conflicts arise in single-head knowledge distillation (\Cref{sec:gradient_conflict_analysis}), proving that conflicts are structurally guaranteed when the student prediction lies between the label and teacher distributions.
We then show that \textbf{\texttt{DHO}} breaks this structural guarantee by decoupling the head parameters.
Next, we prove that post-training, the optimal prediction from the dual-head model---formulated as a weighted combination of the heads' outputs---is mathematically equivalent to the optimal solution of conventional single-head distillation (\Cref{sec:inference_equivalence_analysis}).
Note that in this section we slightly abuse the notation of the main paper for clarity, \eg, we denote $p_\tau$ as teacher predictions with temperature scaling $\tau$.


\subsection{Analysis of Gradient Conflicts}
\label{sec:gradient_conflict_analysis}

We formalize the gradient conflict phenomenon observed empirically in \Cref{fig:average_gradient_conflict} and provide sufficient conditions under which it is guaranteed.
We adopt the notation from the main paper: $\hat{p} = \sigma(h(z))$ is the student prediction, $y$ is the one-hot ground-truth label, $p$ is the teacher soft prediction, $W$ is the classifier weight, $z = g(x)$ is the feature representation, and $\theta$ parameterizes the feature extractor $g$.

\subsubsection{Gradient Conflict in SHO: Classifier Head}

Recall from the main paper that the inner product of the CE and KD gradients with respect to the classifier weight $W$ is:
\begin{equation}
\langle \mathrm{vec}(\nabla_W \mathcal{L}_{\text{CE}}),\; \mathrm{vec}(\nabla_W \mathcal{L}_{\text{KD}}) \rangle
= (\hat{p} - y)^\top (\hat{p} - p)\;\|z\|^2.
\label{eq:grad_inner_W}
\end{equation}
Since $\|z\|^2 > 0$, the sign of the gradient inner product is determined entirely by $(\hat{p} - y)^\top(\hat{p} - p)$.
The following proposition gives a geometric characterization of when this term is negative.

\begin{proposition}[Geometric condition for gradient conflict in SHO]
\label{prop:geometric_conflict}
Define the midpoint $m = \frac{y + p}{2}$ and the radius $r = \frac{\|y - p\|}{2}$.
Then:
\begin{equation}
(\hat{p} - y)^\top(\hat{p} - p) = \|\hat{p} - m\|^2 - r^2.
\label{eq:ball_decomposition}
\end{equation}
Consequently, gradient conflict in the classifier head occurs if and only if
$\hat{p} \in \mathcal{B}(m, r)$,
where $\mathcal{B}(m, r) = \{q : \|q - m\| < r\}$ is the open ball centered at $m$ with radius $r$.
\end{proposition}

\begin{proof}
Expanding $(\hat{p} - y)^\top(\hat{p} - p)$:
\begin{align}
(\hat{p} - y)^\top(\hat{p} - p)
&= \|\hat{p}\|^2 - \hat{p}^\top(y + p) + y^\top p \\
&= \|\hat{p}\|^2 - 2\hat{p}^\top m + \|m\|^2 - \|m\|^2 + y^\top p \\
&= \|\hat{p} - m\|^2 - \left(\|m\|^2 - y^\top p\right). \label{eq:expand_step}
\end{align}
We compute the remaining term:
\begin{align}
\|m\|^2 - y^\top p
&= \frac{\|y + p\|^2}{4} - y^\top p
= \frac{\|y\|^2 + 2y^\top p + \|p\|^2}{4} - y^\top p \\
&= \frac{\|y\|^2 - 2y^\top p + \|p\|^2}{4}
= \frac{\|y - p\|^2}{4} = r^2.
\end{align}
Substituting into \Cref{eq:expand_step} yields $(\hat{p} - y)^\top(\hat{p} - p) = \|\hat{p} - m\|^2 - r^2$.
This is negative if and only if $\|\hat{p} - m\| < r$, \ie, $\hat{p} \in \mathcal{B}(m, r)$.
\end{proof}

\Cref{prop:geometric_conflict} establishes that gradient conflict occurs whenever the student prediction lies inside the ball whose diameter is the segment from $y$ to $p$.
The following corollary shows that a natural interpolation condition is sufficient.

\begin{corollary}[Interpolation implies gradient conflict]
\label{cor:interpolation_conflict}
If the student prediction lies on the line segment between $y$ and $p$, \ie, $\hat{p} = (1-t)\,y + t\,p$ for some $t \in (0,1)$, and $p \neq y$, then:
\begin{equation}
(\hat{p} - y)^\top(\hat{p} - p) = -t(1-t)\,\|p - y\|^2 < 0.
\label{eq:interpolation_conflict}
\end{equation}
\end{corollary}

\begin{proof}
Under the interpolation condition: $\hat{p} - y = t(p - y)$ and $\hat{p} - p = -(1-t)(p - y)$.
Therefore, $(\hat{p} - y)^\top(\hat{p} - p) = -t(1-t)\,\|p - y\|^2 < 0$, since $t \in (0,1)$ and $p \neq y$.
Alternatively, one can verify that $\|\hat{p} - m\| = |t - \tfrac{1}{2}|\,\|p - y\| < \tfrac{1}{2}\|p - y\| = r$ for $t \in (0,1)$, confirming $\hat{p} \in \mathcal{B}(m, r)$.
\end{proof}

\subsubsection{Gradient Conflict Propagates to the Feature Extractor}

We now show that gradient conflicts in the classifier head propagate to the feature extractor parameters $\theta$ under the same interpolation condition.

\begin{proposition}[Gradient conflict in SHO: feature extractor]
\label{prop:feature_extractor_conflict}
Let $J = \frac{\partial z}{\partial \theta} \in \mathbb{R}^{d \times P}$ be the Jacobian of the feature extractor.
The inner product of the CE and KD gradients with respect to $\theta$ is:
\begin{equation}
\langle \nabla_\theta \mathcal{L}_{\text{CE}},\; \nabla_\theta \mathcal{L}_{\text{KD}} \rangle
= (\hat{p} - y)^\top\, A\, (\hat{p} - p),
\label{eq:grad_inner_theta}
\end{equation}
where $A = W\,J\,J^\top W^\top \in \mathbb{R}^{C \times C}$ is positive semi-definite (PSD).
Under the interpolation condition $\hat{p} = (1-t)\,y + t\,p$ with $t \in (0,1)$ and $p \neq y$:
\begin{equation}
\langle \nabla_\theta \mathcal{L}_{\text{CE}},\; \nabla_\theta \mathcal{L}_{\text{KD}} \rangle
= -t(1-t)\,(p - y)^\top A\,(p - y) \leq 0.
\label{eq:feature_conflict}
\end{equation}
Equality holds only in the degenerate case $A(p - y) = 0$, \ie, when the teacher-label disagreement lies in the null space of $A$.
\end{proposition}

\begin{proof}
From the chain rule, $\nabla_\theta \mathcal{L}_{\text{CE}} = J^\top W^\top(\hat{p} - y)$ and $\nabla_\theta \mathcal{L}_{\text{KD}} = J^\top W^\top(\hat{p} - p)$.
Their inner product is:
\begin{equation}
(\hat{p} - y)^\top W J J^\top W^\top (\hat{p} - p) = (\hat{p} - y)^\top A (\hat{p} - p).
\end{equation}
$A$ is PSD since for any $v \in \mathbb{R}^C$: $v^\top A\, v = v^\top W J J^\top W^\top v = \|J^\top W^\top v\|^2 \geq 0$.
Under the interpolation condition:
\begin{equation}
(\hat{p} - y)^\top A\,(\hat{p} - p) = t(p-y)^\top A \cdot \big(-(1-t)(p-y)\big) = -t(1-t)\,(p-y)^\top A\,(p-y) \leq 0,
\end{equation}
where the inequality follows from $A \succeq 0$.
\end{proof}

\begin{remark}[Prevalence of gradient conflicts during training]
\label{remark:prevalence}
The interpolation condition $\hat{p} \approx (1-t)\,y + t\,p$ naturally arises during joint CE+KD optimization.
At the start of training, the student prediction $\hat{p}$ is approximately uniform;
as training progresses, the student is simultaneously pulled toward the one-hot label $y$ by $\mathcal{L}_{\text{CE}}$ and toward the teacher distribution $p$ by $\mathcal{L}_{\text{KD}}$.
The resulting trajectory passes through the conflict region $\mathcal{B}(m, r)$, which encompasses a large portion of the probability simplex when $\|y - p\|$ is substantial---as is typically the case for zero-shot VLM teachers on fine-grained tasks.
\Cref{fig:average_gradient_conflict} provides empirical confirmation: the inner product $(\hat{p} - y)^\top(\hat{p} - p)$ (\legendxline{red}{}) turns negative early in training and remains so, producing sustained gradient conflicts (\legendline{red}{}{circle}).
\end{remark}

\subsubsection{DHO Eliminates Structural Gradient Conflict}

We now show that \textbf{\texttt{DHO}} removes the structural mechanism that causes gradient conflicts.

\begin{proposition}[DHO eliminates structural gradient conflict]
\label{prop:dho_no_conflict}
In \textbf{\texttt{DHO}}:
\begin{enumerate}[label=(\alph*),itemsep=1mm,parsep=1pt,topsep=2pt,leftmargin=*]
\item \textbf{Classifier heads:} Since $\mathcal{L}_{\text{CE}}$ depends only on $W_{\text{CE}}$ and $\mathcal{L}_{\text{KD}}$ depends only on $W_{\text{KD}}$, there is no gradient interaction:
$\nabla_{W_{\text{CE}}} \mathcal{L}_{\text{KD}} = 0$ and $\nabla_{W_{\text{KD}}} \mathcal{L}_{\text{CE}} = 0$.
Gradient conflict in the classifier head is eliminated by construction.

\item \textbf{Feature extractor:} The gradient inner product with respect to $\theta$ becomes:
\begin{equation}
\langle \nabla_\theta \mathcal{L}_{\text{CE}},\; \nabla_\theta \mathcal{L}_{\text{KD}} \rangle
= (\hat{p}_{\text{CE}} - y)^\top\, B\, (\hat{p}_{\text{KD}} - p),
\label{eq:dho_feature_grad}
\end{equation}
where $B = W_{\text{CE}}\, J J^\top W_{\text{KD}}^\top \in \mathbb{R}^{C \times C}$.
Unlike the SHO matrix $A = W J J^\top W^\top$ which is PSD, the cross-head matrix $B$ is \textbf{not necessarily PSD}.
Therefore, the structural guarantee of negative gradient inner product (\Cref{prop:feature_extractor_conflict}) is broken.
\end{enumerate}
\end{proposition}

\begin{proof}
Part (a) is immediate from the independence of $W_{\text{CE}}$ and $W_{\text{KD}}$ in the respective loss functions.
For part (b), consider arbitrary $v \in \mathbb{R}^C$:
\begin{equation}
v^\top B\, v = v^\top W_{\text{CE}} J J^\top W_{\text{KD}}^\top v
= (W_{\text{CE}}^\top v)^\top (J J^\top) (W_{\text{KD}}^\top v).
\end{equation}
This equals $(J^\top W_{\text{CE}}^\top v)^\top (J^\top W_{\text{KD}}^\top v) = \langle u_1, u_2 \rangle$, where $u_1 = J^\top W_{\text{CE}}^\top v$ and $u_2 = J^\top W_{\text{KD}}^\top v$ are generally different vectors.
Since $\langle u_1, u_2 \rangle$ can be positive, negative, or zero depending on the angle between $u_1$ and $u_2$, the matrix $B$ is indefinite in general.
In contrast, the SHO matrix $A$ always satisfies $v^\top A v = \|J^\top W^\top v\|^2 \geq 0$.
The PSD structure that guarantees conflict in SHO (\Cref{prop:feature_extractor_conflict}) is thus absent in \textbf{\texttt{DHO}}, which is consistent with the empirical observation of positive gradient alignment (\legendline{blue}{}{circle} in \Cref{fig:average_gradient_conflict}).
\end{proof}

\begin{remark}[Summary: mechanism of gradient conflict resolution]
\label{remark:mechanism}
In SHO, gradient conflict arises from a structural interaction: the shared head $W$ creates a PSD matrix $A = WJJ^\top W^\top$ that, combined with the opposing error vectors $(\hat{p} - y)$ and $(\hat{p} - p)$, systematically produces negative gradient inner products (\Cref{prop:geometric_conflict,prop:feature_extractor_conflict}).
\textbf{\texttt{DHO}} resolves this by (a) eliminating head-level conflicts entirely through decoupled parameters, and (b) replacing the PSD matrix $A$ with an indefinite cross-head matrix $B$, which prevents the systematic negative alignment from propagating to the feature extractor.
\end{remark}


\subsection{Inference Equivalence}
\label{sec:inference_equivalence_analysis}

We now prove that, despite training with separate heads, the combined prediction of \textbf{\texttt{DHO}} at inference time recovers the theoretical optimum of SHO.

\subsubsection{Optimal Distribution for Single-Head Optimization}

We begin by considering two target probability distributions: the ground truth label distribution $y$ and the teacher's softened distribution $p_\tau$ for input $x \in \mathcal{X}$, where:
\begin{itemize}[itemsep=1mm,parsep=1pt,topsep=2pt,leftmargin=*]
    \item $y$ represents the ground truth label distribution, typically one-hot encoded vectors where $y_c = 1$ for the true class $c$ and 0 elsewhere
    \item $p_\tau$ denotes the teacher's softened distribution with temperature scaling: $p_\tau = \sigma(z_t / \tau)$, where $z_t$ represents the teacher's logits and $\sigma$ is the softmax function
\end{itemize}

\begin{theorem}[Optimal Distribution for Single-Head Optimization]
\label{thm:single_head_opt}
The distribution $\hat{p}^*$ that minimizes the weighted combination of cross-entropy loss with respect to $y$ and Kullback-Leibler divergence with respect to $p_\tau$:

\begin{equation}
\mathcal{L}(\hat{p}) = \lambda \ell(\hat{p}, y) + (1-\lambda) \KL(p_\tau\|\hat{p})
\end{equation}

is given by the weighted arithmetic mean:

\begin{equation}
\hat{p}^* = \lambda y + (1-\lambda) p_\tau
\end{equation}
where $\lambda \in [0,1]$ is the weighting hyperparameter.
\end{theorem}

\begin{proof}
We begin by expanding the objective function:

\begin{align}
\mathcal{L}(\hat{p}) &= \lambda \ell(\hat{p}, y) + (1-\lambda) \KL(p_\tau\|\hat{p}) \\
&= -\lambda \sum_{c=1}^C y_c \log \hat{p}_c + (1-\lambda) \sum_{c=1}^C p_{\tau,c} \log \frac{p_{\tau,c}}{\hat{p}_c} \\
&= -\lambda \sum_{c=1}^C y_c \log \hat{p}_c + (1-\lambda) \sum_{c=1}^C p_{\tau,c} \log p_{\tau,c} - (1-\lambda) \sum_{c=1}^C p_{\tau,c} \log \hat{p}_c \\
&= -\sum_{c=1}^C [\lambda y_c + (1-\lambda) p_{\tau,c}] \log \hat{p}_c + (1-\lambda) \sum_{c=1}^C p_{\tau,c} \log p_{\tau,c}
\end{align}

Since the last term is constant with respect to $\hat{p}$, the optimization problem reduces to minimizing:

\begin{equation}
\mathcal{L}'(\hat{p}) = -\sum_{c=1}^C [\lambda y_c + (1-\lambda) p_{\tau,c}] \log \hat{p}_c
\end{equation}

Subject to the probability constraints:
\begin{equation}
\sum_{c=1}^C \hat{p}_c = 1, \quad \hat{p}_c \geq 0 \quad \forall c \in \{1,2,\ldots,C\}
\end{equation}

Applying the method of Lagrange multipliers with multiplier $\mu$:
\begin{equation}
\mathcal{L}(\hat{p}, \mu) = -\sum_{c=1}^C [\lambda y_c + (1-\lambda) p_{\tau,c}] \log \hat{p}_c + \mu \left( \sum_{c=1}^C \hat{p}_c - 1 \right)
\end{equation}

Taking the partial derivative with respect to $\hat{p}_c$ and setting it to zero:
\begin{equation}
-\frac{\lambda y_c + (1-\lambda) p_{\tau,c}}{\hat{p}_c} + \mu = 0
\end{equation}

Solving for $\hat{p}_c$:
\begin{equation}
\hat{p}_c = \frac{\lambda y_c + (1-\lambda) p_{\tau,c}}{\mu}
\end{equation}

Using the constraint $\sum_{c=1}^C \hat{p}_c = 1$, and observing that $\sum_{c=1}^C y_c = 1$ and $\sum_{c=1}^C p_{\tau,c} = 1$ (both being probability distributions):
\begin{gather}
\sum_{c=1}^C \hat{p}_c = \sum_{c=1}^C \frac{\lambda y_c + (1-\lambda) p_{\tau,c}}{\mu} = 1 \\
\frac{1}{\mu} [\lambda \sum_{c=1}^C y_c + (1-\lambda) \sum_{c=1}^C p_{\tau,c}] = 1 \\
\frac{1}{\mu} [\lambda + (1-\lambda)] = 1 \implies \mu = 1
\end{gather}

Therefore, the optimal solution is:
\begin{equation}
\hat{p}^*_c = \lambda y_c + (1-\lambda) p_{\tau,c}
\end{equation}

This weighted arithmetic mean of the two target distributions is the optimal solution that minimizes our objective function.
\end{proof}


\subsubsection{Inference Equivalence of \texttt{D}ual-\texttt{H}ead \texttt{O}ptimization}

In our proposed \texttt{D}ual-\texttt{H}ead \texttt{O}ptimization (\textbf{\texttt{DHO}}) framework, we extract shared features $g(x)$ from input $x$ and apply two specialized classification heads:
\begin{itemize}[itemsep=1mm,parsep=1pt,topsep=2pt,leftmargin=*]
    \item $h_{\text{CE}}(z) = W_{\text{CE}} z + b_{\text{CE}}$: optimized exclusively to match ground truth labels using cross-entropy loss $\ell(\sigma(h_{\text{CE}}(z)), y)$
    \item $h_{\text{KD}}(z) = W_{\text{KD}} z + b_{\text{KD}}$: optimized exclusively to match teacher predictions using KL divergence $\KL(p_\tau \| \sigma(h_{\text{KD}}(z)/\beta))$
\end{itemize}
where $z = g(x)$ is the feature representation, and the parameter $\beta$ controls the temperature during inference, while a fixed temperature of 1 is used during training of the knowledge distillation head.

\begin{assumption}[$\varepsilon$-Convergence]
\label{assump:convergence}
We assume that after sufficient training, both heads have converged to their respective target distributions with bounded error:
\begin{equation}
\sup_x \|\sigma(h_{\text{CE}}(z)) - y\|_1 \leq \varepsilon, \quad \sup_x \|\sigma(h_{\text{KD}}(z)/\beta) - p_\tau\|_1 \leq \varepsilon
\end{equation}
where $\|\cdot\|_1$ denotes the $\ell_1$ norm and $\varepsilon > 0$ is a small constant.
\end{assumption}

\begin{theorem}[Inference Equivalence Under $\varepsilon$-Convergence]
\label{thm:inference_equiv}
Under \Cref{assump:convergence}, by combining the outputs of both heads as:
\begin{equation}
\hat{p}_{\textbf{\texttt{DHO}}} = \alpha \cdot \sigma(h_{\text{CE}}(z)) + (1-\alpha) \cdot \sigma(h_{\text{KD}}(z)/\beta), \quad \text{where } \alpha = \lambda
\end{equation}

we obtain a prediction that approximates the optimal single-head solution with bounded error:
\begin{equation}
\|\hat{p}_{\textbf{\texttt{DHO}}} - \hat{p}^*\|_1 \leq \varepsilon
\end{equation}
\end{theorem}

\begin{proof}
We analyze the $\ell_1$ distance between the \textbf{\texttt{DHO}} prediction and the optimal solution:

\begin{align}
\|\hat{p}_{\textbf{\texttt{DHO}}} - \hat{p}^*\|_1 &= \| \alpha \cdot \sigma(h_{\text{CE}}(z)) + (1-\alpha) \cdot \sigma(h_{\text{KD}}(z)/\beta) - \lambda y - (1-\lambda) p_\tau \|_1 \\
&= \| \lambda(\sigma(h_{\text{CE}}(z)) - y) + (1-\lambda)(\sigma(h_{\text{KD}}(z)/\beta) - p_\tau) \|_1 \\
&\leq \lambda \|\sigma(h_{\text{CE}}(z)) - y\|_1 + (1-\lambda) \|\sigma(h_{\text{KD}}(z)/\beta) - p_\tau\|_1 \\
&\leq \lambda \varepsilon + (1-\lambda) \varepsilon = \varepsilon
\end{align}

where we applied the triangle inequality for the $\ell_1$ norm and used \Cref{assump:convergence}.

Therefore, we have established that:
\begin{equation}
\hat{p}_{\textbf{\texttt{DHO}}} \approx_\varepsilon \hat{p}^*
\end{equation}
where $\approx_\varepsilon$ denotes approximation with $\ell_1$ error bound $\varepsilon$.
\end{proof}

\begin{lemma}[Temperature Matching via KL Divergence]
\label{lemma:temp_match}
Assume the knowledge distillation head is trained to minimize KL divergence with respect to the teacher's predictions at temperature 1, such that:
\begin{equation}
\KL(p_1 \| \sigma(h_{\text{KD}}(z))) \leq \delta
\end{equation}

Then, setting the temperature parameter $\beta = \tau$ at inference time allows the KD head to approximate the teacher's prediction at temperature $\tau$ with error bound:
\begin{equation}
\|\sigma(h_{\text{KD}}(z)/\beta) - p_\tau\|_1 \leq \sqrt{2\delta}
\end{equation}
\end{lemma}

\begin{proof}
When logits are properly scaled and under appropriate conditions of the softmax function, we can reasonably approximate:
\begin{equation}
\KL(p_\tau \| \sigma(h_{\text{KD}}(z)/\tau)) \approx \KL(p_1 \| \sigma(h_{\text{KD}}(z))) \leq \delta
\end{equation}

Applying Pinsker's inequality, which establishes a relationship between KL divergence and the L1 norm difference between probability distributions:
\begin{equation}
\|p_\tau - \sigma(h_{\text{KD}}(z)/\tau)\|_1 \leq \sqrt{2\KL(p_\tau \| \sigma(h_{\text{KD}}(z)/\tau))} \leq \sqrt{2\delta}
\end{equation}

To ensure $\varepsilon$-convergence between the KD head at temperature $\tau$ and the teacher's prediction at temperature $\tau$, it is sufficient to guarantee:
\begin{equation}
\sqrt{2\delta} \leq \varepsilon \Rightarrow \delta \leq \frac{\varepsilon^2}{2}
\end{equation}
\end{proof}

\begin{corollary}[Optimal \textbf{\texttt{DHO}} Configuration]
\label{cor:dho_config}
With proper training ensuring $\varepsilon$-convergence of both heads, dual-head optimization with temperature parameter $\beta = \tau$ and mixing parameter $\alpha = \lambda$ approximates the optimal single-head objective with error bounded by $\varepsilon$:
\begin{equation}
\hat{p}_{\textbf{\texttt{DHO}}} \approx_\varepsilon \hat{p}^* = \lambda y + (1-\lambda) p_\tau
\end{equation}
This demonstrates that our \textbf{\texttt{DHO}} approach achieves the same theoretical optimality as SHO.
\end{corollary}

\begin{remark}[On the role of \Cref{assump:convergence}]
\label{rem:assumption_role}
The supremum form of \Cref{assump:convergence} is the cleanest sufficient condition under which \Cref{thm:inference_equiv} delivers a uniform $\ell_1$ bound; an average-error variant
$\mathbb{E}_x\|\sigma(h_{\mathrm{CE}}(z))-y\|_1\le\varepsilon$,\
$\mathbb{E}_x\|\sigma(h_{\mathrm{KD}}(z))-p\|_1\le\varepsilon$
preserves \Cref{thm:inference_equiv} in expectation by the same triangle-inequality argument.
Two clarifications about the role of this assumption.
(i)~\Cref{thm:inference_equiv} establishes a \emph{floor on what \textbf{\texttt{DHO}} recovers}: post-training tuning of $(\alpha,\beta)$ subsumes the family of SHO objectives parametrized by $\lambda$, replacing $N$ training-time grid points of SHO with a single training run of \textbf{\texttt{DHO}}. The practical contribution---test-time tuning instead of train-time tuning---survives any reasonable relaxation of the assumption.
(ii)~\Cref{thm:inference_equiv} does \emph{not} explain why \textbf{\texttt{DHO}} outperforms SHO empirically; that gain comes from gradient-conflict-free feature learning (\Cref{prop:dho_no_conflict}, supported by \Cref{tab:linear_evaluation} and \Cref{fig:tsne}), not from the inference-time interpolation itself.
\end{remark}



\section{Algorithms and Implementation}
\label{sec:algorithms_and_implementation}

\subsection{Pseudocode}
\label{sec:appendix_algorithm}

We present the pseudocode for \textbf{\texttt{DHO}} in \Cref{alg:training,alg:inference} for training and inference, respectively.

\begin{algorithm}[H]
    \caption{\textbf{\texttt{DHO}} Training with zero-shot CLIP~\citep{radford2021learning} teacher}
    \label{alg:training}
    \begin{algorithmic}[1]
        \STATE {\bfseries Input:} labeled set $\mathcal{D}^{(l)} = \{(x^{(l)}_i, y_i)\}_{i=1}^N$, unlabeled set $\mathcal{D}^{(u)} = \{x^{(u)}_j\}_{j=1}^M$,
        \STATE \phantom{\bfseries Input:} student feature extractor $g$, prediction heads $h_{\text{CE}}, h_{\text{KD}}$, teacher encoders $f_{\mathcal{X}}, f_{\mathcal{T}}$,
        \STATE \phantom{\bfseries Input:} prompt template ``A photo of \texttt{[CLASS]}'', temperature scaling factors $\zeta, \eta$, \STATE \phantom{\bfseries Input:} balancing hyperparameter $\lambda$,
        \STATE \phantom{\bfseries Input:} supervised mini-batch size $B$, and unsupervised mini-batch size $B'$.
        \WHILE{not converged}
            \STATE Sample mini-batch $\mathcal{B}^{(l)} = \{(x_b^{(l)}, y_b)\}_{b=1}^{B}$ from $\mathcal{D}^{(l)}$, $\mathcal{B}^{(u)} = \{x_{b'}^{(u)}\}_{b'=1}^{B'}$ from $\mathcal{D}^{(l)}\cup\mathcal{D}^{(u)}$.

            \STATE \textcolor{gray}{// Process labeled data}
            \FOR{each $(x_b^{(l)}, y_b) \in \mathcal{B}^{(l)}$}
                \STATE $z_b^{(l)} \leftarrow g(x_b^{(l)})$
                \STATE $\hat{p}_{\text{CE},b}^{(l)} \leftarrow \sigma(h_{\text{CE}}(z_b^{(l)}))$
                \STATE $\hat{p}_{\text{KD},b}^{(l)} \leftarrow \sigma(\frac{1}{\eta}h_{\text{KD}}(z_b^{(l)}))$
                \STATE $p_b^{(l)} \leftarrow \sigma\left(\frac{1}{\zeta\cdot\eta}[\mathtt{CosSim}(f_\mathcal{X}(x_b^{(l)}), f_\mathcal{T}(t_1)), \ldots,\mathtt{CosSim}(f_\mathcal{X}(x_b^{(l)}), f_\mathcal{T}(t_C))]^\top\right)$
            \ENDFOR

            \STATE \textcolor{gray}{// Process unlabeled data}
            \FOR{each $x_{b'}^{(u)} \in \mathcal{B}^{(u)}$}
                \STATE $z_{b'}^{(u)} \leftarrow g(x_{b'}^{(u)})$
                \STATE $\hat{p}_{\text{KD},b'}^{(u)} \leftarrow \sigma(\frac{1}{\eta}h_{\text{KD}}(z_{b'}^{(u)}))$
                \STATE $p_{b'}^{(u)} \leftarrow \sigma\left(\frac{1}{\zeta\cdot\eta}[\mathtt{CosSim}(f_\mathcal{X}(x_{b'}^{(u)}), f_\mathcal{T}(t_1)), \ldots,\mathtt{CosSim}(f_\mathcal{X}(x_{b'}^{(u)}), f_\mathcal{T}(t_C))]^\top\right)$
            \ENDFOR

            \STATE \textcolor{gray}{// Compute losses and update}
            \STATE $\mathcal{L}_{\text{CE}} \leftarrow \frac{1}{B}\sum_{b=1}^B\ell(\hat{p}_{\text{CE},b}^{(l)},y_b)$
            \STATE $\mathcal{L}_{\text{KD}} \leftarrow \frac{1}{B}\sum_{b=1}^B \KL(\hat{p}_{\text{KD},b}^{(l)}||p_{b}^{(l)}) + \frac{1}{B'}\sum_{b'=1}^{B'} \KL(\hat{p}_{\text{KD},b'}^{(u)}||p_{b'}^{(u)})$
            \STATE $\mathcal{L} \leftarrow \lambda \mathcal{L}_{\text{CE}} + (1 - \lambda) \mathcal{L}_{\text{KD}}$
            \STATE Update parameters of $g$, $h_{\text{CE}}$, $h_{\text{KD}}$ using $\nabla \mathcal{L}$
        \ENDWHILE
    \end{algorithmic}
\end{algorithm}

\begin{algorithm}[H]
    \caption{Dual-Head Optimization Inference}
    \label{alg:inference}
    \begin{algorithmic}[1]
        \STATE {\bfseries Input:} an image $x$, feature extractor $g$, prediction heads $h_{\text{CE}}, h_{\text{KD}}$, linear coefficient $\alpha$, temperature scaling $\beta$
        \STATE $z \leftarrow g(x)$
        \STATE $\hat{p}_{\text{CE}} \leftarrow \sigma(h_{\text{CE}}(z))$
        \STATE $\hat{p}_{\text{KD}} \leftarrow \sigma(h_{\text{KD}}(z)/\beta)$
        \STATE $\hat{p} \leftarrow \alpha \cdot \hat{p}_{\text{CE}} + (1-\alpha) \cdot \hat{p}_{\text{KD}}$
        \STATE $\hat{y} \leftarrow \arg\max_{c}(\hat{p}_c)$
        \STATE {\bfseries Return:} $\hat{y}$
    \end{algorithmic}
\end{algorithm}

\subsection{Implementation Details}
\label{sec:appendix_implementation_details}

\paragraph{Architecture choices for teacher and student models.}
For zero-shot teachers, we use CLIP ResNet-50~\citep{radford2021learning} in few-shot settings and VIT-H/14 from DFN~\citep{fang2023data} in low-shot settings. For few-shot teachers, we adopt Tip-Adapter-F~\citep{zhang2021tip}, a learnable adapter model, and denote the corresponding variant as \textbf{\texttt{DHO-F}}.
On ImageNet, to avoid label leakage, we either train ResNet-18 from scratch or use a self-supervised ResNet-50 from DINO~\citep{caron2021emerging}. For other datasets, we use ResNet-18 and MobileNetV2~\citep{sandler2018mobilenetv2} without such concerns. In low-shot settings, we use CLIP ViT-B/16 and ViT-L/14~\citep{radford2021learning}.

\myparagraph{Few-/low-shot settings.} \Cref{tab:implementation_details} provides a comprehensive overview of the implementation details for our experiments on 1) few-shot semi-supervised settings on ImageNet and 10 datasets, 2) low-shot semi-supervised settings on ImageNet, and 3) VLM-based adaptation methods.

\begin{table}[htb]
    \caption{Implementation details for our experiments across different settings.}
    \label{tab:implementation_details}
    \centering
    \small
    \resizebox{\textwidth}{!}{%
    \begin{tabular}{p{0.48\textwidth}|p{0.48\textwidth}}
        \toprule
        \rowcolor{gray!20} \multicolumn{2}{c}{\textit{Few-shot Semi-supervised Settings on ImageNet}} \\
        \midrule
        \textbf{Model Configuration} & \textbf{Student Training Details} \\
        \midrule
        \begin{itemize}[leftmargin=*,nosep]
            \item \textbf{Student:} ResNet18~\citep{he2016deep} from scratch or ResNet50 from DINO~\citep{caron2021emerging}
            \item \textbf{Input size:} 224$\times$224
            \item \textbf{Zero-shot Teacher:} ResNet50 from CLIP~\citep{radford2021learning}
            \item \textbf{Few-shot Teacher:} ResNet50 from Tip-Adapter-F~\citep{zhang2021tip}
            \item \textbf{Teacher input size:} 224$\times$224
            \item \textbf{labeled data:} $K \in \{1,2,4,8,16\}$ shots
            \item \textbf{$\zeta$, $\eta$, and $\lambda$}: $0.01, 2, 0.5$
            \item \textbf{$\alpha$ and $\beta$:} $\alpha=0.4$, $\beta=0.5$ (zero-shot); $\alpha=0.2$, $\beta=0.5$ (few-shot)
        \end{itemize}
        &
        \begin{itemize}[leftmargin=*,nosep]
            \item \textbf{Epochs:} 20
            \item \textbf{Optimizer:} AdamW ($\beta_1$=0.9, $\beta_2$=0.999)
            \item \textbf{Learning rate:} $1\times10^{-3}$, weight decay: $1\times10^{-2}$
            \item \textbf{Batch size:} 512 (labeled: 256, unlabeled: 256)
            \item \textbf{Scheduler:} Cosine decay without warmup
            \item \textbf{Augmentation:} Random crops (x0.5-1.0), horizontal flips
        \end{itemize} \\
        \midrule
        \rowcolor{gray!20} \multicolumn{2}{c}{\textit{Few-shot Semi-supervised Settings on 10 Fine-Grained Datasets}} \\
        \midrule
        \textbf{Model Configuration} & \textbf{Student Training Details} \\
        \midrule
        \begin{itemize}[leftmargin=*,nosep]
            \item \textbf{Student:} ResNet18~\citep{he2016deep} or MobileNet~\citep{sandler2018mobilenetv2} pre-trained on ImageNet under supervision
            \item \textbf{Input size:} 224$\times$224
            \item \textbf{Zero-shot Teacher:} ResNet50 from CLIP~\citep{radford2021learning}
            \item \textbf{Few-shot Teacher:} ResNet50 from Tip-Adapter-F~\citep{zhang2021tip}
            \item \textbf{Teacher input size:} 224$\times$224
            \item \textbf{labeled data:} $K \in \{1,2,4,8,16\}$ shots
            \item \textbf{$\zeta$, $\eta$, and $\lambda$}: $0.01, 2, 0.5$
            \item \textbf{$\alpha$ and $\beta$:} determined by validation
        \end{itemize}
        &
        \begin{itemize}[leftmargin=*,nosep]
            \item \textbf{Epochs:} 200
            \item \textbf{Optimizer:} AdamW ($\beta_1$=0.9, $\beta_2$=0.999)
            \item \textbf{Learning rate:} $1\times10^{-3}$, weight decay: $1\times10^{-2}$
            \item \textbf{Batch size:} 128 (labeled: 64, unlabeled: 64)
            \item \textbf{Scheduler:} Cosine decay without warmup
            \item \textbf{Augmentation:} Random crops (x0.5-1.0), horizontal flips
        \end{itemize} \\
        \midrule
        \rowcolor{gray!20} \multicolumn{2}{c}{\textit{Low-shot Semi-supervised Settings on ImageNet}} \\
        \midrule
        \textbf{Model Configuration} & \textbf{Student Training Details} \\
        \midrule
        \begin{itemize}[leftmargin=*,nosep]
            \item \textbf{Student:} CLIP ViT-B/16 or ViT-L/14~\citep{radford2021learning}
            \item \textbf{Input size:} 224$\times$224 (ViT-B/16) or 336$\times$336 (ViT-L/14)
            \item \textbf{Zero-shot Teacher:} CLIP ViT-L/14 or ViT-H/14~\citep{fang2023data}
            \item \textbf{Teacher input size:} 336$\times$336 (ViT-L/14) or 378$\times$378 (ViT-H/14)
            \item \textbf{Few-shot Teacher:} N/A
            \item \textbf{labeled data:} 1\% ($\frac{N}{N+M}\approx0.01$) or 10\% ($\frac{N}{N+M}\approx0.1$) of training data
            \item \textbf{$\zeta$, $\eta$, and $\lambda$}: $0.01, 2, 0.5$
            \item \textbf{$\alpha$ and $\beta$:} $\alpha=0.5$, $\beta=0.5$
        \end{itemize}
        &
        \begin{itemize}[leftmargin=*,nosep]
            \item \textbf{Epochs:} 32
            \item \textbf{Optimizer:} AdamW ($\beta_1$=0.9, $\beta_2$=0.999)
            \item \textbf{Learning rate:} $5\times10^{-5}$, weight decay: $5\times10^{-2}$
            \item \textbf{Batch size:} 512 (labeled: 256, unlabeled: 256)
            \item \textbf{Scheduler:} Cosine warmup decay (5000 steps)
            \item \textbf{Augmentation:} Random crops (x0.5-1.0), horizontal flips
        \end{itemize} \\
        \bottomrule
    \end{tabular}
    }
\end{table}

\myparagraph{OOD settings.} 
\Cref{tab:ood_implementation_details} provides the implementation details for our out-of-distribution (OOD) generalization experiments on 1) full training model evaluation and 2) adaptation methods including linear evaluation, visual prompt tuning, and VLM-based methods.

\begin{table}[htb]
    \caption{Implementation details for our out-of-distribution generalization experiments.}
    \label{tab:ood_implementation_details}
    \centering
    \small
    \resizebox{\textwidth}{!}{%
    \begin{tabular}{p{0.48\textwidth}|p{0.48\textwidth}}
        \toprule
        \rowcolor{gray!20} \multicolumn{2}{c}{\textit{Full Training}} \\
        \midrule
        \textbf{Model Configuration} & \textbf{Training Details} \\
        \midrule
        \begin{itemize}[leftmargin=*,nosep]
            \item \textbf{Student:} CLIP ViT-B/16~\citep{radford2021learning}
            \item \textbf{Student input size:} 224$\times$224
            \item \textbf{Zero-shot Teacher:} CLIP ViT-L/14~\citep{radford2021learning}
            \item \textbf{Teacher input size:} 336$\times$336
            \item \textbf{Labeled data:} 1\% and 10\% ImageNet
            \item \textbf{$\zeta$, $\eta$, and $\lambda$}: $0.01, 2, 0.5$
            \item \textbf{$\alpha$ and $\beta$:} $0.5$ and $1$
        \end{itemize}
        &
        \begin{itemize}[leftmargin=*,nosep]
            \item \textbf{Epochs:} 32
            \item \textbf{Optimizer:} AdamW ($\beta_1$=0.9, $\beta_2$=0.999)
            \item \textbf{Learning rate:} $5\times10^{-5}$, weight decay: $5\times10^{-2}$
            \item \textbf{Batch size:} 512 (labeled: 256, unlabeled: 256)
            \item \textbf{Scheduler:} Cosine warmup decay (5000 steps)
            \item \textbf{Augmentation:} Random crops (x0.5-1.0), horizontal flips
        \end{itemize} \\
        \midrule
        \rowcolor{gray!20} \multicolumn{2}{c}{\textit{Adaptation Methods (Linear Evaluation \& Visual Prompt Tuning)}} \\
        \midrule
        \textbf{Method Configuration} & \textbf{Training Details} \\
        \midrule
        \begin{itemize}[leftmargin=*,nosep]
            \item \textbf{Linear evaluation}~\citep{caron2021emerging}
            \item \textbf{Visual prompt tuning}~\citep{jia2022visual}
            \item \textbf{Frozen backbone:} CLIP ViT-B/16~\citep{radford2021learning}
            \item \textbf{Input size:} 224$\times$224
            \item \textbf{Zero-shot Teacher:} CLIP ViT-L/14~\citep{radford2021learning}
            \item \textbf{Teacher input size:} 336$\times$336
            \item \textbf{Labeled data:} 1\% and 10\% ImageNet
            \item \textbf{$\zeta$, $\eta$, and $\lambda$}: $0.01, 2, 0.5$
            \item \textbf{$\alpha$ and $\beta$:} $0.5$ and $1$
        \end{itemize}
        &
        \begin{itemize}[leftmargin=*,nosep]
            \item \textbf{Epochs:} 20
            \item \textbf{Optimizer:} AdamW ($\beta_1$=0.9, $\beta_2$=0.999)
            \item \textbf{Learning rate:} $5\times10^{-5}$, weight decay: $5\times10^{-2}$
            \item \textbf{Batch size:} 512 (labeled: 256, unlabeled: 256)
            \item \textbf{Scheduler:} Cosine warmup decay (5000 steps)
            \item \textbf{Augmentation:} Random crops (x0.5-1.0), horizontal flips
        \end{itemize} \\
        \midrule
        \rowcolor{gray!20} \multicolumn{2}{c}{\textit{Adaptation Methods (Prompt Tuning)}} \\
        \midrule
        \textbf{Method Configuration} & \textbf{Training Details} \\
        \midrule
        \begin{itemize}[leftmargin=*,nosep]
            \item \textbf{Prompt tuning:} CoOp~\citep{zhou2022learning}, PromptSRC~\citep{khattak2023self}
            \item \textbf{Frozen backbone:} CLIP ViT-B/16~\citep{radford2021learning}
            \item \textbf{Input size:} 224$\times$224
            \item \textbf{Zero-shot Teacher:} CLIP ViT-L/14~\citep{radford2021learning}
            \item \textbf{Teacher input size:} 336$\times$336
            \item \textbf{Labeled data:} 1\% and 10\% ImageNet
            \item \textbf{$\zeta$, $\eta$, and $\lambda$}: $0.01, 2, 0.5$
            \item \textbf{$\alpha$ and $\beta$:} $0.5$ and $1$
        \end{itemize}
        &
        \begin{itemize}[leftmargin=*,nosep]
            \item \textbf{Prompt tuning:} Following PromptSRC~\citep{khattak2023self} configurations
            \item \textbf{Comparison:} CasPL~\citep{wu2025cascade} with domain-specific unlabeled data
        \end{itemize} \\
        \bottomrule
    \end{tabular}%
    }
\end{table}


\clearpage

\subsection{Computational Costs}
\label{sec:appendix_computational_overhead}

\paragraph{Inference overhead of \textbf{\texttt{DHO}} over SHO.}
\Cref{tab:full_computation_cost} presents computational overheads at inference time introduced by \textbf{\texttt{DHO}} over SHO for all the architectures in this paper, such as MobileNetV2~\citep{sandler2018mobilenetv2}, ResNet-18~\citep{he2016deep}, ResNet-50~\citep{he2016deep}, ViT-B/16~\citep{dosovitskiy2020image}, and ViT-L/16~\citep{dosovitskiy2020image}.

\begin{table}[h]
    \small
    \caption{Inference overhead using RTX 4090 across different architecture.}
    \label{tab:full_computation_cost}
    \centering
    \begin{tabular}{lrrr}
        \toprule
        \textbf{Model} & \textbf{Params (M)} & \textbf{FLOPs (G)} & \textbf{Throughput (im/s)} \\
        \midrule
        MobileNetV2 & 3.50 & 0.33 & 2978.4 \\
        \rowcolor{gray!20} + \textbf{\texttt{DHO}} & 4.79 {\color{red}(+36.5\%)} & 0.34 {\color{red}(+3.0\%)} & 2971.2 {\color{red}(-0.24\%)} \\
        \midrule
        ResNet-18 & 11.69 & 1.83 & 3525.7 \\
        \rowcolor{gray!20} + \textbf{\texttt{DHO}} & 12.20 {\color{red}(+4.4\%)} & 1.83 {\color{red}(+0.0\%)} & 3518.6 {\color{red}(-0.20\%)} \\
        \midrule
        ResNet-50 & 25.56 & 4.14 & 1018.4 \\
        \rowcolor{gray!20} + \textbf{\texttt{DHO}} & 27.61 {\color{red}(+8.0\%)} & 4.15 {\color{red}(+0.2\%)} & 1016.4 {\color{red}(-0.19\%)} \\
        \midrule
        ViT-B/16 & 86.57 & 16.87 & 290.2 \\
        \rowcolor{gray!20} + \textbf{\texttt{DHO}} & 87.34 {\color{red}(+0.9\%)} & 16.87 {\color{red}(+0.0\%)} & 290.1 {\color{red}(-0.02\%)} \\
        \midrule
        ViT-L/16 & 304.33 & 59.70 & 255.1 \\
        \rowcolor{gray!20} + \textbf{\texttt{DHO}} & 305.35 {\color{red}(+0.3\%)} & 59.70 {\color{red}(+0.0\%)} & 255.6 {\color{red}(+0.18\%)} \\
        \bottomrule
    \end{tabular}
\end{table}

\begin{wraptable}{r}{0.5\textwidth}
    \vspace{-0.15in}
    \caption{
        Training time and hardware.
    }
    \label{tab:hardware_training_time}
    \vspace{-0.1in}
    \centering
    \resizebox{0.5\textwidth}{!}{
        \begin{tabular}{l|l|c|c}
            \toprule
            \textbf{Student} & \textbf{Teacher} & \textbf{Training Time} & \textbf{Hardware} \\
            \midrule
            ResNet-18 & ResNet-50 & $\approx$ 6 hours & 4× RTX 4090 \\
            ResNet-50 & ResNet-50 & $\approx$ 8 hours & 4× RTX 4090 \\
            ViT-B/16 & ViT-L/14 & $\approx$ 28 hours & 8× RTX 4090 \\
            ViT-B/16 & ViT-H/14 & $\approx$ 40 hours & 8× RTX 4090 \\
            ViT-L/14 & ViT-H/14 & $\approx$ 80 hours & 8× RTX 4090 \\
            \bottomrule
        \end{tabular}
    }
    \vspace{-0.15in}
\end{wraptable}

\paragraph{Training time and hardware requirements.}
\Cref{tab:hardware_training_time} presents the training time required for our experiments.
For VLM distillation experiments, which represent the most resource-intensive component of our work, we used 8× NVIDIA RTX 4090 GPUs.
The ViT-H/14 to ViT-L/14 distillation required approximately 80 hours, while the ViT-H/14 to ViT-B/16 and ViT-L/14 to ViT-B/16 distillations required approximately 40 and 28 hours, respectively.
For the ViT-H/14 to ViT-L/14 distillation, we implemented gradient accumulation with 4 steps and mixed precision training~\citep{micikevicius2017mixed} to optimize computational efficiency.
For ImageNet experiments, we used 4× NVIDIA RTX 4090 GPUs, with ResNet-18 and ResNet-50 models requiring approximately 6 and 8 hours of training time, respectively.
We provide these details to facilitate reproduction of our results and to give researchers a clear understanding of the computational resources needed to implement our approach at scale.

\paragraph{Inference overhead improvements with ToMe.}
To further improve the computational efficiency of our approach, we explored integrating Token Merging (ToMe)~\citep{bolya2022token} with \textbf{\texttt{DHO}}. ToMe is a technique that reduces the number of tokens in ViTs by merging similar tokens to improve the efficiency of ViTs. \Cref{tab:tome_efficiency} shows that combining \textbf{\texttt{DHO}} with ToMe significantly reduces computational costs with minimal impact on performance.

\begin{table}[h]
    \small
    \caption{Performance and inference overhead of \textbf{\texttt{DHO}} with Token Merging (ToMe) on ImageNet under low-shot semi-supervised settings using RTX 4090.}
    \label{tab:tome_efficiency}
    \centering
    \begin{tabular}{lccccc}
        \toprule
        \textbf{Method} & \textbf{Labeled} & \textbf{Accuracy (\%)} & \textbf{Params (M)} & \textbf{FLOPs (G)} & \textbf{Throughput (im/s)} \\
        \midrule
        \textbf{\texttt{DHO}} & 1\% & 81.6 & 87.22 & 17.58 & 243.35 \\
        \rowcolor{gray!20} \textbf{\texttt{DHO}} + ToMe & 1\% & 81.4 {\color{red}(--0.2)} & 87.22 & 13.12 {\color{blue}(--25.4\%)} & 323.39 {\color{blue}(+32.9\%)} \\
        \midrule
        \textbf{\texttt{DHO}} & 10\% & 82.8 & 87.22 & 17.58 & 238.11 \\
        \rowcolor{gray!20} \textbf{\texttt{DHO}} + ToMe & 10\% & 82.5 {\color{red}(--0.3)} & 87.22 & 13.12 {\color{blue}(--25.4\%)} & 308.49 {\color{blue}(+29.6\%)} \\
        \bottomrule
    \end{tabular}
\end{table}


\section{Datasets}
\label{sec:appendix_datasets}

\begin{table}[H]
    \caption{
        Overview of datasets used in our experiments, organized into three categories:
        \textbf{(top)} standard classification datasets,
        \textbf{(middle)} ImageNet, and
        \textbf{(bottom)} ImageNet variants for out-of-distribution (OOD) evaluation.
        For few-shot semi-supervised learning experiments, we report both the absolute number of labeled samples and their percentage relative to the full training set.
    }
    \label{tab:dataset_overview}
    \centering
    \small
    \resizebox{\textwidth}{!}{
        \begin{tabular}{l|rrrr|rr}
            \toprule
            \textbf{Dataset} & \textbf{\# Classes} & \textbf{\# Train} & \textbf{\# Val} & \textbf{\# Test} & \multicolumn{1}{c}{\textbf{\# Labeled (1-shot)}} & \multicolumn{1}{c}{\textbf{\# Labeled (16-shot)}} \\
            \midrule
            \rowcolor{gray!20}\multicolumn{7}{l}{\textit{Fine-grained 10 Datasets}} \\
            Caltech101~\citep{fei2004learning} & 100 & 4,128 & 1,649 & 2,465 & 100 (2.42\%) & 1,600 (38.76\%) \\
            OxfordPets~\citep{parkhi2012cats} & 37 & 2,944 & 736 & 3,669 & 37 (1.26\%) & 592 (20.11\%) \\
            StanfordCars~\citep{krause20133d} & 196 & 6,509 & 1,635 & 8,041 & 196 (3.01\%) & 3,136 (48.18\%) \\
            Flowers102~\citep{nilsback2008automated} & 102 & 4,093 & 1,633 & 2,463 & 102 (2.49\%) & 1,632 (39.87\%) \\
            Food101~\citep{bossard2014food} & 101 & 50,500 & 20,200 & 30,300 & 101 (0.20\%) & 1,616 (3.20\%) \\
            FGVCAircraft~\citep{maji2013fine} & 100 & 3,334 & 3,333 & 3,333 & 100 (3.00\%) & 1,600 (48.00\%) \\
            SUN397~\citep{xiao2010sun} & 397 & 15,880 & 3,970 & 19,850 & 397 (2.50\%) & 6,352 (40.00\%) \\
            DTD~\citep{cimpoi2014describing} & 47 & 2,820 & 1,128 & 1,692 & 47 (1.67\%) & 752 (26.67\%) \\
            EuroSAT~\citep{helber2019eurosat} & 10 & 13,500 & 5,400 & 8,100 & 10 (0.07\%) & 160 (1.19\%) \\
            UCF101~\citep{soomro2012ucf101} & 101 & 7,639 & 1,898 & 3,783 & 101 (1.32\%) & 1,616 (21.15\%) \\
            \midrule
            \rowcolor{gray!20}\multicolumn{7}{l}{\textit{Coarse-grained Dataset}} \\
            ImageNet~\citep{russakovsky2015imagenet} & 1,000 & 1.28M & - & 50,000 & 1,000 (0.08\%) & 16,000 (1.25\%) \\
            \midrule
            \rowcolor{gray!20}\multicolumn{7}{l}{\textit{ImageNet OOD Variants}} \\
            ImageNet-V2~\citep{recht2019imagenet} & 1,000 & - & - & 10,000 & - & - \\
            ImageNet-Sketch~\citep{wang2019learning} & 1,000 & - & - & 50,889 & - & - \\
            ImageNet-A~\citep{hendrycks2021natural} & 200 & - & - & 7,500 & - & - \\
            ImageNet-R~\citep{hendrycks2021many} & 200 & - & - & 30,000 & - & - \\
            \bottomrule
        \end{tabular}
    }
\end{table}

We evaluated our approach on 11 diverse datasets, with ImageNet~\citep{russakovsky2015imagenet} serving as our primary benchmark.
The datasets span general object recognition~\citep{russakovsky2015imagenet, fei2004learning}, fine-grained classification tasks (vehicles~\citep{krause20133d,maji2013fine}, natural entities~\citep{nilsback2008automated,parkhi2012cats,bossard2014food}), and specialized domains (scenes~\citep{xiao2010sun}, textures~\citep{cimpoi2014describing}, remote sensing~\citep{helber2019eurosat}, and human actions~\citep{soomro2012ucf101}).
Additionally, we conduct experiments on four out-of-distribution test sets to further validate the model's generalization capabilities.
To assess our model's robustness to distribution shifts, we evaluate it on several challenging variants of ImageNet: ImageNet-v2~\citep{recht2019imagenet}, ImageNet-Sketch~\citep{wang2019learning}, ImageNet-A~\citep{hendrycks2021natural}, and ImageNet-R~\citep{hendrycks2021many}.

We summarize the overview of datasets used in \Cref{tab:dataset_overview}, these datasets exhibit diversity in their characteristics, with varying numbers of classes and samples per dataset.
This diversity enabled us to thoroughly validate our method across different few-shot semi-supervised learning scenarios by systematically varying the ratios between labeled and unlabeled samples.


\section{Additional Experiments}
\label{sec:additional_experiments}

\subsection{Experiments on Motivation}
\label{sec:appendix_motivation}

\begin{wraptable}{r}{0.48\textwidth}
\vspace{-0.15in}
\caption{Results on ImageNet using a ResNet-18 student and a ResNet-50 teacher. 100\% denotes using the fully labeled dataset with a fully supervised teacher.}
\centering
\label{tab:appendix_motivation}
\vspace{-0.1in}
\resizebox{\linewidth}{!}{%
\begin{tabular}{lcccccc}
\toprule
Method & 1-shot & 2-shot & 4-shot & 8-shot & 16-shot & 100\% \\
\midrule
CE+KD (logit) & 50.5 & 50.6 & 50.6 & 51.0 & 51.2 & 66.3 \\
\textbf{\texttt{DHO}}  & \textbf{51.8} & \textbf{52.4} & \textbf{52.6} & \textbf{53.3} & \textbf{54.5} & \textbf{66.7} \\
\bottomrule
\end{tabular}
}
\vspace{-0.1in}
\end{wraptable}

We observe that existing KD methods, such as logit distillation~\citep{hinton2015distilling}, face gradient conflict in \textbf{few-/low-shot semi-supervised settings}.
To investigate whether gradient conflicts are severe or only present in few-/low-shot semi-supervised settings, we conduct additional experiments by using full labels on ImageNet and 8 datasets using ResNet-18/ResNet-50 as student/teacher. First of all, we observe \textbf{no gradient conflicts at all in any fully supervised settings}, \ie, $\mathtt{CosSim}(\nabla_\theta\mathcal{L}_\text{CE},\nabla_\theta\mathcal{L}_\text{KD})>0$ throughout training. 
This is because the distillation signal from a fully supervised teacher aligns well with the labeled data when the teacher is trained on the same dataset as the student.
However, in few-/low-shot semi-supervised settings, this alignment often breaks down, as shown in \Cref{fig:average_gradient_conflict}. \Cref{tab:appendix_motivation} also reports results on ImageNet using a ResNet-18 student and a ResNet-50 teacher. With fully labeled data, the gap is significantly reduced (0.4), while $\textbf{\texttt{DHO}}$ outperforms CE+KD (logit) in few-shot settings, validating that $\textbf{\texttt{DHO}}$ is specifically tailored to few-/low-shot semi-supervised settings.

\subsection{Experiments on MobileNet}
\label{sec:appendix_mobilenet}

\begin{figure}[htb]
    \begin{minipage}[b]{1.0\textwidth}
        \centering
        \includegraphics[width=\textwidth, scale=0.8]{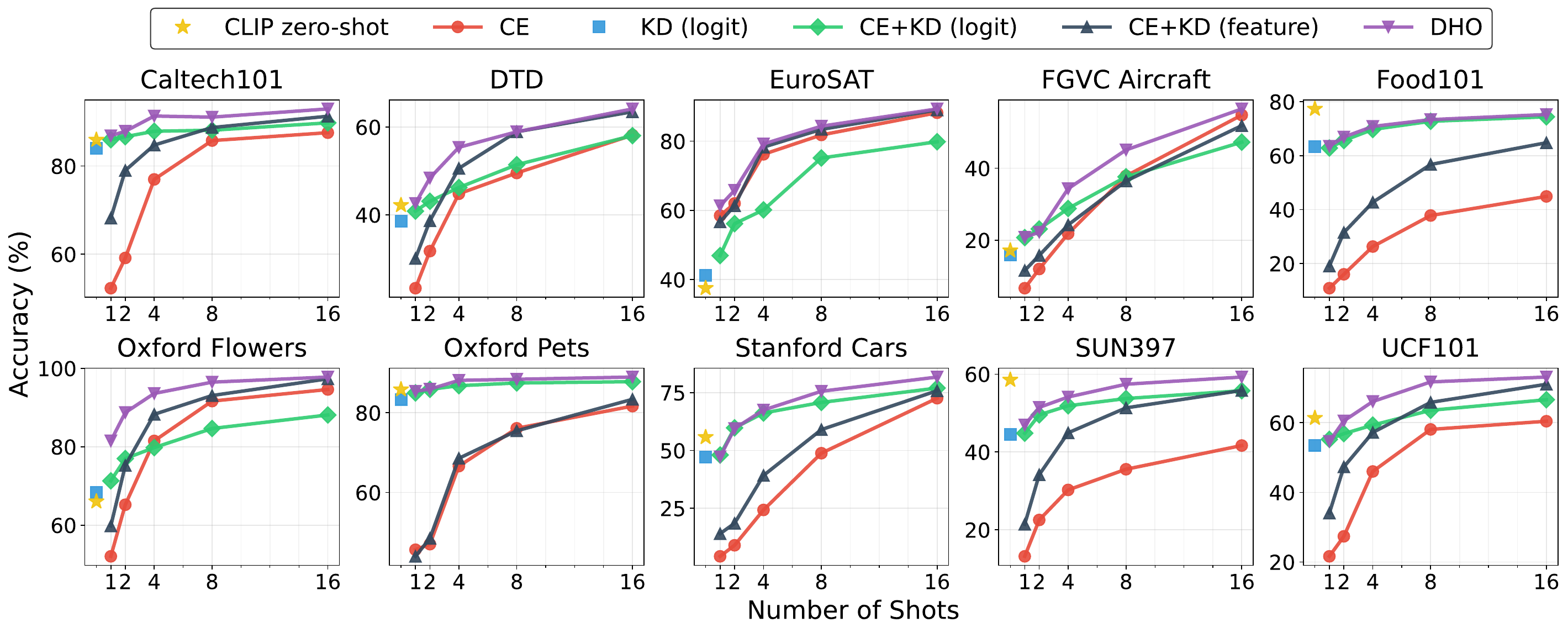}
        \caption{
            Results on \textbf{10 datasets} under few-shot semi-supervision using \textbf{MobileNetV2} with \textbf{zero-shot teacher}~\citep{radford2021learning}.
        }
        \label{fig:mobilenet_baseline}
    \end{minipage}

    \vspace{0.5cm}

    \begin{minipage}[b]{1.0\textwidth}
        \centering
        \includegraphics[width=\textwidth, scale=0.8]{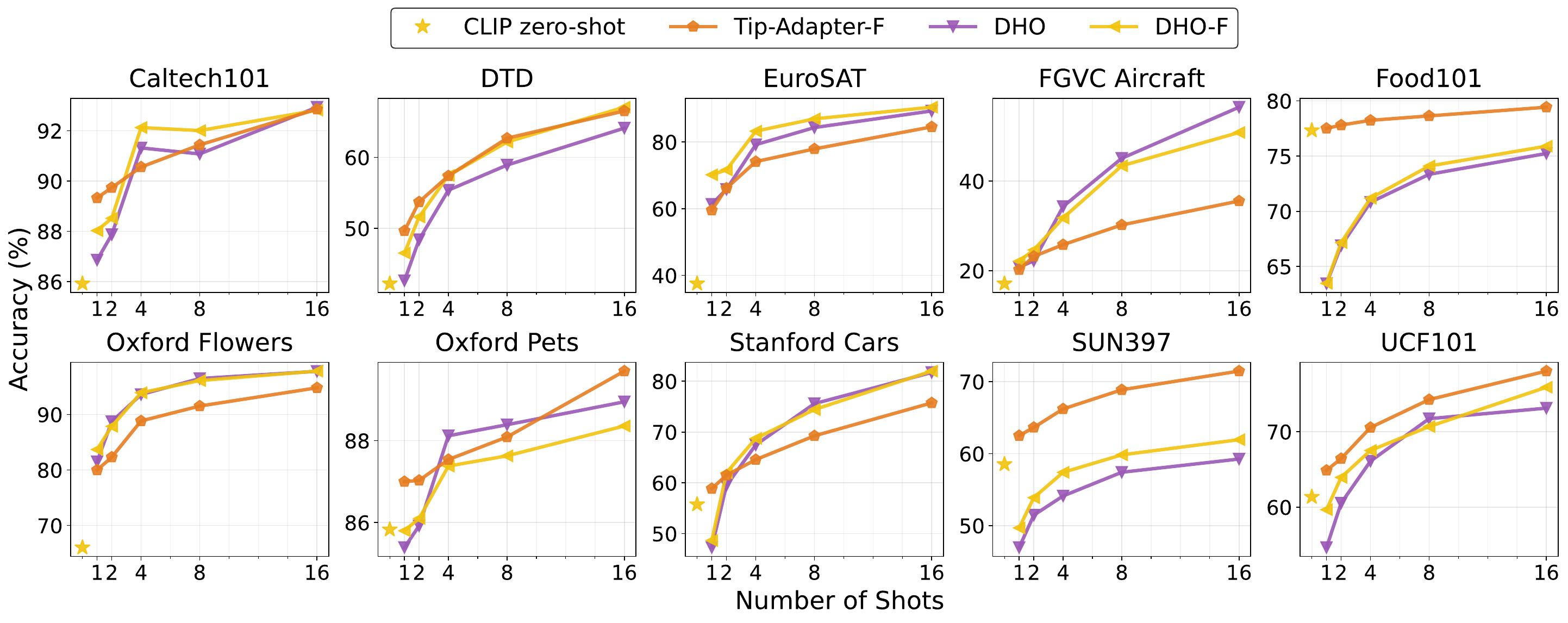}
        \caption{
            Results on \textbf{10 datasets} using \textbf{MobileNetV2} with either zero- or \textbf{few-shot teacher}~\citep{zhang2021tip}.
        }
        \label{fig:mobilenet_fewshot}
    \end{minipage}
\end{figure}

To demonstrate the versatility of our \textbf{\texttt{DHO}} approach beyond ResNet~\citep{he2016deep} and ViT~\citep{dosovitskiy2020image} architectures, we extended our experiments to the MobileNetV2~\citep{sandler2018mobilenetv2}, which is specifically designed for real-world applications with compact models.
We maintained identical experimental settings as described in \Cref{sec:appendix_implementation_details}, using MobileNetV2 as the student model while distilling from CLIP ResNet50.

As illustrated in \Cref{fig:mobilenet_baseline}, our \textbf{\texttt{DHO}} consistently outperforms all single-head baseline methods, demonstrating its effectiveness on lightweight model architectures along with ResNet18.
Furthermore, \Cref{fig:mobilenet_fewshot} reveals patterns similar to our ResNet18 experiments regarding few-shot integration.
Our method successfully incorporates few-shot teacher knowledge, although we observe that the few-shot teacher does not consistently yield improvements over the zero-shot teacher.
Notably, our distilled MobileNetV2 model sometimes achieves superior performance to the zero- and few-shot teachers (ResNet-50) despite having significantly fewer parameters.
This pattern of outperforming both zero-shot and few-shot teachers mirrors the observations from our main experiments, further validating the effectiveness of our approach across different architectural families.


\subsection{Additional Results with Gradient Surgery Methods}
\label{sec:appendix_pcgrad}

In this section, we present experimental results comparing $\textbf{\texttt{DHO}}$ with PCGrad~\citep{yu2020gradient}, a well-known gradient surgery method from the multi-task learning literature.
PCGrad addresses gradient conflicts by projecting conflicting gradients to resolve conflicts post hoc.

We compare $\textbf{\texttt{DHO}}$ with PCGrad using 8 of the 10 datasets (without Food101, Sun397) with ResNet-18 as the student model.
The results are shown in \Cref{tab:pcgrad_comparison}.

\begin{table}[h]
\small
\centering
\caption{Performance comparison of PCGrad with $\textbf{\texttt{DHO}}$ on 8 datasets (average) using ResNet-18.}
\label{tab:pcgrad_comparison}
\begin{tabular}{l|ccccc}
\toprule
Method & 1-shot & 2-shot & 4-shot & 8-shot & 16-shot \\
\midrule
CE+KD (logit) & 56.2 & 58.9 & 61.7 & 65.8 & 69.8 \\
PCGrad & 56.8 & 60.3 & 62.1 & 67.5 & 71.6 \\
\rowcolor{gray!20} $\textbf{\texttt{DHO}}$ & \textbf{60.1} & \textbf{63.7} & \textbf{70.2} & \textbf{75.4} & \textbf{79.1} \\
\bottomrule
\end{tabular}
\end{table}

We observe that PCGrad improves over CE+KD (logit), but still underperforms compared to $\textbf{\texttt{DHO}}$.
Moreover, PCGrad incurs additional memory and computational costs ($\mathcal{O}(|\theta|)$) due to gradient projection and storage, whereas $\textbf{\texttt{DHO}}$ remains lightweight ($\mathcal{O}(d \times C)$) and simple to implement.

The superior performance of $\textbf{\texttt{DHO}}$ can be attributed to its approach of avoiding gradient conflicts at the source by isolating the learning dynamics of each objective via dual heads, rather than resolving conflicts after they arise.
Additionally, $\textbf{\texttt{DHO}}$ provides dynamic interpolation capability at inference time between supervised and distillation signals, which standard gradient-based methods do not offer.


\subsection{Additional Results with KD Methods}
\label{sec:appendix_dkd_wkd}

In this section, we present experimental results combining our $\textbf{\texttt{DHO}}$ method with additional distillation approaches: Decoupled Knowledge Distillation (DKD)~\citep{zhao2022decoupled} and Wasserstein Knowledge Distillation (WKD)~\citep{lv2024wasserstein}.
Both DKD and WKD are orthogonal to $\textbf{\texttt{DHO}}$ since their losses can be applied directly to the outputs of $h_\text{CE}$.
DKD decouples the target class from non-target classes, and WKD computes a kernel matrix within each class.
Both approaches require ground-truth labels, restricting their use to the small labeled dataset $\mathcal{D}^{(l)}$.

Despite these limitations, we conducted experiments on ImageNet (\Cref{tab:dkd_wkd_imagenet}) and on 8 of the 10 datasets (without Food101, Sun397) as shown in \Cref{tab:dkd_wkd_8datasets}.
The results demonstrate that $\textbf{\texttt{DHO}}$ substantially improves the performance of these state-of-the-art KD methods.
Notably, WKD+$\textbf{\texttt{DHO}}$ consistently outperforms $\textbf{\texttt{DHO}}$ alone in most settings, demonstrating the extensibility of $\textbf{\texttt{DHO}}$ due to its simplicity.

\begin{table}[h]
\small
\centering
\caption{Performance comparison of DKD and WKD with and without $\textbf{\texttt{DHO}}$ on ImageNet. Numbers in parentheses show improvement over the base method.}
\label{tab:dkd_wkd_imagenet}
\begin{tabular}{l|ccccc}
\toprule
Method & 1-shot & 2-shot & 4-shot & 8-shot & 16-shot \\
\midrule
\rowcolor{gray!20} $\textbf{\texttt{DHO}}$ & 51.8 & 52.4 & 52.6 & 53.3 & 54.5 \\
DKD & 8.9 & 15.0 & 20.6 & 28.2 & 34.9 \\
\rowcolor{gray!20} DKD+$\textbf{\texttt{DHO}}$ & 47.1 (+38.2) & 44.3 (+29.3) & 40.9 (+20.3) & 40.2 (+12.0) & 42.4 (+7.5) \\
WKD & 11.0 & 17.0 & 17.0 & 27.7 & 34.6 \\
\rowcolor{gray!20} WKD+$\textbf{\texttt{DHO}}$ & \textbf{53.2} (+42.2) & \textbf{53.3} (+36.3) & \textbf{53.3} (+36.3) & \textbf{54.0} (+26.3) & \textbf{54.8} (+20.2) \\
\bottomrule
\end{tabular}
\end{table}

\begin{table}[h]
\small
\centering
\caption{Performance comparison of DKD and WKD with and without $\textbf{\texttt{DHO}}$ on 8 datasets (average). Numbers in parentheses show improvement over the base method.}
\label{tab:dkd_wkd_8datasets}
\begin{tabular}{l|ccccc}
\toprule
Method & 1-shot & 2-shot & 4-shot & 8-shot & 16-shot \\
\midrule
\rowcolor{gray!20} $\textbf{\texttt{DHO}}$ & \textbf{60.1} & 63.7 & 70.2 & 75.4 & 79.1 \\
DKD & 28.7 & 41.3 & 55.4 & 66.3 & 73.1 \\
\rowcolor{gray!20} DKD+$\textbf{\texttt{DHO}}$ & 46.2 (+17.5) & 51.2 (+9.9) & 61.8 (+6.4) & 69.8 (+3.5) & 74.9 (+1.8) \\
WKD & 30.0 & 38.2 & 53.3 & 66.1 & 73.6 \\
\rowcolor{gray!20} WKD+$\textbf{\texttt{DHO}}$ & 59.6 (+29.6) & \textbf{64.5} (+26.3) & \textbf{71.2} (+17.9) & \textbf{75.7} (+9.6) & \textbf{79.6} (+6.0) \\
\bottomrule
\end{tabular}
\end{table}


\subsection{Additional Results with Adaptive Weighting}
\label{sec:appendix_adaptive_weighting}

To further explore the potential of $\textbf{\texttt{DHO}}$, we implement an entropy-based adaptive weighting mechanism.
Let the entropy of a probability vector $p \in \Delta^{C-1}$ be $H(p) = -\sum_{c=1}^{C} p_c \log p_c$.
We compute the adaptive weight $\alpha$ as:
\begin{equation}
\alpha = \frac{\exp(-H(\hat{p}_\text{CE}))}{\exp(-H(\hat{p}_\text{CE})) + \exp(-H(\hat{p}_\text{KD}))}
\end{equation}
where $\hat{p}_\text{CE}$ and $\hat{p}_\text{KD}$ are the output probability vectors from $h_\text{CE}$ and $h_\text{KD}$, respectively.
The final prediction is then computed as $\hat{p} = \alpha \cdot \hat{p}_\text{CE} + (1-\alpha) \cdot \hat{p}_\text{KD}$.

The intuition behind this approach is that lower entropy (higher confidence) predictions should receive higher weights in the final ensemble.
When one head produces more confident predictions than the other, the adaptive weighting mechanism automatically emphasizes the more certain prediction.
The results show that the entropy-based adaptive weighting method is not proved to be effective in these experiments.
While the adaptive weighting does not consistently outperform the fixed interpolation approach, this is likely due to modern neural networks producing overconfident predictions, making entropy an unreliable proxy for uncertainty without proper calibration.
However, we believe adaptive weighting could be beneficial with well-calibrated models or alternative uncertainty measures.

\subsection{Sensitivity of Single-head KD to Hyperparameters}
\label{sec:appendix_sho_hyperparameter}
\begin{wrapfigure}{r}{0.4\textwidth}
    \vspace{-0.2in}
    \centering
    \includegraphics[width=\linewidth]{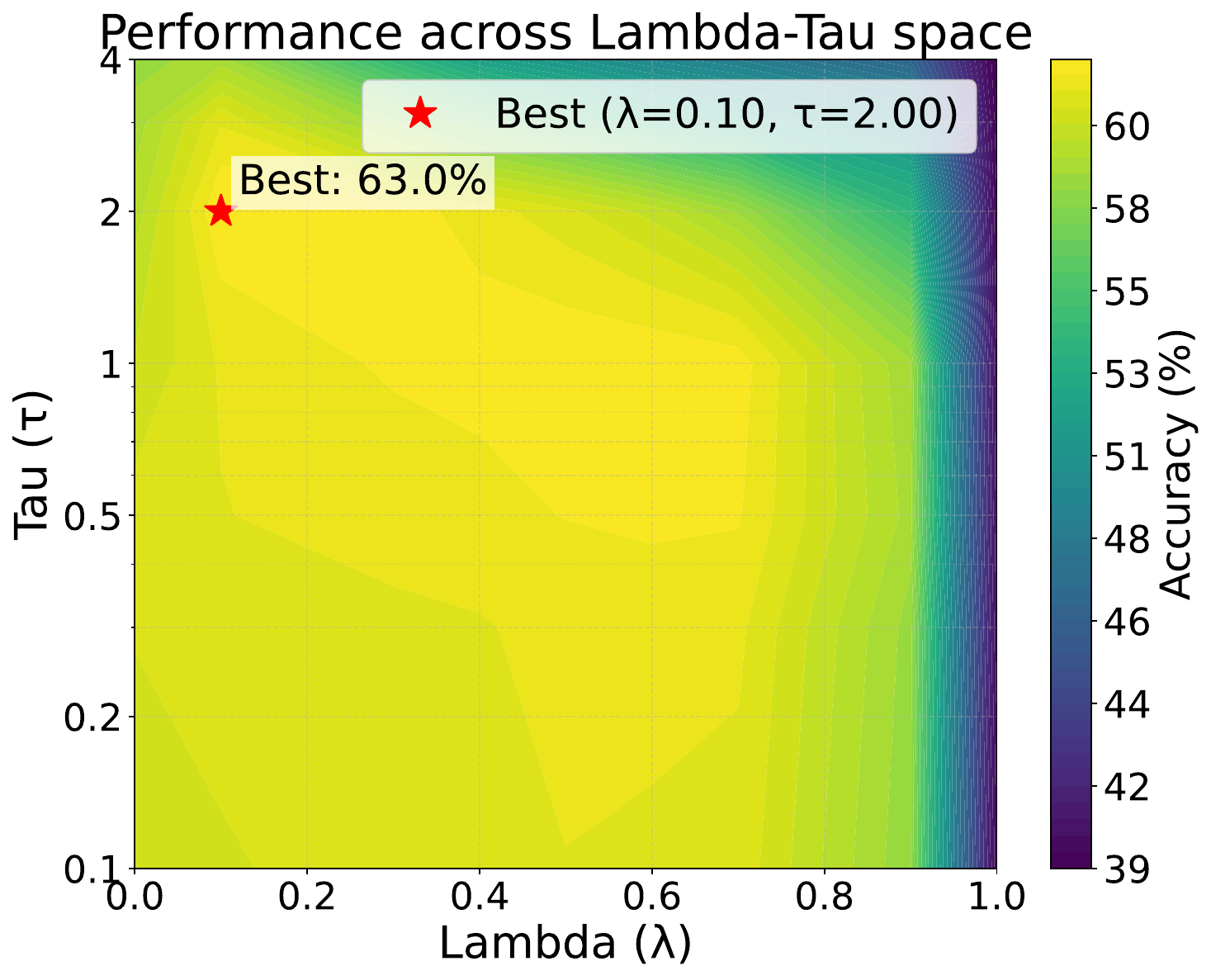}
    \vspace{-0.25in}
    \caption{
        \small Grid search results for $\lambda$ and $\tau$ for \textbf{CE + KD (logit)}.
    }
    \label{fig:ablation_lambda_tau}
    \vspace{-0.2in}
\end{wrapfigure}
To investigate the sensitivity of single-head KD baseline, \ie, CE + KD (logit), we conduct additional experiments by varying the loss balancing hyperparameter $\lambda \in \{0.0, 0.1, \ldots, 0.9, 1.0\}$ and the temperature hyperparameter $\tau \in \{0.1, 0.3, 0.5, 1, 2, 4\}$, resulting in total 55 training runs. \Cref{fig:ablation_lambda_tau} shows the interpolated results from these 55 experiments on ImageNet using a ResNet-50 student and a CLIP ResNet-50 teacher, following the same experimental setup in \Cref{fig:ablation_alpha_beta}. We observe that the single-head KD baseline is also sensitive to the hyperparameters $(\lambda, \tau)$, requiring intensive training to obtain the best combination. Crucially, the best performance of CE + KD (logit) is 63.0\%, which is less than the performance of $\textbf{\texttt{DHO}}$ (64.7\%) with heuristically selected $(\alpha, \beta)$.

\paragraph{Validation tuning of $(\alpha, \beta)$ is optional.}
We further verify that $\textbf{\texttt{DHO}}$ does not depend on a validation split for hyperparameter selection.
\Cref{tab:dho_fixed_vs_searched} reports $\textbf{\texttt{DHO}}$ with a single \emph{fixed} choice $(\alpha, \beta) = (0.5, 1)$---the value implied by \Cref{main_thm:inference_equiv} when $\lambda = 0.5$ and the temperatures are matched---alongside $\textbf{\texttt{DHO}}$ with $(\alpha, \beta)$ selected on a held-out validation split, and the SHO baseline with its training-time weight $\lambda = 0.5,\ \tau = 1$.
Numbers are 8-dataset averages with ResNet-18 student and CLIP ResNet-50 zero-shot teacher.

\begin{table}[h]
\caption{$\textbf{\texttt{DHO}}$ \textbf{does not require validation tuning} of $(\alpha, \beta)$: a single fixed $(0.5, 1)$ already outperforms SHO at every shot, and matches validation-searched $\textbf{\texttt{DHO}}$ from 4-shot onward. 8-dataset averages with ResNet-18 student and CLIP ResNet-50 zero-shot teacher.}
\label{tab:dho_fixed_vs_searched}
\centering
\small
\begin{tabular}{lccccc}
\toprule
\textbf{Method} & \textbf{1-shot} & \textbf{2-shot} & \textbf{4-shot} & \textbf{8-shot} & \textbf{16-shot} \\
\midrule
SHO ($\lambda{=}0.5,\ \tau{=}1$)               & 55.9 & 59.4 & 62.1 & 67.4 & 71.2 \\
$\textbf{\texttt{DHO}}$, fixed $(\alpha,\beta){=}(0.5,1)$       & \underline{57.5} \scriptsize{\textcolor{blue}{(+1.6)}} & \underline{63.5} \scriptsize{\textcolor{blue}{(+4.1)}} & \underline{71.6} \scriptsize{\textcolor{blue}{(+9.5)}} & \textbf{75.3} \scriptsize{\textcolor{blue}{(+7.9)}} & \textbf{79.2} \scriptsize{\textcolor{blue}{(+8.0)}} \\
$\textbf{\texttt{DHO}}$, validation-searched $(\alpha,\beta)$  & \textbf{59.1} & \textbf{64.3} & \textbf{71.7} & \underline{75.2} & \textbf{79.2} \\
\bottomrule
\end{tabular}
\end{table}

The gap between fixed and searched $\textbf{\texttt{DHO}}$ is $\le 0.1\%$ from 4-shot onward and only $1.6\%$ at 1-shot---i.e., validation tuning is most valuable exactly where the validation signal is the noisiest (one labeled example per class).
Compared with SHO's 55-run grid search over $(\lambda, \tau)$ in \Cref{fig:ablation_lambda_tau} (best 63.0\%), $\textbf{\texttt{DHO}}$ with a single fixed $(\alpha, \beta)$ at 16-shot achieves 79.2\% with no grid search and a single training run.

\paragraph{Grid search over $(\alpha, \beta)$.}
\Cref{fig:ablation_alpha_beta} (in the main paper) visualizes a $(\alpha, \beta)$ grid search on ImageNet with ResNet-50 under the 16-shot setting, complementing $\Cref{tab:dho_fixed_vs_searched}$.
The peak (65.5\%) is at $(\alpha,\beta){\approx}(0.5,0.3)$, close to the heuristic setting (64.7\%) and the fixed $(0.5,1)$ value (\Cref{tab:dho_fixed_vs_searched}).
These gains require no additional training, demonstrating the \textbf{efficiency of post-training hyperparameter search} for the dual-head interpolation in \Cref{sec:dual_head_interpolation}.

\subsection{Comparison under Few-shot Teacher}


\begin{wraptable}{r}{0.58\textwidth}
\centering
\caption{Results on 10 datasets (average) using ResNet-18 with either zero- or \textbf{few-shot teachers}.}
\label{tab:hHhD-1}
\vspace{-0.1in}
\resizebox{\linewidth}{!}{
\begin{tabular}{lcccccc}
\toprule
Method & 0-shot & 1-shot & 2-shot & 4-shot & 8-shot & 16-shot \\
\midrule
\rowcolor{gray!20}\multicolumn{7}{l}{\textit{Zero-shot teacher distillation}} \\
CE + KD (logit)        & -- & 55.6 & 58.2 & 61.1 & 64.8 & 68.4 \\
$\textbf{\texttt{DHO}}$ & -- & \textbf{58.9} & \textbf{62.2} & \textbf{68.4} & \textbf{73.1} & \textbf{76.5} \\
\midrule
\rowcolor{gray!20}\multicolumn{7}{l}{\textit{Few-shot teacher distillation}} \\
CE + KD (logit)-F           & -- & 59.8 & 63.9 & 68.9 & 72.6 & 76.2 \\
$\textbf{\texttt{DHO-F}}$   & -- & \textbf{61.1} & \textbf{64.7} & \textbf{69.5} & \textbf{73.6} & \textbf{76.9} \\
\midrule
\rowcolor{gray!20}\multicolumn{7}{l}{\textit{Teacher Models}} \\
Zero-shot teacher  & 58.8 & -- & -- & -- & -- & -- \\
Few-shot teacher   & -- & 65.0 & 67.2 & 70.4 & 73.3 & 76.9 \\
\bottomrule
\end{tabular}
}
\vspace{-0.1in}
\end{wraptable}

In \Cref{tab:F1_imagenet,fig:F1_10_datasets_fewshot}, we observe the improvement of \texttt{\textbf{DHO}} under few-shot teachers, \ie, \texttt{\textbf{DHO-F}}. In this section, we also compare \texttt{\textbf{DHO-F}} against single-head KD baselines under few-shot teachers, \ie, CE + KD (logit)-F. Specifically, we train and evaluate CE+KD (logit)-F under the same experimental setup as \Cref{fig:F1_10_datasets_fewshot}, and summarize the results in \Cref{tab:hHhD-1}. We observe that the baseline improves with the few-shot teacher; however, $\texttt{\textbf{DHO-F}}$ still outperforms the baseline even under the same few-shot teacher setting.

\subsection{Comparison to Simple 2-Stage Baseline}
\label{sec:appendix_2-stage}

\begin{wraptable}{r}{0.58\textwidth}
\centering
\vspace{-0.15in}
\caption{Comparison with \textbf{simple two-stage baselines} with/without \textbf{FixMatch}~\citep{sohn2020fixmatch} on 10 datasets (average) using ResNet-18 with a zero-shot teacher.}
\label{tab:2-stage}
\vspace{-0.1in}
\resizebox{\linewidth}{!}{
\begin{tabular}{lccccccc}
\toprule
Method & FixMatch & 0-shot & 1-shot & 2-shot & 4-shot & 8-shot & 16-shot \\
\midrule
KD & - & \textbf{53.9} & - & - & - & - & - \\
CE + KD (logit) & - & - & \underline{55.6} & \underline{58.2} & \underline{61.1} & 64.8 & 68.4 \\
\midrule
2-stage & - & - & 53.0 & 53.6 & 58.5 & 63.7 & 68.6 \\
2-stage & \checkmark & - & 53.5 & 54.6 & 59.8 & \underline{66.7} & \underline{70.5} \\
\midrule
$\textbf{\texttt{DHO}}$ & - & - & \textbf{58.9} & \textbf{62.2} & \textbf{68.4} & \textbf{73.1} & \textbf{76.5} \\
\midrule
Zero-shot teacher & - & 58.8 & - & - & - & - & - \\
\bottomrule
\end{tabular}
}
\vspace{-0.1in}
\end{wraptable}

One possible way to mitigate gradient conflict in few-/low-shot settings with VLM teachers is to train a single-head classifier in two distinct stages. Specifically, we can (1) train a single-head classifier using only the KD loss $\mathcal{L}_\text{KD}$ in \Cref{eq:logit} on both the labeled dataset $\mathcal{D}^{(l)}$ and the unlabeled dataset $\mathcal{D}^{(u)}$, and then (2) fine-tune the classifier using only the CE loss $\mathcal{L}_\text{CE}$ in \Cref{eq:ce} on the labeled dataset $\mathcal{D}^{(l)}$. To investigate the effectiveness of \textbf{\texttt{DHO}} over this simple 2-stage baseline, we conduct additional experiments using the exact setup employed in \Cref{fig:F1_10_datasets_zeroshot} and \Cref{tab:implementation_details} on the same 10 datasets, with the only exception that we train each stage with 100 epochs to make up a total 200 epochs for fair comparison. 
We also consider using a representative semi-supervised learning method, \ie, \textbf{FixMatch}~\citep{sohn2020fixmatch}, in the second stage to prevent overfitting.
For FixMatch, we follow the original hyperparameters $\tau=0.95$ and $\lambda_u=1$, using RandAugment and Cutout as strong augmentations. The only exception is $\mu=7$, which we set to $1$ (\ie, a 1:1 ratio of labeled to unlabeled samples per mini-batch), matching the setup of \texttt{\textbf{DHO}}.

As shown in \Cref{tab:2-stage}, the two-stage baseline performs similarly to CE + KD (logit) when a comparatively larger number of labeled samples is available (\eg, 8-/16-shot), but it significantly degrades when labeled data is scarce (\eg, 1-/2-/4-shot). This result is intuitive: during the second stage of fine-tuning on the labeled data only, as the amount of labeled data decreases, the classifier becomes more prone to overfitting.
As expected, applying FixMatch in the second stage improves performance by preventing overfitting.
However, \textbf{\texttt{DHO}} consistently achieves the best performance, even compared to this two-stage baseline. This indicates that \textbf{\texttt{DHO}} is more efficient in practice, since it avoids splitting training into multiple stages and does not require additional fine-tuning cycles, thereby reducing engineering complexity (\eg, tuning $\tau$, $\lambda_u$, $\mu$, and strong augmentations of FixMatch).

\subsection{\texttt{\textbf{DHO}} is Orthogornal to SSL Methods}
\label{sec:orthogornal_ssl}

Due to the simplicity of $\texttt{\textbf{DHO}}$, we observe that it can be seamlessly integrated with, and improves upon, recent category-aware KD methods~\citep{zhao2022decoupled,lv2024wasserstein}, as shown in \Cref{sec:appendix_dkd_wkd}. In this section, we further demonstrate that $\texttt{\textbf{DHO}}$ is also orthogonal to semi-supervised learning (SSL) methods such as DebiasedPL~\citep{wang2022debiased}, which leverages \textbf{debiased pseudo labels} to train a classifier.
Specifically, we train $h_{\text{CE}}$ using DebiasedPL, while $h_{\text{KD}}$ is trained with the KD loss of $\texttt{\textbf{DHO}}$, \ie, $\mathcal{L}_{\text{KD}}$ in \Cref{eq:dho_kd}. 

\begin{wraptable}{r}{0.35\textwidth}
\centering
\caption{Comparison with DebiasedPL~\citep{wang2022debiased}.}
\label{tab:debiasedpl}
\vspace{-0.12in}
\resizebox{0.95\linewidth}{!}{
\begin{tabular}{lc}
\toprule
Method & ImageNet 1\% \\
\midrule
DebiasedPL & 63.9 \\
DebiasedPL + $\texttt{\textbf{DHO}}$ & \textbf{65.3} \textbf{(+1.4)} \\
Zero-shot VLM teacher & 60.3 \\
\bottomrule
\end{tabular}
}
\vspace{-0.12in}
\end{wraptable}
We conduct experiments on ImageNet with 1\% labels on ImageNet using ResNet-50 with a zero-shot VLM teacher. We use a batch size of 384, an EMA teacher model, multi-view augmentation, and 50 training epochs. Because the Google Drive link in the \href{https://github.com/frank-xwang/debiased-pseudo-labeling}{official codebase} for the 1\% ImageNet split has expired, we randomly sample a new 1\% split, which may introduce minor discrepancies relative to the numbers reported in the original paper.
As shown in \Cref{tab:debiasedpl}, $\texttt{\textbf{DHO}}$ improves DebiasedPL by \textbf{+1.4}, demonstrating that $\texttt{DHO}$ is both simple and complementary to strong SSL approaches.

\begin{table}[h]
    \centering
    \caption{Performance of ViT-B/16 and ViT-L/14 distilled from ViT-H/14 with entropy adaptive weighting under different percentages of labeled data.}
    \label{tab:vit_results}
    \begin{tabular}{l|cc|cc}
    \toprule
    \multirow{2}{*}{Method} & \multicolumn{2}{c|}{ViT-B/16} & \multicolumn{2}{c}{ViT-L/14} \\
    & 1\% & 10\% & 1\% & 10\% \\
    \midrule
    $\textbf{\texttt{DHO}}$ & \textbf{81.66} & \textbf{82.78} & 84.59 & \textbf{85.94} \\
    $\textbf{\texttt{DHO}}$+Ent & \textbf{81.66} & 82.65 & \textbf{84.60} & 85.92 \\
    \bottomrule
    \end{tabular}
\end{table}


\subsection{Out-of-Distribution Evaluation upon fully trained model}
\label{sec:appendix_robustness_full}

We provide the evaluation on out-of-distribution datasets with fully trained model in \Cref{tab:appendix_robustness}.
\textbf{\texttt{DHO}} significantly outperformed zero-shot baselines on similar-distribution variants (ImageNet-V2, ImageNet-Sketch) across both ViT-B/16 and ViT-L/14 architectures, but showed performance degradation on out-of-distribution datasets (ImageNet-R, ImageNet-A), suggesting increased distribution overfitting from full model training.
Interestingly, ViT-B/16 models distilled from ViT-L/14 handled shifted distributions better than those taught by the larger ViT-H/14 DFN~\citep{fang2023data}, despite the latter's superior performance on shifted distributions such as ImageNet-R and ImageNet-A.
We attribute this to the shared training background between ViT-B/16 and ViT-L/14 in the CLIP framework~\citep{radford2021learning}, which appears to better preserve generalization capabilities during the adaptation.
This points to an important insight: \textbf{our method works best on out-of-distributions when the teacher and student models share similar training distributions}, suggesting that successful knowledge distillation also depends on the alignment between teacher and student than just the teacher's raw capabilities.

\begin{table}[ht]
    \caption{
        Accuracy(\%) of \textbf{\texttt{DHO}} with full training model on the ImageNet distribution-shifted variants.
    }
    \label{tab:appendix_robustness}
    \centering
    \resizebox{\textwidth}{!}{
    \begin{tabular}{llllccccc}
        \toprule
        \textbf{Student Model} & \textbf{Params (M)} & \textbf{Labeled Data} & \textbf{Teacher Model} & \textbf{Val} & \textbf{V2} & \textbf{Sketch} & \textbf{R} & \textbf{A} \\
        \midrule
        \rowcolor{gray!20}\multicolumn{9}{l}{\textit{ViT-B/16 Student}} \\
        ViT-B/16~\citep{radford2021learning} & 86M & zero-shot & - & 66.7 & 60.8 & 46.2 & \textbf{74.0} & \textbf{47.0} \\
        ViT-B/16 & 86M & 1\% & ViT-L/14 & 78.7 & 70.1 & 48.0 & 70.9 & 41.1 \\
        ViT-B/16 & 86M & 10\% & ViT-L/14 & 80.8 & 71.3 & 47.4 & \underline{71.7} & \underline{41.4} \\
        ViT-B/16 & 86M & 1\% & ViT-H/14 & \underline{81.6} & \underline{72.6} & \underline{50.6} & 65.5 & 35.6 \\
        ViT-B/16 & 86M & 10\% & ViT-H/14 & \textbf{82.8} & \textbf{73.6} & \textbf{50.7} & 67.7 & 37.8 \\
        \midrule
        \rowcolor{gray!20}\multicolumn{9}{l}{\textit{ViT-L/14 Student}} \\
        ViT-L/14~\citep{radford2021learning} & 304M & zero-shot & - & 75.3 & 68.3 & 59.2 & \textbf{86.5} & \textbf{74.6} \\
        ViT-L/14 & 304M & 1\% & ViT-H/14 & \underline{84.6} & \underline{77.0} & \underline{61.5} & 79.9 & 60.8 \\
        ViT-L/14 & 304M & 10\% & ViT-H/14 & \textbf{85.9} & \textbf{77.8} & \textbf{61.7} & \underline{82.8} & \underline{64.4} \\
        \midrule
        \rowcolor{gray!20}\multicolumn{9}{l}{\textit{Zero-shot VLM}} \\
        ViT-H~\citep{fang2023data} & 632M & zero-shot & - & 83.6 & 77.2 & 71.7 & 92.3 & 77.4 \\
        \bottomrule
    \end{tabular}
    }
\end{table}


\subsection{Repeated-Run Reliability}
\label{sec:appendix_reliability}

We evaluate the variability of \textbf{\texttt{DHO}} in the 16-shot semi-supervised setting on eight datasets using three seeds (1, 2, and 3). The student is an ImageNet-pretrained ResNet-18 with linear CE and KD heads, and the teacher is zero-shot OpenAI CLIP RN50. Each seed controls both labeled-example sampling and training randomness; the test split is fixed. These measurements therefore capture both sources of variability.

\paragraph{Protocol.}
For this supplementary study, we train for 200 epochs using AdamW with learning rate $10^{-3}$, weight decay $10^{-2}$, batch size 64 per labeled/unlabeled stream, distillation temperature 2, and loss weight $\lambda=0.5$. The checkpoint and inference parameters are selected using validation accuracy: $\alpha\in\{0,0.05,\ldots,1\}$ and $\beta\in\{0.5,0.75,1,1.25,1.5,2,4\}$. The selected model is then evaluated on the test set. Each run uses one RTX 5090; the 24 runs require 4.38 GPU-hours of recorded runtime.

\begin{table}[htbp]
\centering
\caption{Three-seed test accuracy (\%) of \textbf{\texttt{DHO}} under 16-shot semi-supervision. SD is the sample standard deviation across seeds; both training and labeled-split seeds vary.}
\label{tab:rebuttal_reliability}
\begin{tabular}{lrrrr}
\toprule
Dataset & Seed 1 & Seed 2 & Seed 3 & Mean $\pm$ SD \\
\midrule
Caltech-101 & 92.54 & 92.54 & 93.51 & $92.86\pm0.56$ \\
DTD & 61.23 & 60.40 & 61.23 & $60.95\pm0.48$ \\
EuroSAT & 88.26 & 89.15 & 90.44 & $89.28\pm1.10$ \\
FGVC Aircraft & 54.16 & 52.78 & 52.93 & $53.29\pm0.76$ \\
Oxford Flowers & 97.48 & 97.24 & 97.85 & $97.52\pm0.31$ \\
Oxford Pets & 88.20 & 88.14 & 88.83 & $88.39\pm0.38$ \\
Stanford Cars & 81.21 & 81.54 & 81.51 & $81.42\pm0.18$ \\
UCF-101 & 73.25 & 73.41 & 72.32 & $72.99\pm0.59$ \\
\midrule
Macro average & 79.54 & 79.40 & 79.83 & $79.59\pm0.22$ \\
\bottomrule
\end{tabular}
\end{table}

\paragraph{Aggregation and scope.}
We report arithmetic means and sample standard deviations (SDs, with denominator $n-1$) across the three runs. For the macro row, we first average accuracy equally across the eight datasets within each seed, then compute the mean and SD of those three macro averages. The resulting accuracy is $79.59\pm0.22\%$. These SDs describe run-to-run variation, not confidence intervals or significance tests. This study covers the eight-dataset protocol above; it does not provide uncertainty estimates for the single-run ImageNet 1\%/10\% results.

\subsection{Additional VLM Teachers}
\label{sec:appendix_additional_teachers}

We evaluate \textbf{\texttt{DHO}} with two additional teacher checkpoints: OpenCLIP RN50 pretrained on YFCC15M (\texttt{RN50-quickgelu/yfcc15m}) and SigLIP ViT-B/16 pretrained on WebLI (\texttt{ViT-B-16-SigLIP/webli})~\citep{zhai2023sigmoid}. OpenCLIP provides another CLIP pretraining source, whereas SigLIP extends the evaluation beyond the CLIP family.

\paragraph{Controlled comparison.}
We use the eight datasets, ImageNet-pretrained ResNet-18 student, 16-shot supervision, and 200-epoch protocol in \Cref{sec:appendix_reliability}. Each CE+KD/DHO pair uses the same teacher checkpoint, labeled and test splits, optimizer settings, and training/split seed (1). CE+KD applies both losses to one head; \textbf{\texttt{DHO}} uses separate linear heads and validation-selected interpolation. Checkpoints are selected by validation accuracy for both methods. The table reports one run per method and dataset, with equal-weight macro averages; no three-seed mean is mixed into this comparison. The 32 runs require 7.44 GPU-hours of recorded runtime, with one RTX 5090 per run.

\begin{table}[htbp]
\centering
\caption{Test accuracy (\%) with additional VLM teachers under 16-shot semi-supervision. Both methods use seed 1 and matched teacher/data settings within each comparison.}
\label{tab:rebuttal_teachers}
\begin{tabular}{lrrrr}
\toprule
& \multicolumn{2}{c}{OpenCLIP RN50} & \multicolumn{2}{c}{SigLIP ViT-B/16} \\
\cmidrule(lr){2-3}\cmidrule(lr){4-5}
Dataset & CE+KD & DHO & CE+KD & DHO \\
\midrule
Caltech-101 & 65.92 & \textbf{87.79} & 93.83 & \textbf{94.16} \\
DTD & 34.46 & \textbf{58.63} & 58.39 & \textbf{66.37} \\
EuroSAT & 33.49 & \textbf{87.37} & 71.64 & \textbf{91.73} \\
FGVC Aircraft & 17.22 & \textbf{45.87} & 47.76 & \textbf{61.45} \\
Oxford Flowers & 69.71 & \textbf{96.22} & 94.64 & \textbf{98.54} \\
Oxford Pets & 49.66 & \textbf{72.23} & 88.96 & \textbf{90.19} \\
Stanford Cars & 19.66 & \textbf{52.20} & 86.83 & \textbf{89.07} \\
UCF-101 & 23.71 & \textbf{62.68} & 68.78 & \textbf{74.41} \\
\midrule
Macro average & 39.23 & \textbf{70.37} & 76.36 & \textbf{83.24} \\
\bottomrule
\end{tabular}
\end{table}

\paragraph{Results.}
\textbf{\texttt{DHO}} improves over CE+KD on all eight datasets for both teachers. The macro gains are 31.14 percentage points with OpenCLIP and 6.88 points with SigLIP. These controlled comparisons support applicability to the additional checkpoints tested; they do not establish statistical significance or performance across all VLM teachers.

\section{Additional Analysis}
\label{sec:additional_analyses}

\subsection{Experiments on Feature Quality}
\label{sec:appendix_feature_quality}

In this section, we provide more insights into the impact of gradient conflict on \textbf{feature learning}. To measure feature quality, we introduce the following definition:

\begin{definition}[Feature quality] \emph{Let $Z\in\mathbb{R}^{N\times d}$ be the feature matrix of $X$ (\ie, $Z=[g(x_1),\ldots,g(x_N)]^\top$) and $Y\in\mathbb{R}^{N\times C}$ the one-hot encoded target matrix. For $\lambda\ge0$, we define the \textbf{feature quality} of $Z$ as the expected risk ($\downarrow$) of ridge regression}:
\[
\mathcal{L}=\mathbb{E}_{z,y}\!\left[\|y-z^\top\hat W_\lambda\|_2^2\right],\quad
\text{where }\hat W_\lambda=(Z^\top Z+\lambda N I_d)^{-1}Z^\top Y.
\]
\label{def:feature_quality}
\end{definition}

We adopt the expected risk of ridge regression ($\downarrow$) as the measure of feature quality for simplicity (while linear probing is widely used and we also use it in \Cref{tab:linear_evaluation}).

Let $g_1:=\nabla_\theta \mathcal{L}_{\mathrm{CE}}(\theta)$, $g_2:=\nabla_\theta \mathcal{L}_{\mathrm{KD}}(\theta)$, $g:=\nabla_\theta \mathcal{L}(\theta)$, and $G=\tfrac12 g_1+\tfrac12 g_2$. We assume:

\begin{assumption}[Gradient cone]
\label{assump:gradient_cone}
\[
g=\beta_1 g_1+\beta_2 g_2+r_\perp,\qquad \beta_1,\beta_2\ge0,\qquad r_\perp^\top g_1=r_\perp^\top g_2=0.
\]
\end{assumption}

This means that the probe gradient lies within the non-negative cone spanned by $g_1$ and $g_2$, up to an orthogonal residual $r_\perp$---that is, both $g_1$ and $g_2$ \textbf{contribute positively to improving feature quality}.

Taking a vanilla SGD step ($\theta - \eta G$, $\eta>0$), with first-order Taylor expansion:
\begin{equation}
\label{eq:Nkfv-1}
\mathcal{L}(\theta-\eta G)
=\mathcal{L}(\theta)-\eta\, g^\top G+\mathcal{O}(\eta^2)
\approx \mathcal{L}(\theta)-\eta g^\top G.
\end{equation}

We neglect the $\mathcal{O}(\eta^{2})$ term (\eg, $\eta<10^{-3}$ makes it negligible). Under the \Cref{assump:gradient_cone}:
\begin{equation}
\label{eq:Nkfv-2}
\begin{aligned}
g^\top G
&= \tfrac{1}{2}\, g^\top (g_1 + g_2) \\
&= \tfrac{1}{2}\, (\beta_1 g_1 + \beta_2 g_2 + r_\perp)^\top (g_1 + g_2) \\
&= \tfrac{1}{2}\bigl(\beta_1 \|g_1\|_2^2 + \beta_2 \|g_2\|_2^2\bigr)
  + \tfrac{1}{2}(\beta_1 + \beta_2)\, g_1^\top g_2.
\end{aligned}
\end{equation}
If gradients are aligned ($g_1^\top g_2>0$), then \Cref{eq:Nkfv-2} is positive, accelerating the decrease in \Cref{eq:Nkfv-1}, and thus \textbf{the expected probe risk decreases more}.

To empirically validate this, we report
\(
\hat{\mathcal{L}}:=\frac{1}{M}\sum_{i=1}^M \|y_i-z_i^\top\hat W_\lambda\|_2^2
\)
on the test set over 200 epochs with the same setup as \Cref{fig:F1_10_datasets_zeroshot}:

\begin{table}[H]
\vspace{-0.1in}
\caption{\textbf{Feature quality} throughout training, computed using \Cref{def:feature_quality}.}
\label{tab:appendix_feature_quality}
\centering
\resizebox{\textwidth}{!}{
\begin{tabular}{lccccccccccccccccccc}
\toprule
10 Datasets (Average) & 1 & 11 & 21 & 31 & 41 & 51 & 61 & 71 & 81 & 91 & 101 \\
\midrule
CE+KD (logit) 
& 4.974 & 4.664 & 4.638 & 4.608 & 4.601 & 4.598 & 4.595 & 4.595 & 4.594 & 4.595 & 4.595 \\
\textbf{\texttt{DHO}} 
& 4.672 & 4.435 & 4.413 & 4.391 & 4.390 & 4.390 & 4.390 & 4.393 & 4.393 & 4.394 & 4.395 \\
\midrule
10 Datasets (Average) & 111 & 121 & 131 & 141 & 151 & 161 & 171 & 181 & 191 \\
\midrule
CE+KD (logit) 
& 4.595 & 4.596 & 4.596 & 4.596 & 4.597 & 4.598 & 4.598 & 4.599 & 4.599 \\
\textbf{\texttt{DHO}}
& 4.396 & 4.396 & 4.398 & 4.399 & 4.400 & 4.400 & 4.400 & 4.401 & 4.401 \\
\bottomrule
\end{tabular}
}
\vspace{-0.1in}
\end{table}

\Cref{tab:appendix_feature_quality} shows that across all epochs, $\hat{\mathcal{L}}$ in \Cref{def:feature_quality} of \textbf{\texttt{DHO}} is lower than that of CE+KD (logit), \textbf{confirming our analysis on feature quality}.

\subsection{Per-Layer Gradient Conflict Analysis}
\label{sec:appendix_per_layer_conflict}

\Cref{fig:average_gradient_conflict} reports the gradient cosine similarity computed over the entire flattened parameter vector $\nabla_\theta \mathcal{L}$. A natural concern is that this single high-dimensional inner product may average out per-layer structure---for instance, conflict could be concentrated only in the classifier head while early backbone layers remain neutral, which would weaken the propagation argument in \Cref{prop:feature_extractor_conflict}. We address this concern by decomposing the cosine similarity per ResNet stage.

\paragraph{Setup.}
We train SHO (CE+KD logit, single head) and \textbf{\texttt{DHO}} from scratch on 8 datasets (Caltech-101, DTD, EuroSAT, FGVC-Aircraft, Oxford Flowers, Oxford Pets, Stanford Cars, UCF101) under the 16-shot setting using a ResNet-18 student and a CLIP ResNet-50 zero-shot teacher, with the same hyperparameters as \Cref{fig:F1_10_datasets_zeroshot}. During training, every 20 epochs we record $\nabla_\theta \mathcal{L}_\mathrm{CE}$ and $\nabla_\theta \mathcal{L}_\mathrm{KD}$ separately and compute their cosine similarity restricted to the parameters of each ResNet stage \texttt{layer1}--\texttt{layer4} (and, for SHO, the shared classifier head). For \textbf{\texttt{DHO}}, the classifier-head similarity is undefined by construction: $h_\mathrm{CE}$ receives only $\nabla\mathcal{L}_\mathrm{CE}$ and $h_\mathrm{KD}$ receives only $\nabla\mathcal{L}_\mathrm{KD}$, so the two never meet at the head---this architectural decoupling is precisely the design intent.

\paragraph{Conflict in SHO is uniform across the backbone and deepens with depth.}
\Cref{tab:per_layer_conflict_avg} reports the 8-dataset average cosine similarity at the final epoch. SHO exhibits negative alignment at \emph{every} backbone stage and the magnitude grows monotonically with depth (\texttt{layer1}: $-0.39 \to$ \texttt{layer4}: $-0.45$, classifier: $-0.76$). This is the strongest empirical confirmation of \Cref{prop:feature_extractor_conflict}: the conflict at the classifier head propagates backward through the Jacobian and registers at every depth, not only near the head. \textbf{\texttt{DHO}}, in contrast, maintains near-zero / mildly positive alignment uniformly across \texttt{layer1}--\texttt{layer4}, with the gap to SHO exceeding $+0.46$ at every stage.

\begin{table}[h]
\caption{\textbf{Per-layer gradient cosine similarity} between $\nabla_\theta\mathcal{L}_\mathrm{CE}$ and $\nabla_\theta\mathcal{L}_\mathrm{KD}$ at the final epoch, averaged over 8 datasets (16-shot, ResNet-18 student, CLIP ResNet-50 zero-shot teacher). SHO conflict is uniform across the backbone and deepens monotonically with depth; \textbf{\texttt{DHO}} maintains near-zero / mildly positive alignment at every stage. The classifier-head similarity for \textbf{\texttt{DHO}} is undefined by construction (separate heads receive disjoint gradients).}
\label{tab:per_layer_conflict_avg}
\centering
\small
\begin{tabular}{lccc}
\toprule
\textbf{Layer} & \textbf{SHO} & \textbf{\texttt{DHO}} & \textbf{$\Delta$ (DHO $-$ SHO)} \\
\midrule
\texttt{layer1} (early)         & $-0.3924$ & $+0.0738$ & $+0.4662$ \\
\texttt{layer2}                 & $-0.3931$ & $+0.0695$ & $+0.4626$ \\
\texttt{layer3}                 & $-0.4048$ & $+0.0642$ & $+0.4690$ \\
\texttt{layer4} (deepest)       & $-0.4526$ & $+0.0459$ & $+0.4985$ \\
\midrule
classifier head                 & $-0.7570$ & --        & -- \\
\bottomrule
\end{tabular}
\end{table}

\paragraph{Per-dataset breakdown at \texttt{layer4}.}
\Cref{tab:per_layer_conflict_per_dataset} reports the per-dataset cosine similarity at \texttt{layer4} (the deepest backbone stage, where conflict is strongest). SHO shows negative alignment on \emph{every} dataset, and \textbf{\texttt{DHO}} resolves the conflict on \emph{every} dataset. Notably, the datasets with the most severe SHO conflict (EuroSAT: $-0.80$, DTD: $-0.67$, FGVC-Aircraft: $-0.62$) are the same datasets on which \textbf{\texttt{DHO}} achieves the largest improvements over SHO in \Cref{fig:F1_10_datasets_zeroshot}, suggesting that the per-layer alignment metric tracks the magnitude of the realized accuracy gain.

\begin{table}[h]
\caption{\textbf{Per-dataset gradient cosine similarity at \texttt{layer4}} (16-shot, ResNet-18 student, CLIP ResNet-50 zero-shot teacher). \textbf{\texttt{DHO}} resolves conflict on every dataset; the datasets with the strongest SHO conflict (EuroSAT, DTD, FGVC-Aircraft) coincide with the largest DHO accuracy gains in \Cref{fig:F1_10_datasets_zeroshot}.}
\label{tab:per_layer_conflict_per_dataset}
\centering
\small
\begin{tabular}{lccc}
\toprule
\textbf{Dataset} & \textbf{SHO} & \textbf{\texttt{DHO}} & \textbf{$\Delta$ (DHO $-$ SHO)} \\
\midrule
EuroSAT          & $-0.8035$ & $+0.0040$ & $+0.8075$ \\
DTD              & $-0.6682$ & $+0.0615$ & $+0.7297$ \\
FGVC-Aircraft    & $-0.6213$ & $+0.0223$ & $+0.6437$ \\
Oxford Flowers   & $-0.4542$ & $+0.0572$ & $+0.5114$ \\
UCF101           & $-0.3752$ & $+0.0372$ & $+0.4124$ \\
Stanford Cars    & $-0.2728$ & $+0.0500$ & $+0.3228$ \\
Oxford Pets      & $-0.2434$ & $+0.0672$ & $+0.3106$ \\
Caltech-101      & $-0.1822$ & $+0.0772$ & $+0.2594$ \\
\midrule
\textbf{Average} & $\mathbf{-0.4526}$ & $\mathbf{+0.0459}$ & $\mathbf{+0.4985}$ \\
\bottomrule
\end{tabular}
\end{table}

\paragraph{Implication for the empirical claim.}
The per-layer decomposition makes the empirical statement in \Cref{sec:dual_head_interpolation}---that \textbf{\texttt{DHO}} maintains positive gradient alignment throughout training---precise and robust. It is not an artifact of averaging a high-dimensional vector: positivity (or, conservatively, non-negativity) holds at every backbone stage and on every dataset evaluated. Conversely, the corresponding statement for SHO---negative alignment forced by the PSD form $A = J^\top W^\top W J$ in \Cref{prop:feature_extractor_conflict}---is realized at every layer rather than only at the classifier head, supporting the propagation argument.





\subsection{Additional Analysis on Non-Linear Head Design}
\label{sec:appendix_nonlinear}

To further investigate the architectural advantages of dual head optimization, we conducted experiments with non-linear head designs, replacing the linear heads used in our main experiments.
We design a non-linear classifier with a sequence of layers: an initial linear projection layer, followed by layer normalization~\citep{ba2016layer}, GELU activation~\citep{hendrycks2016gaussian}, dropout~\citep{srivastava2014dropout}, and a final linear classification layer.
We compared \textbf{\texttt{DHO}} with three non-linear configurations; \textbf{\texttt{DHO}+NL-Head-CE}: non-linear CE head, \textbf{\texttt{DHO}+NL-Head-KD}: non-linear KD head, and \textbf{\texttt{DHO}+NL-Head-CE+KD}: non-linear both CE and KD heads.
All experiments followed the few-shot semi-supervised setting detailed in \Cref{sec:appendix_implementation_details}.

\begin{table}[htb]
    \caption{
    Results of different \textbf{non-linear head configurations} on \textbf{11 datasets} including \textbf{ImageNet} under few-shot semi-supervision using \textbf{ResNet-18} with \textbf{zero-shot teacher}~\citep{radford2021learning}.
    We report averaged accuracy for 10 visual recognition datasets except the ImageNet.
    }
    \centering
    \resizebox{\textwidth}{!}{%
    \begin{tabular}{lcccccc|ccccc}
        \toprule
        & \multicolumn{6}{c|}{\textbf{ImageNet}} & \multicolumn{5}{c}{\textbf{Average of 10 tasks}} \\
        \cmidrule(r){2-7} \cmidrule(l){8-12}
        \textbf{Configuration} & 1-shot & 2-shot & 4-shot & 8-shot & 16-shot & & 1-shot & 2-shot & 4-shot & 8-shot & 16-shot \\
        \midrule
        \textbf{\texttt{DHO}} (base) & 61.7 & 62.2 & 62.6 & 63.8 & 65.1 & & 58.9 & \textbf{62.2} & 68.4 & \textbf{73.1} & \textbf{76.5} \\
        \textbf{\texttt{DHO}}+NL-Head-CE & 61.7 & 61.9 & 62.2 & 63.1 & 64.8 & & \textbf{59.3} & \underline{62.1} & 68.0 & \underline{72.7} & \underline{76.2} \\
        \textbf{\texttt{DHO}}+NL-Head-KD & \textbf{62.1} & \textbf{62.6} & \textbf{62.9} & \textbf{64.0} & \textbf{65.9} & & 58.3 & \underline{62.1} & 67.8 & 72.4 & \textbf{76.5} \\
        \textbf{\texttt{DHO}}+NL-Head-CE+KD & \underline{62.0} & \underline{62.3} & \underline{62.6} & \underline{63.8} & \underline{65.4} & & \underline{59.1} & \textbf{62.2} & \textbf{68.6} & 72.4 & \textbf{76.5} \\
        \bottomrule
    \end{tabular}
    }
    \label{tab:nonlinear_all_tasks}
\end{table}

\begin{figure}[htb]
    \centering
    \includegraphics[width=0.9\textwidth]{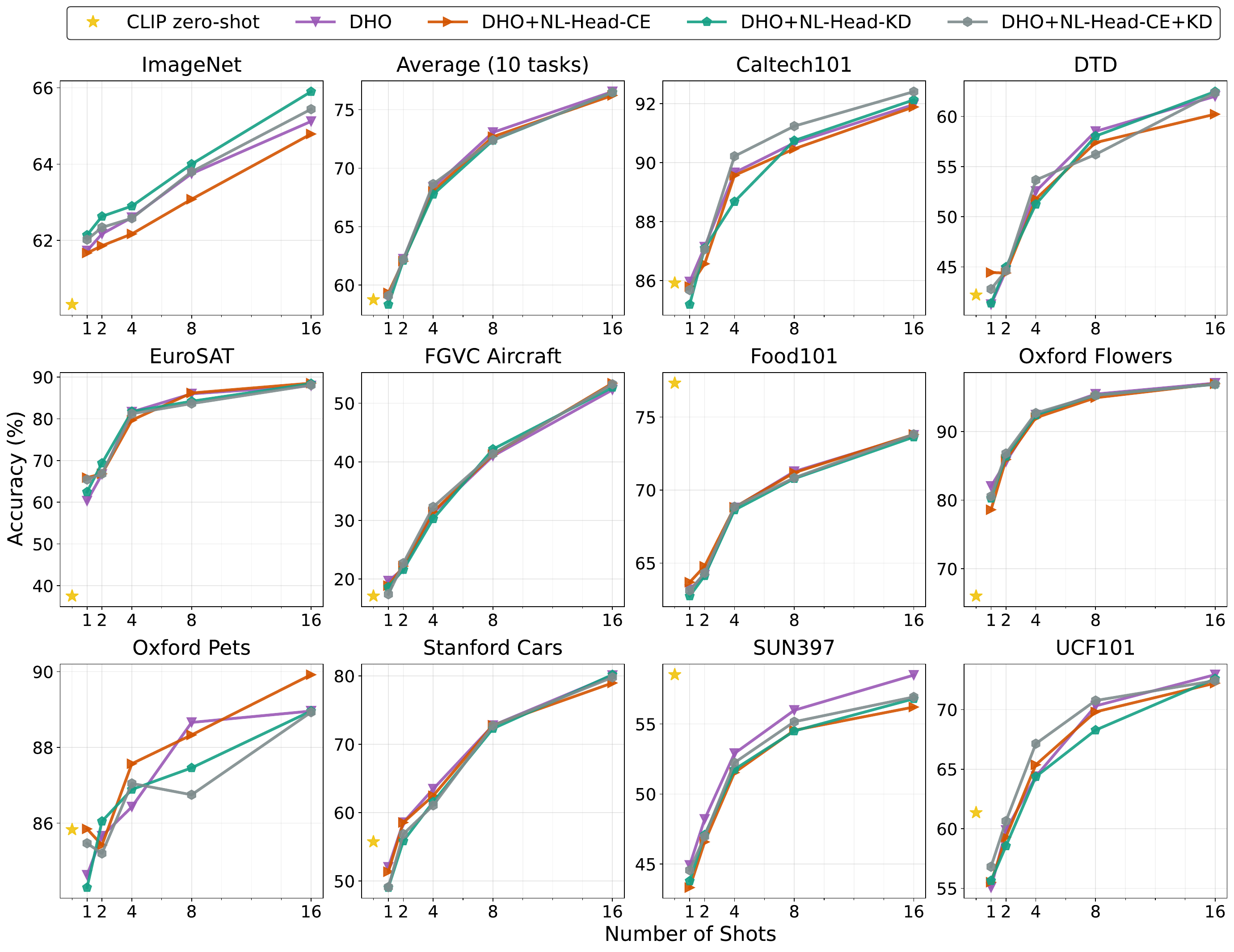}
    \caption{
        Results of different \textbf{non-linear head configurations} on \textbf{11 datasets} including \textbf{ImageNet} under few-shot semi-supervision using \textbf{ResNet-18} with \textbf{zero-shot teacher}~\citep{radford2021learning}.
    }
    \label{fig:nonlinear_all_tasks}
    \vspace{-0.2in}
\end{figure}

\paragraph{Performance Analysis.}
Our experiments revealed key findings regarding head architecture (\Cref{tab:nonlinear_all_tasks}, \Cref{fig:nonlinear_all_tasks}).
On ImageNet, non-linear KD heads consistently outperformed linear ones, suggesting complex architectures better capture teacher predictions.
Conversely, non-linear CE heads degraded performance, likely due to overfitting on limited labeled data.
While dual non-linear heads outperformed fully linear configurations, they were less effective than non-linearity in the KD head alone.

\begin{wraptable}{r}{0.44\textwidth}
    \small
    \vspace{-0.2in}
    \caption{\small
    Results on \textbf{dual-heads interpolation strategy} of different \textbf{non-linear head configurations} on \textbf{ImageNet} under 16-shots semi-supervised setting.
    }
    \vspace{-0.1in}
    \centering
    \resizebox{0.44\textwidth}{!}{%
        \begin{tabular}{lccc}
            \toprule
            \textbf{Configuration} & \textbf{CE Head} & \textbf{KD Head} & \textbf{Combined} \\
            \midrule
            \textbf{\texttt{DHO}} (base) & 60.64 & 61.55 & 65.37 \\
            \textbf{\texttt{DHO}}+NL-Head-CE & 60.18 & 61.39 & 64.91 \\
            \textbf{\texttt{DHO}}+NL-Head-KD & \underline{60.95} & \underline{61.76} & \textbf{65.97} \\
            \textbf{\texttt{DHO}}+NL-Head-CE+KD & \textbf{61.66} & \textbf{61.81} & \underline{65.59} \\
            \bottomrule
        \end{tabular}
    }
    \vspace{-0.2in}
    \label{tab:appendix_nonlinear}
\end{wraptable}

On the other 10 datasets, optimal configurations varied considerably with no consistently superior approach.
This highlights that non-linear transformation effectiveness depends strongly on dataset characteristics and head functionality.
Given comparable performance but superior computational efficiency, we adopted linear head architectures for all subsequent experiments.


\paragraph{Head Decomposition Analysis.}
Analysis under the 16-shot semi-supervised setting revealed complex relationships between architectural choices and head-wise performance as shown in \Cref{tab:appendix_nonlinear}.
Non-linear CE branches decreased CE head performance from $60.64\%$ to $60.18\%$ despite increased parameters.
Conversely, non-linear KD heads improved both heads: CE accuracy increased to $60.95\%$ and KD prediction to $61.76\%$.
However, dual non-linear heads reduced combined performance from $65.97\%$ to $65.59\%$, suggesting head specialization may compromise joint feature representation.
These findings highlight the complex interplay between architectural decisions and multi-head learning dynamics.


\subsubsection{Regularization of the Nonlinear CE Head}
\label{sec:appendix_nonlinear_regularization}

To examine whether the degradation of nonlinear CE heads depends on regularization, we compare the linear head, the default MLP, and five MLP configurations with stronger dropout and/or CE-head weight decay. The KD head remains linear. We use the eight-dataset, 16-shot ResNet-18/OpenAI CLIP RN50 protocol in \Cref{sec:appendix_reliability}, with training and split seed 1 for every configuration. All other training settings are fixed; the 56 configuration--dataset runs require 10.07 GPU-hours of recorded runtime, including the linear baseline also used in \Cref{sec:appendix_reliability}.

\paragraph{Validation and reporting.}
For each configuration and dataset, we select the checkpoint and interpolation parameters using validation accuracy, then evaluate that model on the test set. \Cref{tab:rebuttal_regularization} reports the equal-weight macro validation and test accuracies of all seven configurations. The validation column uses the recorded combined-head accuracy at the selected checkpoint. We retain all rows as an exploratory comparison; the row with the highest test accuracy is not a validation-selected architectural recommendation.

\begin{table}[htbp]
\centering
\caption{Nonlinear CE-head regularization under 16-shot semi-supervision. Validation and test accuracies (\%) are macro averages over eight datasets, using one seed per configuration. Bold marks the highest value in each accuracy column.}
\label{tab:rebuttal_regularization}
\begin{tabular}{lrrrr}
\toprule
CE head & Dropout & CE-head WD & Validation & Test \\
\midrule
Linear & --- & 0.01 & 82.54 & 79.54 \\
MLP (default) & 0.5 & 0.01 & \textbf{82.83} & 79.30 \\
MLP & 0.7 & 0.01 & 82.73 & 79.07 \\
MLP & 0.5 & 0.05 & 82.38 & 78.97 \\
MLP & 0.5 & 0.10 & 82.48 & 79.06 \\
MLP & 0.7 & 0.05 & 82.76 & 79.10 \\
MLP & 0.7 & 0.10 & 82.74 & \textbf{79.71} \\
\bottomrule
\end{tabular}
\end{table}

\paragraph{Regularization sensitivity.}
The default MLP has the highest macro validation accuracy ($82.83\%$). The MLP with dropout 0.7 and weight decay 0.10 has the highest observed test accuracy ($79.71\%$), exceeding the linear head ($79.54\%$) by 0.17 percentage points. The other four stronger regularizers do not exceed the linear head. Thus, the results indicate sensitivity to regularization, rather than general superiority of nonlinear heads. This is a single-seed study and does not establish statistical significance for the small observed difference.

\subsection{Additional Dual-Head Analysis}
\label{sec:appendix_qualitative_results}

\begin{figure*}[h]
    \centering
    \includegraphics[width=0.95\textwidth]{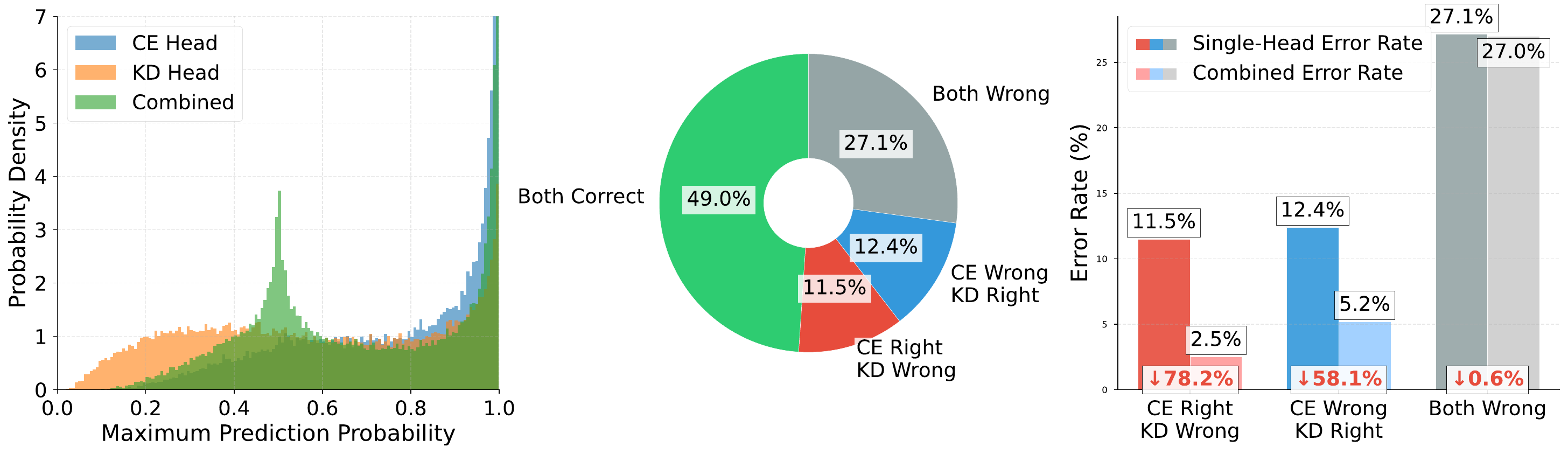}
    \vspace{-0.05in}
    \caption{
        Analysis of \textbf{\texttt{DHO}} on the ImageNet under 16-shot semi-supervised setting.
        \textbf{(Left)} Maximum probability distributions for predictions from CE head, KD head, and their combined output.
        \textbf{(Middle)} Prediction agreement diagram analysis, categorizing cases where both heads are correct, only one head is correct, and both heads are incorrect.
        \textbf{(Right)} Error reduction analysis comparing single-head failure cases against improvements achieved through combined predictions.
    }
    \label{fig:branch_analysis}
\end{figure*}

In this section, we further analyze the prediction behavior of \textbf{\texttt{DHO}}.
As shown in \Cref{fig:branch_analysis} (left), despite sharing feature representations, the CE head ($h_{\text{CE}}$), trained on labeled data, produces sharper predictions, whereas the KD head ($h_{\text{KD}}$), guided by teacher distillation, generates smoother distributions.
Prediction agreement analysis (\Cref{fig:branch_analysis}, middle) shows that the two heads agree in 76.2\% of cases while complementing each other:
the CE head correctly classifies 11.5\% of cases where the KD head fails, and vice versa for 12.4\%.
Error reduction analysis (\Cref{fig:branch_analysis}, right) further demonstrates that our combined approach reduces failure rates from 11.5\% to 2.5\% for the KD head and from 12.4\% to 5.2\% for the CE head, confirming the effectiveness of \textbf{\texttt{DHO}}.

We also present additional qualitative results of \textbf{\texttt{DHO}}, both on ImageNet (see \Cref{fig:additional_qualitative_results_imgnet_1,fig:additional_qualitative_results_imgnet_2}) and on other 10 datasets (see \Cref{fig:additional_qualitative_results_others_1,fig:additional_qualitative_results_others_2}).

\begin{figure}[hb]
    \centering
    \includegraphics[width=0.98\textwidth]{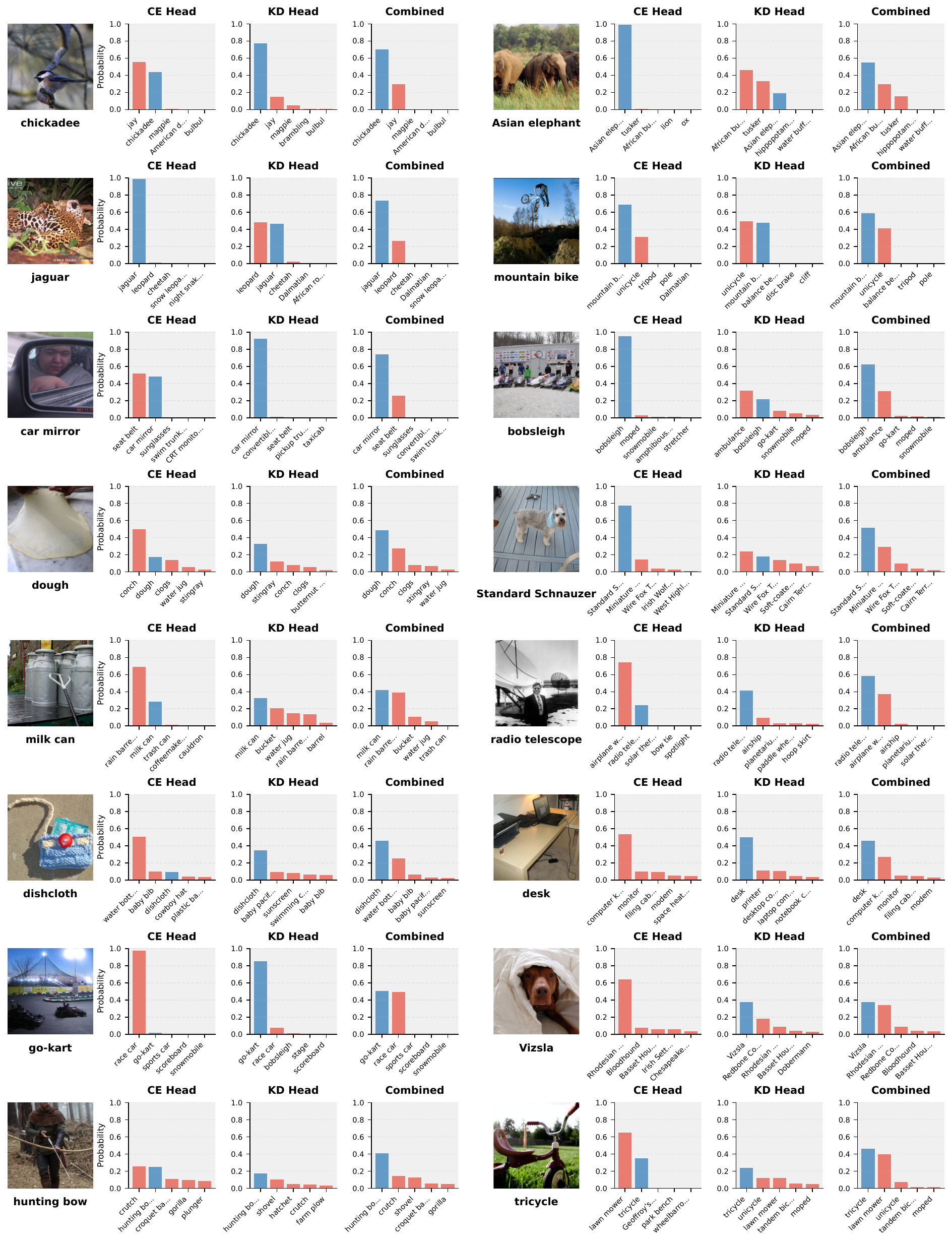}
    \caption{Additional qualitative results on ImageNet under 16-shot semi-supervised setting.}
    \label{fig:additional_qualitative_results_imgnet_1}
\end{figure}

\begin{figure}[hb]
    \centering
    \includegraphics[width=1.0\textwidth]{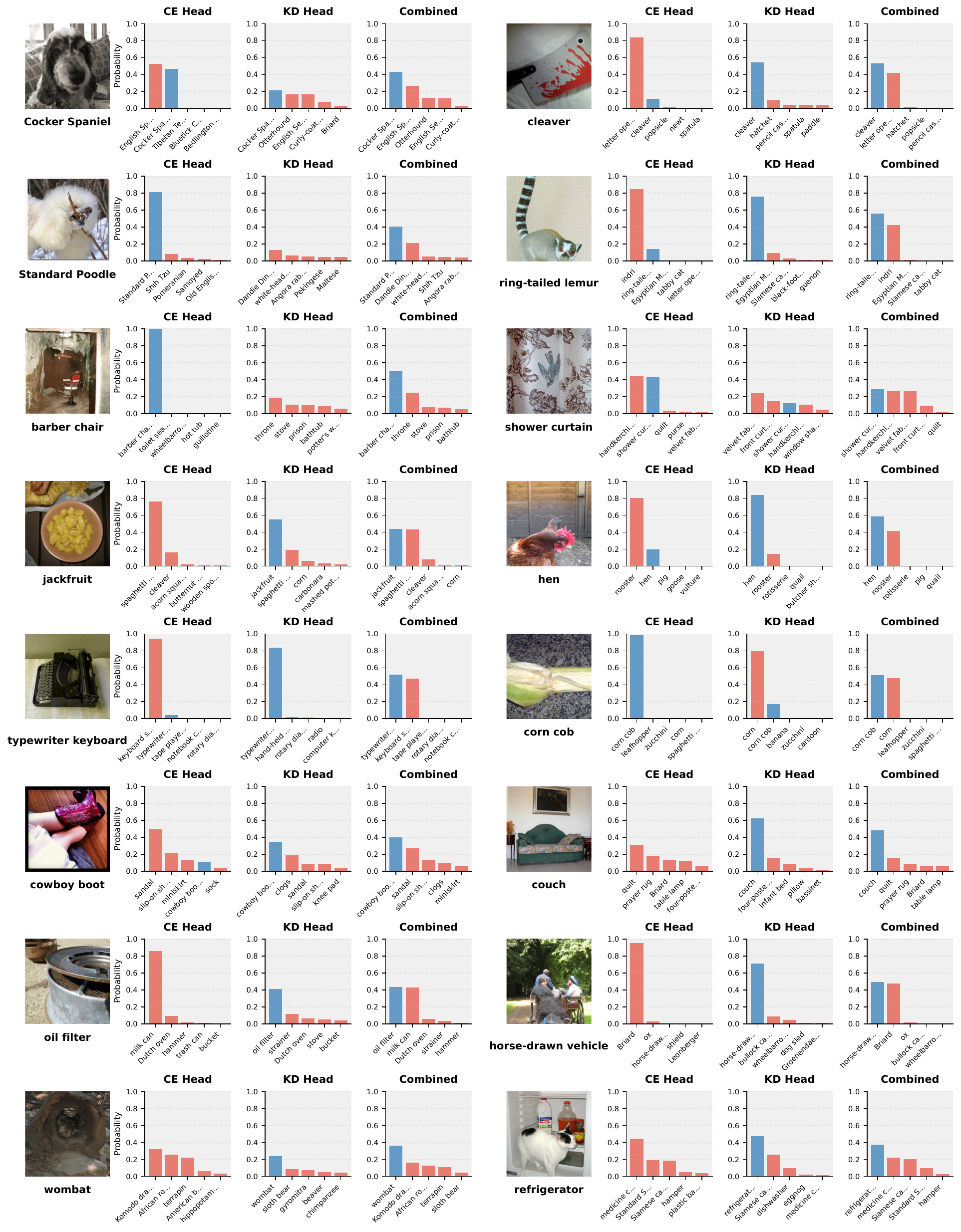}
    \caption{Additional qualitative results on ImageNet under 16-shot semi-supervised setting.}
    \label{fig:additional_qualitative_results_imgnet_2}
\end{figure}

\begin{figure}[hb]
    \centering
    \includegraphics[width=1.0\textwidth]{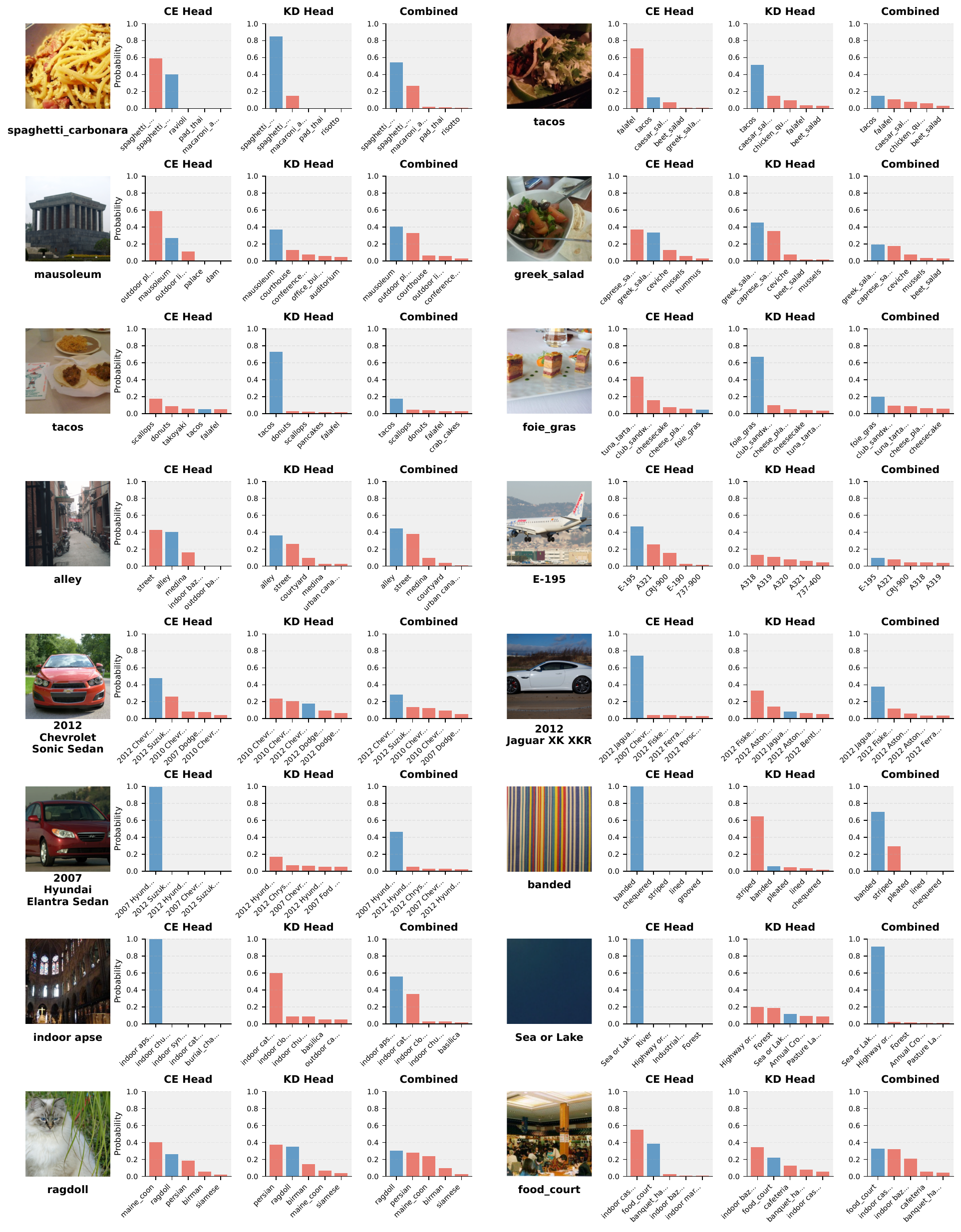}
    \caption{Additional qualitative results on other 10 datasets for models trained under 16-shot semi-supervised setting.}
    \label{fig:additional_qualitative_results_others_1}
\end{figure}

\begin{figure}[hb]
    \centering
    \includegraphics[width=0.98\textwidth]{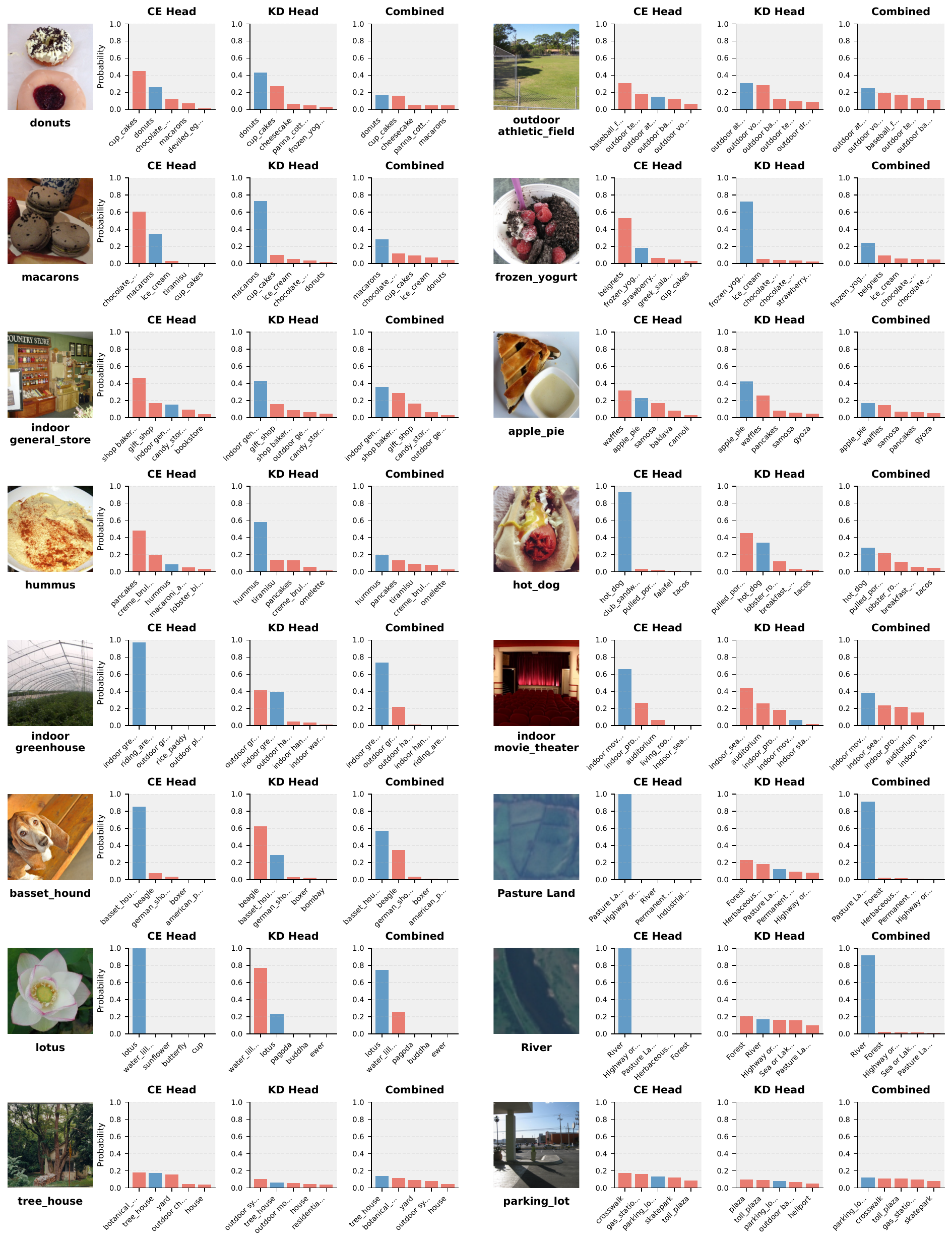}
    \caption{Additional qualitative results on other 10 datasets for models trained under 16-shot semi-supervised setting.}
    \label{fig:additional_qualitative_results_others_2}
\end{figure}

\clearpage
\section{The Use of LLMs}\label{sec:llm}
We used LLMs solely for light editing such as correcting grammatical errors and polishing some words. They did not contribute to research ideation, experiments, analysis, or substantive writing.


\newpage
\newpage
\section*{NeurIPS Paper Checklist}

\begin{enumerate}

\item {\bf Claims}
    \item[] Question: Do the main claims made in the abstract and introduction accurately reflect the paper's contributions and scope?
    \item[] Answer: \answerYes{} 
    \item[] Justification: Yes. Our main contributions are also detailed in \S\ref{sec:introduction}. Our gradient conflict analysis is presented in \S\ref{sec:dho} with formal theoretical guarantees in \Cref{sec:gradient_conflict_analysis}. We also provide post-hoc analysis on dual-head combination in \Cref{sec:inference_equivalence_analysis} with empirical analysis in \Cref{sec:appendix_qualitative_results}.
    \item[] Guidelines:
    \begin{itemize}
        \item The answer NA means that the abstract and introduction do not include the claims made in the paper.
        \item The abstract and/or introduction should clearly state the claims made, including the contributions made in the paper and important assumptions and limitations. A No or NA answer to this question will not be perceived well by the reviewers.
        \item The claims made should match theoretical and experimental results, and reflect how much the results can be expected to generalize to other settings.
        \item It is fine to include aspirational goals as motivation as long as it is clear that these goals are not attained by the paper.
    \end{itemize}

\item {\bf Limitations}
    \item[] Question: Does the paper discuss the limitations of the work performed by the authors?
    \item[] Answer: \answerYes{} 
    \item[] Justification: Yes, we discussed limitations in \S\ref{sec:conclusion_future_work} and \Cref{sec:appendix_limitations} (foundation-model scope and task scope). We also reported computational cost in training in \Cref{sec:appendix_implementation_details}, and analyzed the computational efficiency of our method in \S\ref{sec:experimental_results} (\Cref{tab:computation_cost}) and \Cref{sec:appendix_computational_overhead}.
    \item[] Guidelines:
    \begin{itemize}
        \item The answer NA means that the paper has no limitation while the answer No means that the paper has limitations, but those are not discussed in the paper.
        \item The authors are encouraged to create a separate "Limitations" section in their paper.
        \item The paper should point out any strong assumptions and how robust the results are to violations of these assumptions (e.g., independence assumptions, noiseless settings, model well-specification, asymptotic approximations only holding locally). The authors should reflect on how these assumptions might be violated in practice and what the implications would be.
        \item The authors should reflect on the scope of the claims made, e.g., if the approach was only tested on a few datasets or with a few runs. In general, empirical results often depend on implicit assumptions, which should be articulated.
        \item The authors should reflect on the factors that influence the performance of the approach. For example, a facial recognition algorithm may perform poorly when image resolution is low or images are taken in low lighting. Or a speech-to-text system might not be used reliably to provide closed captions for online lectures because it fails to handle technical jargon.
        \item The authors should discuss the computational efficiency of the proposed algorithms and how they scale with dataset size.
        \item If applicable, the authors should discuss possible limitations of their approach to address problems of privacy and fairness.
        \item While the authors might fear that complete honesty about limitations might be used by reviewers as grounds for rejection, a worse outcome might be that reviewers discover limitations that aren't acknowledged in the paper. The authors should use their best judgment and recognize that individual actions in favor of transparency play an important role in developing norms that preserve the integrity of the community. Reviewers will be specifically instructed to not penalize honesty concerning limitations.
    \end{itemize}

\item {\bf Theory assumptions and proofs}
    \item[] Question: For each theoretical result, does the paper provide the full set of assumptions and a complete (and correct) proof?
    \item[] Answer: \answerYes{} 
    \item[] Justification: Yes, we provided the full set of assumptions and complete proofs for gradient conflict analysis (\Cref{sec:gradient_conflict_analysis}, Propositions~\ref{prop:geometric_conflict}--\ref{prop:dho_no_conflict}) and inference equivalence (\Cref{sec:inference_equivalence_analysis}, Theorems~\ref{thm:single_head_opt}--\ref{thm:inference_equiv}) in \Cref{sec:theoretical_analysis}.
    \item[] Guidelines:
    \begin{itemize}
        \item The answer NA means that the paper does not include theoretical results.
        \item All the theorems, formulas, and proofs in the paper should be numbered and cross-referenced.
        \item All assumptions should be clearly stated or referenced in the statement of any theorems.
        \item The proofs can either appear in the main paper or the supplemental material, but if they appear in the supplemental material, the authors are encouraged to provide a short proof sketch to provide intuition.
        \item Inversely, any informal proof provided in the core of the paper should be complemented by formal proofs provided in appendix or supplemental material.
        \item Theorems and Lemmas that the proof relies upon should be properly referenced.
    \end{itemize}

    \item {\bf Experimental result reproducibility}
    \item[] Question: Does the paper fully disclose all the information needed to reproduce the main experimental results of the paper to the extent that it affects the main claims and/or conclusions of the paper (regardless of whether the code and data are provided or not)?
    \item[] Answer: \answerYes{} 
    \item[] Justification: We used publicly available datasets and models for our experiments. Our proposed method is also simple and easy to implement, with described details in \S\ref{sec:experiments}. We will also release our code and instructions for reproducing the results.
    \item[] Guidelines:
    \begin{itemize}
        \item The answer NA means that the paper does not include experiments.
        \item If the paper includes experiments, a No answer to this question will not be perceived well by the reviewers: Making the paper reproducible is important, regardless of whether the code and data are provided or not.
        \item If the contribution is a dataset and/or model, the authors should describe the steps taken to make their results reproducible or verifiable.
        \item Depending on the contribution, reproducibility can be accomplished in various ways. For example, if the contribution is a novel architecture, describing the architecture fully might suffice, or if the contribution is a specific model and empirical evaluation, it may be necessary to either make it possible for others to replicate the model with the same dataset, or provide access to the model. In general. releasing code and data is often one good way to accomplish this, but reproducibility can also be provided via detailed instructions for how to replicate the results, access to a hosted model (e.g., in the case of a large language model), releasing of a model checkpoint, or other means that are appropriate to the research performed.
        \item While NeurIPS does not require releasing code, the conference does require all submissions to provide some reasonable avenue for reproducibility, which may depend on the nature of the contribution. For example
        \begin{enumerate}
            \item If the contribution is primarily a new algorithm, the paper should make it clear how to reproduce that algorithm.
            \item If the contribution is primarily a new model architecture, the paper should describe the architecture clearly and fully.
            \item If the contribution is a new model (e.g., a large language model), then there should either be a way to access this model for reproducing the results or a way to reproduce the model (e.g., with an open-source dataset or instructions for how to construct the dataset).
            \item We recognize that reproducibility may be tricky in some cases, in which case authors are welcome to describe the particular way they provide for reproducibility. In the case of closed-source models, it may be that access to the model is limited in some way (e.g., to registered users), but it should be possible for other researchers to have some path to reproducing or verifying the results.
        \end{enumerate}
    \end{itemize}

\item {\bf Open access to data and code}
    \item[] Question: Does the paper provide open access to the data and code, with sufficient instructions to faithfully reproduce the main experimental results, as described in supplemental material?
    \item[] Answer: \answerYes{} 
    \item[] Justification: We will provide open access to the data and code, with sufficient instructions to faithfully reproduce the main experimental results, once the blind review is over.
    \item[] Guidelines:
    \begin{itemize}
        \item The answer NA means that paper does not include experiments requiring code.
        \item Please see the NeurIPS code and data submission guidelines (\url{https://nips.cc/public/guides/CodeSubmissionPolicy}) for more details.
        \item While we encourage the release of code and data, we understand that this might not be possible, so "No" is an acceptable answer. Papers cannot be rejected simply for not including code, unless this is central to the contribution (e.g., for a new open-source benchmark).
        \item The instructions should contain the exact command and environment needed to run to reproduce the results. See the NeurIPS code and data submission guidelines (\url{https://nips.cc/public/guides/CodeSubmissionPolicy}) for more details.
        \item The authors should provide instructions on data access and preparation, including how to access the raw data, preprocessed data, intermediate data, and generated data, etc.
        \item The authors should provide scripts to reproduce all experimental results for the new proposed method and baselines. If only a subset of experiments are reproducible, they should state which ones are omitted from the script and why.
        \item At submission time, to preserve anonymity, the authors should release anonymized versions (if applicable).
        \item Providing as much information as possible in supplemental material (appended to the paper) is recommended, but including URLs to data and code is permitted.
    \end{itemize}

\item {\bf Experimental setting/details}
    \item[] Question: Does the paper specify all the training and test details (\eg, data splits, hyperparameters, how they were chosen, type of optimizer, etc.) necessary to understand the results?
    \item[] Answer: \answerYes{} 
    \item[] Justification: We specified all the training and test details (\eg, data splits, hyperparameters, how they were chosen, type of optimizer, etc.) necessary to understand the results in \S\ref{sec:experimental_setups} and \Cref{sec:appendix_implementation_details}.
    \item[] Guidelines:
    \begin{itemize}
        \item The answer NA means that the paper does not include experiments.
        \item The experimental setting should be presented in the core of the paper to a level of detail that is necessary to appreciate the results and make sense of them.
        \item The full details can be provided either with the code, in appendix, or as supplemental material.
    \end{itemize}

\item {\bf Experiment statistical significance}
    \item[] Question: Does the paper report error bars suitably and correctly defined or other appropriate information about the statistical significance of the experiments?
    \item[] Answer: \answerNo{} 
    \item[] Justification: We report three-seed means and sample standard deviations for the eight-dataset study in \Cref{sec:appendix_reliability}. Both training randomness and labeled-split sampling vary across seeds; the test split is fixed. However, repeated runs do not cover all main results, including ImageNet 1\%/10\% comparisons. The additional-teacher comparisons and nonlinear-head regularization study each use one seed per setting. We therefore do not claim statistical significance for those results or the small reported system-level margins.
    \item[] Guidelines:
    \begin{itemize}
        \item The answer NA means that the paper does not include experiments.
        \item The authors should answer "Yes" if the results are accompanied by error bars, confidence intervals, or statistical significance tests, at least for the experiments that support the main claims of the paper.
        \item The factors of variability that the error bars are capturing should be clearly stated (for example, train/test split, initialization, random drawing of some parameter, or overall run with given experimental conditions).
        \item The method for calculating the error bars should be explained (closed form formula, call to a library function, bootstrap, etc.)
        \item The assumptions made should be given (e.g., Normally distributed errors).
        \item It should be clear whether the error bar is the standard deviation or the standard error of the mean.
        \item It is OK to report 1-sigma error bars, but one should state it. The authors should preferably report a 2-sigma error bar than state that they have a 96\% CI, if the hypothesis of Normality of errors is not verified.
        \item For asymmetric distributions, the authors should be careful not to show in tables or figures symmetric error bars that would yield results that are out of range (e.g. negative error rates).
        \item If error bars are reported in tables or plots, The authors should explain in the text how they were calculated and reference the corresponding figures or tables in the text.
    \end{itemize}

\item {\bf Experiments compute resources}
    \item[] Question: For each experiment, does the paper provide sufficient information on the computer resources (type of compute workers, memory, time of execution) needed to reproduce the experiments?
    \item[] Answer: \answerYes{} 
    \item[] Justification: We provided type of hardware, and time of execution needed to reproduce the experiments in \Cref{tab:hardware_training_time} inside \Cref{sec:appendix_implementation_details} for experiments taking the most time. Also we provided extensive information of computational efficiency of our method in \S\ref{sec:experimental_results} and \Cref{sec:appendix_computational_overhead}.
    \item[] Guidelines:
    \begin{itemize}
        \item The answer NA means that the paper does not include experiments.
        \item The paper should indicate the type of compute workers CPU or GPU, internal cluster, or cloud provider, including relevant memory and storage.
        \item The paper should provide the amount of compute required for each of the individual experimental runs as well as estimate the total compute.
        \item The paper should disclose whether the full research project required more compute than the experiments reported in the paper (e.g., preliminary or failed experiments that didn't make it into the paper).
    \end{itemize}

\item {\bf Code of ethics}
    \item[] Question: Does the research conducted in the paper conform, in every respect, with the NeurIPS Code of Ethics \url{https://neurips.cc/public/EthicsGuidelines}?
    \item[] Answer: \answerYes{} 
    \item[] Justification: Yes, we followed the NeurIPS Code of Ethics \url{https://neurips.cc/public/EthicsGuidelines}.
    \item[] Guidelines:
    \begin{itemize}
        \item The answer NA means that the authors have not reviewed the NeurIPS Code of Ethics.
        \item If the authors answer No, they should explain the special circumstances that require a deviation from the Code of Ethics.
        \item The authors should make sure to preserve anonymity (e.g., if there is a special consideration due to laws or regulations in their jurisdiction).
    \end{itemize}

\item {\bf Broader impacts}
    \item[] Question: Does the paper discuss both potential positive societal impacts and negative societal impacts of the work performed?
    \item[] Answer: \answerNA{} 
    \item[] Justification: Our work uses only standard academic datasets like ImageNet for improving semi-supervised learning methods. We don't use any private or personal data, and our techniques don't present obvious negative societal impacts.
    \item[] Guidelines:
    \begin{itemize}
        \item The answer NA means that there is no societal impact of the work performed.
        \item If the authors answer NA or No, they should explain why their work has no societal impact or why the paper does not address societal impact.
        \item Examples of negative societal impacts include potential malicious or unintended uses (e.g., disinformation, generating fake profiles, surveillance), fairness considerations (e.g., deployment of technologies that could make decisions that unfairly impact specific groups), privacy considerations, and security considerations.
        \item The conference expects that many papers will be foundational research and not tied to particular applications, let alone deployments. However, if there is a direct path to any negative applications, the authors should point it out. For example, it is legitimate to point out that an improvement in the quality of generative models could be used to generate deepfakes for disinformation. On the other hand, it is not needed to point out that a generic algorithm for optimizing neural networks could enable people to train models that generate Deepfakes faster.
        \item The authors should consider possible harms that could arise when the technology is being used as intended and functioning correctly, harms that could arise when the technology is being used as intended but gives incorrect results, and harms following from (intentional or unintentional) misuse of the technology.
        \item If there are negative societal impacts, the authors could also discuss possible mitigation strategies (e.g., gated release of models, providing defenses in addition to attacks, mechanisms for monitoring misuse, mechanisms to monitor how a system learns from feedback over time, improving the efficiency and accessibility of ML).
    \end{itemize}

\item {\bf Safeguards}
    \item[] Question: Does the paper describe safeguards that have been put in place for responsible release of data or models that have a high risk for misuse (e.g., pretrained language models, image generators, or scraped datasets)?
    \item[] Answer: \answerNA{} 
    \item[] Justification: Our work focuses on semi-supervised learning methods using standard academic datasets and does not pose high risks for misuse.
    \item[] Guidelines:
    \begin{itemize}
        \item The answer NA means that the paper poses no such risks.
        \item Released models that have a high risk for misuse or dual-use should be released with necessary safeguards to allow for controlled use of the model, for example by requiring that users adhere to usage guidelines or restrictions to access the model or implementing safety filters.
        \item Datasets that have been scraped from the Internet could pose safety risks. The authors should describe how they avoided releasing unsafe images.
        \item We recognize that providing effective safeguards is challenging, and many papers do not require this, but we encourage authors to take this into account and make a best faith effort.
    \end{itemize}

\item {\bf Licenses for existing assets}
    \item[] Question: Are the creators or original owners of assets (e.g., code, data, models), used in the paper, properly credited and are the license and terms of use explicitly mentioned and properly respected?
    \item[] Answer: \answerYes{} 
    \item[] Justification: We properly credit all creators of assets used in our work.
    \item[] Guidelines:
    \begin{itemize}
        \item The answer NA means that the paper does not use existing assets.
        \item The authors should cite the original paper that produced the code package or dataset.
        \item The authors should state which version of the asset is used and, if possible, include a URL.
        \item The name of the license (e.g., CC-BY 4.0) should be included for each asset.
        \item For scraped data from a particular source (e.g., website), the copyright and terms of service of that source should be provided.
        \item If assets are released, the license, copyright information, and terms of use in the package should be provided. For popular datasets, \url{paperswithcode.com/datasets} has curated licenses for some datasets. Their licensing guide can help determine the license of a dataset.
        \item For existing datasets that are re-packaged, both the original license and the license of the derived asset (if it has changed) should be provided.
        \item If this information is not available online, the authors are encouraged to reach out to the asset's creators.
    \end{itemize}

\item {\bf New assets}
    \item[] Question: Are new assets introduced in the paper well documented and is the documentation provided alongside the assets?
    \item[] Answer: \answerNA{} 
    \item[] Justification: Our work does not release any new assets such as datasets, code, or models.
    \item[] Guidelines:
    \begin{itemize}
        \item The answer NA means that the paper does not release new assets.
        \item Researchers should communicate the details of the dataset/code/model as part of their submissions via structured templates. This includes details about training, license, limitations, etc.
        \item The paper should discuss whether and how consent was obtained from people whose asset is used.
        \item At submission time, remember to anonymize your assets (if applicable). You can either create an anonymized URL or include an anonymized zip file.
    \end{itemize}

\item {\bf Crowdsourcing and research with human subjects}
    \item[] Question: For crowdsourcing experiments and research with human subjects, does the paper include the full text of instructions given to participants and screenshots, if applicable, as well as details about compensation (if any)?
    \item[] Answer: \answerNA{} 
    \item[] Justification: Our work does not involve crowdsourcing or research with human subjects.
    \item[] Guidelines:
    \begin{itemize}
        \item The answer NA means that the paper does not involve crowdsourcing nor research with human subjects.
        \item Including this information in the supplemental material is fine, but if the main contribution of the paper involves human subjects, then as much detail as possible should be included in the main paper.
        \item According to the NeurIPS Code of Ethics, workers involved in data collection, curation, or other labor should be paid at least the minimum wage in the country of the data collector.
    \end{itemize}

\item {\bf Institutional review board (IRB) approvals or equivalent for research with human subjects}
    \item[] Question: Does the paper describe potential risks incurred by study participants, whether such risks were disclosed to the subjects, and whether Institutional Review Board (IRB) approvals (or an equivalent approval/review based on the requirements of your country or institution) were obtained?
    \item[] Answer: \answerNA{} 
    \item[] Justification: Our work does not involve crowdsourcing or research with human subjects.
    \item[] Guidelines:
    \begin{itemize}
        \item The answer NA means that the paper does not involve crowdsourcing nor research with human subjects.
        \item Depending on the country in which research is conducted, IRB approval (or equivalent) may be required for any human subjects research. If you obtained IRB approval, you should clearly state this in the paper.
        \item We recognize that the procedures for this may vary significantly between institutions and locations, and we expect authors to adhere to the NeurIPS Code of Ethics and the guidelines for their institution.
        \item For initial submissions, do not include any information that would break anonymity (if applicable), such as the institution conducting the review.
    \end{itemize}

\item {\bf Declaration of LLM usage}
    \item[] Question: Does the paper describe the usage of LLMs if it is an important, original, or non-standard component of the core methods in this research? Note that if the LLM is used only for writing, editing, or formatting purposes and does not impact the core methodology, scientific rigorousness, or originality of the research, declaration is not required.
    \item[] Answer: \answerNA{} 
    \item[] Justification: Our work does not involve any LLMs as important, original, or non-standard components.
    \item[] Guidelines:
    \begin{itemize}
        \item The answer NA means that the core method development in this research does not involve LLMs as any important, original, or non-standard components.
        \item Please refer to our LLM policy (\url{https://neurips.cc/Conferences/2026/LLM}) for what should or should not be described.
    \end{itemize}

\end{enumerate}

\end{document}